\documentclass[letterpaper, 10 pt, conference]{ieeeconf}  
\IEEEoverridecommandlockouts

\usepackage[T1]{fontenc}
\usepackage{tgtermes}
\usepackage{subcaption}  

\usepackage{dsfont}
\usepackage{color}
\definecolor{pink}{rgb}{0.58,0,0.83}
\definecolor{orange}{rgb}{1,0.5,0}
\definecolor{lightgreen}{rgb}{0.2, 0.8, 0.2}
\definecolor{lightyellow}{rgb}{0.84, 0.65, 0.13}

\usepackage[colorlinks,linkcolor=blue,anchorcolor=blue,citecolor=red]{hyperref}
\usepackage{amsmath,amsfonts}
\usepackage{algorithmic}
\usepackage[linesnumbered, ruled]{algorithm2e}
\usepackage{array}

\usepackage{makecell}
\usepackage{textcomp}
\usepackage{stfloats}
\usepackage{url}
\usepackage{verbatim}
\usepackage{graphicx}
\usepackage{amssymb}
\usepackage{dsfont}
\usepackage{bbding}
\usepackage{booktabs}
\usepackage{multirow}
\usepackage{tablefootnote}
\usepackage{cite}
\usepackage{xcolor}
\usepackage{threeparttable}
\usepackage{caption}

\newcounter{definition}
\renewcommand{\thedefinition}{\arabic{definition}}
\newenvironment{definition}[1][]{%
  \refstepcounter{definition}%
  \par\medskip\noindent
  \textbf{Definition~\thedefinition%
    \if\relax\detokenize{#1}\relax\else\ (\normalfont #1)\fi.}\normalfont
}{\par\medskip}

\newcounter{problem}
\renewcommand{\theproblem}{\arabic{problem}}
\newenvironment{problem}[1][]{%
  \refstepcounter{problem}%
  \par\medskip\noindent
  \textbf{Problem~\theproblem%
    \if\relax\detokenize{#1}\relax\else\ (\normalfont #1)\fi.~}\normalfont
}{\par\medskip}

\newcounter{assumption}
\renewcommand{\theassumption}{\arabic{assumption}}

\newcounter{lemma}
\renewcommand{\thelemma}{\arabic{lemma}}

\newcounter{theorem}
\renewcommand{\thetheorem}{\arabic{theorem}}
\newenvironment{theorem}[1][]{%
  \refstepcounter{theorem}%
  \par\medskip\noindent
  \textbf{Theorem~\thetheorem%
    \if\relax\detokenize{#1}\relax\else\ (\normalfont #1)\fi.~}\normalfont
}{\par\medskip}

\newtheorem{remark}{Remark}
\usepackage{arydshln}
\usepackage{multicol}
\usepackage{soul}
\usepackage{tablefootnote} 
\definecolor{ForestGreen}{RGB}{34,139,34}
\def\BibTeX{{\rm B\kern-.05em{\sc i\kern-.025em b}\kern-.08em
		T\kern-.1667em\lower.7ex\hbox{E}\kern-.125emX}}
\usepackage{balance}
\begin{document}
	\title{\bf Simultaneous Arrival Control for Distributed Multi-Robot Systems with Curvature and Constant-Speed Constraints}
	\author{
		Zhouru Xiao$^{1}$, Yang Lu$^{2}$, Weijia Yao$^{1}$, Min Liu$^{1}$, Yaonan Wang$^{1}$
  \thanks{$^{1} $Zhouru Xiao, Weijia Yao, Min Liu and Yaonan Wang are with the School of Artificial Intelligence and Robotics, Hunan University, China (email: {xzr798@hnu.edu.cn,}{wjyao@hnu.edu.cn}) (Corresponding author: Weijia Yao). The work of Yao was supported by the National Natural Science Foundation of China under Grant 62573182.}
        \thanks{$^{2}$Yang Lu is with the College of Intelligence Science and Technology, National University of Defense Technology, China (email: {luyang18@mail.sdu.edu.cn}).}
      }

		
	
	\maketitle
	\thispagestyle{empty}
	\pagestyle{empty}
	
	\begin{abstract}
The simultaneous arrival of multiple mobile robots at a target point is crucial for cooperation tasks such as cooperative encirclement, disaster relief, and environmental monitoring.  Although the simultaneous arrival problem itself is already complex, the problem becomes more challenging when there are constraints on the robot trajectory curvatures and the speeds are required to be constant (possibly different for different robots), and the control law for robots needs to be distributed. These constraints are typical for a multi-robot system consisting of, e.g., fixed-wing UAVs. To address this challenge, this paper proposes a distributed switching control method based on the maximum consensus protocol. By exploiting the geometric properties of Dubins paths along with optimization principles, a virtual time variable is introduced, and a hybrid control law that combines optimal control with saturated proportional control is designed. Under the proposed control law, each robot is driven to approach the maximum virtual time among its neighbors, thereby achieving  simultaneous arrival under some mild conditions. Furthermore, we prove that in certain cases the proposed method attains a theoretically optimal arrival time. The approach is scalable and real-time, with low communication overhead. Its effectiveness and robustness are validated through extensive simulations and experiments.

	\end{abstract}
	
	
	\definecolor{limegreen}{rgb}{0.2, 0.8, 0.2}
	\definecolor{forestgreen}{rgb}{0.13, 0.55, 0.13}
	\definecolor{greenhtml}{rgb}{0.0, 0.5, 0.0}
	
	\section{Introduction}
Cooperative motion planning of multi-robot systems represents a challenge in control theory and robotics. Achieving efficient, robust, and distributed coordination becomes particularly difficult in the presence of physical constraints and disturbances from dynamic environments. In particular, the distributed simultaneous arrival problem requires a group of mobile robots, subject to robot trajectory curvature and constant-speed constraints, to reach a target location simultaneously from arbitrary initial states without a centralized coordination mechanism. This problem relates to diverse application scenarios, such as cooperative target interception and cooperative encirclement. Relevant studies on this topic have been reported in \cite{Li2020,9482855,8303222,11004631} and the references therein.

The simultaneous arrival problem is similar to the multi-robot rendezvous problem \cite{795787,doi:10.1137/040620552,doi:10.1137/040620564}, but with stricter requirements on the target points. In the general rendezvous problem \cite{1381658,4200856,5783895}, robots are only required to gather at the same location, and such a location is typically \emph{not} predetermined. In contrast, the simultaneous arrival problem demands that all robots reach a \emph{pre-specified} target location at exactly the same time, which imposes higher requirements on control strategies in terms of accuracy and temporal consistency \cite{Li2020}. This problem resembles the cooperative missile strike and interception problem; however, unlike missile systems, the multi-robot distributed simultaneous arrival problem places greater emphasis on individual kinematic constraints (e.g., curvature limitation and constant speed), and necessitate much stricter requirements on target arrival precision and time synchronization. This problem even requires precise simultaneous arrival in the undesirable scenario where robots are close to the targets in distance but the their initial headings direct towards the opposite direction to the target, thereby highlighting the value as well as the challenge of this problem in cooperative tasks.

For the distributed simultaneous arrival problem, several cooperative control approaches have been proposed, which can be broadly categorized into three classes: time-to-go estimation guidance laws, leader-following methods, and consensus-based algorithms.
Time-to-go estimation guidance laws regulate the impact or arrival time by explicitly estimating the remaining flight time of each robot. For example, \cite{1597196} combined proportional navigation guidance (PNG) with time-to-go error feedback to achieve the desired impact time, while \cite{doi:10.2514/1.G001349} designed a Lyapunov-based guidance law using the heading error to achieve time control under constrained initial conditions.
Leader-following methods achieve coordination by letting the rest of the robots adjust their motion relative to a designated leader or reference trajectory. For instance, \cite{4200856} proposed a discontinuous time-invariant control strategy, and \cite{ZHENG2013401} employed a bearing-based control law to contract the polygon formed by multiple robots' positions, thereby achieving coordinated rendezvous. However, these methods often rely on centralized information or designated leaders, limiting their scalability and robustness. In contrast, consensus algorithms can achieve global time synchronization through local information exchange and have been extensively studied in the literature \cite{1333204,9143934}, with demonstrated effectiveness even under communication delays and switching topologies \cite{11004631}. In fact, many existing cooperative guidance methods employ consensus-based strategies and have demonstrated strong guidance performance \cite{8106791,8318665,11018241}.

Nevertheless, existing studies still exhibit two limitations. First, most studies do not fully account for the physical saturation constraints of actuators in the form of maximum curvature or steering angle limits, which often prevents theoretical control laws from being effectively implemented on real systems. Second, when robots start extremely close to the target, conventional approaches suffer from severely limited maneuvering space and lack effective trajectory generation and time coordination mechanisms, resulting in infeasible paths. For instance, although \cite{9482855} considers saturation constraints, it has not considered scenarios in which robots originate near the target. Moreover, many theoretical results have been validated only in simulation environments, with limited experimental verification on real robotic platforms.

Dubins path theory offers an important insight for addressing the aforementioned challenges (i.e., distributed simultaneous arrival subject to kinematic and saturation constraints with adverse initial conditions). It provides a geometrically feasible and time-optimal path generation method for mobile systems subject to curvature constraints, such as unmanned ground vehicles and fixed-wing aircraft \cite{doi:10.2514/1.59512}. Since Dubins introduced the shortest-path construction for such vehicles in 1957, his model has been widely adopted for minimum-time trajectory planning on various nonholonomic platforms \cite{MLYNCH2003173,MATVEEV2011515}. Under the constant-speed assumption, the shortest path coincides with the minimum-time path, making Dubins paths a natural basis for addressing both temporal coordination and geometric constraints in the simultaneous arrival problem. By integrating Dubins paths with consensus protocols, the challenges posed by kinematic and saturation constraints under extreme initial conditions can be effectively addressed within a distributed control framework.

\textbf{Contributions:}  
To realize curvature-constrained and constant-speed distributed simultaneous arrival control for a heterogeneous nonholonomic multi-robot system, we propose a distributed switching control method based on the maximum consensus protocol. By exploiting the geometric properties of Dubins paths along with optimization principles, a virtual time variable is introduced, and a hybrid control law that combines optimal control with saturated proportional control is designed. Under the proposed control law, each robot is driven to progressively approach the maximum virtual time among its neighbors, thereby achieving  simultaneous arrival. The proposed approach exhibits four notable advantages: 1) It is distributed and scalable, making it applicable to multi-robot systems of arbitrary sizes;  2) It provides a rigorous mechanistic analysis of the simultaneous arrival problem; 3) Under a fixed communication frequency, the communication burden is low, as any two neighboring robots need to transmit and receive at most a single scalar, namely the virtual time variable; 4) Extensive simulations of large-scale multi-robot systems and experiments with multiple quadrotors validate the effectiveness, and we demonstrate the method's capability of collision avoidance.

\textbf{Notations:}  
Let an integer set be denoted by $\mathbb{Z}_i^j := \{ m \in \mathbb{Z} : i \le m \le j\}$, where $i,j \in \mathbb{Z}$ and $i \le j$. We use boldface to represent vectors $\mathbf{v} \in \mathbb{R}^n$ in an $n$-dimensional real vector space, with the $i$-th component denoted by $v_i$, where $i \in \mathbb{Z}_1^n$. For any set $\mathcal{S}$, its complement in the universal set $\mathcal{U}$ is denoted by $\mathcal{S}^c := \mathcal{U} \setminus \mathcal{S}$.

\textbf{Graphs:}  
We define the  \emph{node set} as $\mathcal{V} := \{1, \ldots, N\}$ to represent robots, and the  \emph{edge set} $\mathcal{E} \subseteq (\mathcal{V} \times \mathcal{V})$ encodes the communication links between neighboring robots. The neighbor set of robot $i$ is defined as $\mathcal{N}_i := \{ j \in \mathcal{V} : (i,j) \in \mathcal{E} \}$. In this paper, we consider only undirected graphs, meaning that if $(i,j) \in \mathcal{E}$, then robots $i$ and $j$ can share information bidirectionally. Please see \cite{MesbahiEgerstedt+2010} for an introduction to graph theory.

\section{Preliminaries And Problem Formulation}\label{PreliminariesAndProblemFormulation}
Consider a group of $N$ mobile robots moving in the space $\mathbb{R}^2$, subject to nonholonomic motion constraints. For the $i$-th robot, $i \in \mathbb{Z}_1^N$, its state is represented as
$\boldsymbol{\xi}_i = [\boldsymbol{p}_i^{\top}, \theta_i ] \in \mathbb{R}^2 \times \mathbb{S}^1$, where $\boldsymbol{p}_i = [x_i, y_i]^\top$ denotes the position coordinates and $\theta_i \in [0, 2\pi)$ denotes the heading angle. The $i$-th robot's kinematics are governed by the nonholonomic constraints:
\begin{equation}\label{0001}
\begin{aligned}
& \dot{x}_i = v_i \cos \theta_i, \\
& \dot{y}_i = v_i \sin \theta_i, \\
& \dot{\theta}_i = \omega_i,
\end{aligned}
\end{equation}
where $v_i > 0$ is a \emph{constant} linear speed and $\omega_i \in [-\bar{\omega}_i, \bar{\omega}_i]$ is the controllable angular velocity input. Furthermore, due to the limitations of the robot's steering mechanism, the minimum turning radius is given by
\begin{equation}\label{0002}
\rho_i = \frac{v_i}{\bar{\omega}_i}.
\end{equation}

To describe the relative position of robot $i$ with respect to a fixed target point $\boldsymbol{p}_i^d \in \mathbb{R}^2$, we introduce a polar-coordinate representation of the relative state vector
$\boldsymbol{\zeta}_i = [r_i, \phi_i]^\top \in \mathbb{R}_{\ge 0} \times (-\pi, \pi]$, 
defined as (see Fig.~\ref{fig:0001}):
\begin{equation}\label{0003}
\begin{aligned}
& r_i = \| \boldsymbol{p}_i - \boldsymbol{p}_i^d \|, \\
& \phi_i = \arg(\boldsymbol{p}_i - \boldsymbol{p}_i^d) - \theta_i,
\end{aligned}
\end{equation}
where $\arg(\cdot)$ denotes the angle between a vector in the plane and the positive $x$-axis, taking values in $(-\pi, \pi]$. 


\begin{figure}[!htbp]
    \centering
    \includegraphics[width=3in]{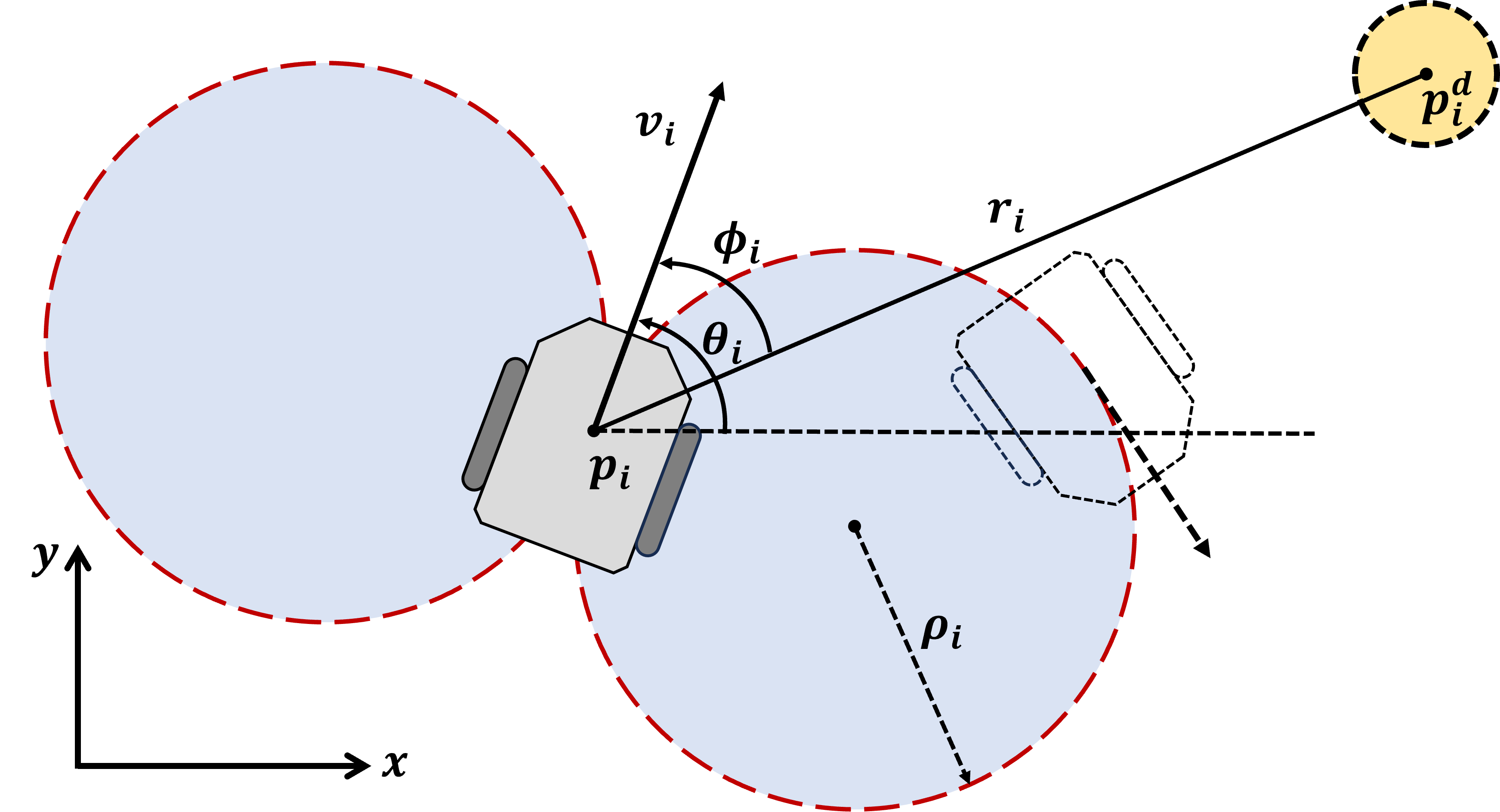}
    \caption{The relative position between a vehicle and the target.}
    \label{fig:0001}
\end{figure}

\begin{problem}
Consider a group of $N$ nonholonomic robots whose kinematics and minimum turning radii are given by \eqref{0001} and \eqref{0002}, respectively, where each robots move at possibly different constant speeds $v_i$. Let $\boldsymbol{p}_i^d$ denote the target position of robot $i$, and let the inter-robot communication be modeled by an undirected connected graph $\mathcal{G} = (\mathcal{V}, \mathcal{E})$. The motion planning objective is to design the angular velocity control inputs $\Omega = [\omega_1, \omega_2, \dots, \omega_N]^\top \in \mathbb{R}^N$ in \eqref{0001} such that the resulting robot trajectories $\boldsymbol{p}_i(t)$ satisfy:

1) (\textbf{Simultaneous Arrival}) There exists a time $t^*$ such that $\|\boldsymbol{p}_i(t^*) - \boldsymbol{p}_i^d\| = 0$ for all $i\in\mathbb{Z}_1^N$.

2) (\textbf{No Premature Arrival}) For any time $t < t^*$, it holds that $\|\boldsymbol{p}_i(t) - \boldsymbol{p}_i^d\| > 0$ for all $i\in\mathbb{Z}_1^N$.

3) (\textbf{Constraint Satisfaction}) The robot trajectories satisfy the constant speed constraint in \eqref{0001} and the minimum turning radius constraint in \eqref{0002}.
\end{problem}

It is noteworthy that each control input $\omega_i$ depends solely on the state $(r_i, \phi_i)$ of robot $i$, and the states of its neighbors set $\mathcal{N}_i$ via communication, and thus the control law is distributed and scalable.

\begin{definition}[Optimal Arrival Time]
The \emph{optimal arrival time is defined as $t^* = \min_{t \in \mathcal{T}} t$, where the set $\mathcal{T} = \{ t\ge 0 : \| \boldsymbol{p}_i(t) - \boldsymbol{p}_i^d \| = 0, \, \forall i \in \mathbb{Z}_1^N \}$.}  
\end{definition}

\begin{remark}
    Compared with single-robot navigation tasks as in \cite{7484276} and \cite{7300456}, our problem involves multiple robots and imposes a temporal coordination constraint requiring simultaneous arrival. Compared with tasks that guide robots to desired poses as in \cite{10886648} and \cite{10540263}, although the terminal orientations in our problem are unconstrained, the additional requirement of simultaneous arrival increases the problem's complexity.
\end{remark}

\section{Max-Consensus-Driven Simultaneous Arrival Control}\label{Time_Coordinated_Vector_Field}

This section introduces a maximum consensus approach by incorporating virtual time variables. The proposed approach not only guarantees that all robots arrive at their respective target positions simultaneously, but also is capable of addressing the simultaneous arrival problem, including undesirable initial states where some robots are initially very close to their targets.

\subsection{Virtual Time Variable}

We introduce a virtual time variable $T(r,\phi):\mathbb{R}^2\to\mathbb{R}$, defined as the theoretically minimal time required for a robot to reach its target under curvature constraints. Given the robot's constant speed, this variable is in a one-to-one correspondence with the Dubins path length $L$ without considering the terminal heading, i.e., $L = v T$, where $v$ is constant. Although $T$ does not necessarily equal the actual time spent by the robot, it serves as a metric to characterize the optimal traveling time in trajectory planning and control. Leveraging this variable, the trajectory planning problem boils down to the analysis of time optimality, establishing a direct connection to classical time-optimal control problems. For the case of a single Dubins-car-modelled robot pursuing a stationary target, this reformulation corresponds to the classical minimum-time control problem, which has been systematically studied in \cite{doi:10.2514/1.59512}. 
In this section, we omit the subscript $i$ for variables associated with robot $i$.

The time optimal control law for the Dubins paths without terminal heading constraints \cite{9482855} is given by
\begin{equation}\label{eq:0006}
\omega^*(r,\phi) =
\begin{cases}
-\bar{\omega}, & (r,\phi) \in S_- \cup C_-, \\
0, & (r,\phi) \in S_0, \\
\bar{\omega}, & (r,\phi) \in S_+ \cup C_+,
\end{cases}
\end{equation}
where the sets $C_-$, $C_+$, $S_-$, $S_+$, and $S_0$ are defined as
\begin{equation}\label{eq:0007}
\begin{aligned}
& C_- = \{ (r,\phi) : r \le 2 \rho \sin\phi, \phi > 0 \}, \\
& C_+ = \{ (r,\phi) : r \le -2 \rho \sin\phi, \phi < 0 \}, \\
& S_- = \{ (r,\phi) : r > -2 \rho \sin\phi, \phi < 0 \}, \\
& S_+ = \{ (r,\phi) : r > 2 \rho \sin\phi, \phi > 0 \}, \\
& S_0 = \{ (r,\phi) : \phi = 0 \}.
\end{aligned}
\end{equation}

\begin{remark}
The term $2\rho \sin \phi$ in \eqref{eq:0007} represents the lateral offset corresponding to the heading error $\phi$ at the radius of curvature $\rho$. It is used to define the boundaries of the regions $C_\pm$ and $S_\pm$ for the bang-bang angular velocity control $\omega^*(r,\phi)$\cite{7799011}. The sign of $\phi$ only determines the turning direction, either a left turn (corresponding to $C_+$, $S_+$) or a right turn (corresponding to $C_-$, $S_-$), and does not affect the time required to reach the target. That is, for the two initial states $(r,\phi)$ and $(r,-\phi)$, the time to reach the goal under the control law in (\ref{eq:0006}) is equal. Therefore, unless otherwise stated, we may assume $\phi \geq 0$ in the following discussion.
\end{remark}

Based on (\ref{eq:0006}) and (\ref{eq:0007}), and under the assumption of no terminal heading constraints, a Dubins path differs from the typical three-segment structure in \cite{SHKEL2001179,10938343} and contains at most two segments: a single circular arc or a straight line ($C$ or $S$), a circular arc followed by a straight line ($CS$), or a double circular arc ($CC$). Based on the above analysis, we introduce the following definition.

\begin{definition}[Workspace Partitioning]
For a robot with minimum turning radius $\rho$, the \emph{curvature-constrained region} is defined as $\mathcal{D}=\{(r,\phi) : r \le 2\rho\sin|\phi|\}$, and the \emph{curvature-feasible region} as its complement $\mathcal{D}^c$.
\end{definition}

This partition reflects the geometric threshold characteristics of Dubins path types. Within the curvature-constrained region $\mathcal{D}$, the robot's optimal Dubins paths are of the $CC$ or $C$ type, where the path curvature always reaches the robot's maximum allowable value. In contrast, within the curvature-feasible region $\mathcal{D}^c$, the optimal Dubins paths are of the $CS$ or $S$ type, where the path can be realized through a combination of straight and arc segments. 

The study \cite{9482855} has pointed out that within the curvature-constrained region $\mathcal{D}$, the path length $L$ exhibits nonlinear characteristics as the initial condition changes, whereas in the curvature-feasible region, $L$ changes approximately linearly. While the study proposed a numerical method to estimate $L$ in the curvature-feasible region, it lacks an effective analytical approach for the curvature-constrained region. To address this,  we leverage the geometric construction rules of relaxed Dubins paths \cite{351019} to derive an analytical solution for the Dubins path length as a function of the initial position, excluding terminal heading constraints, as detailed below.

\begin{figure}[!htbp]
    \centering
    \begin{subfigure}{0.23\textwidth}
        \centering
        \includegraphics[width=\linewidth]{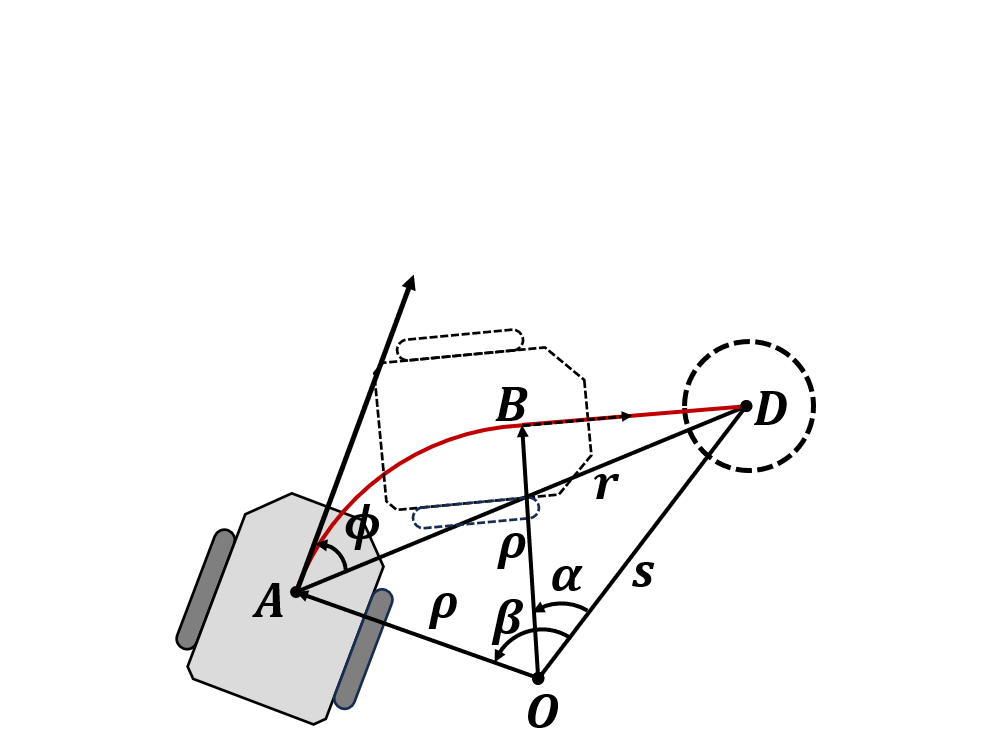}
        \caption{}
        \label{fig:subfig2}
    \end{subfigure}
    \begin{subfigure}{0.23\textwidth}
        \centering
        \includegraphics[width=\linewidth]{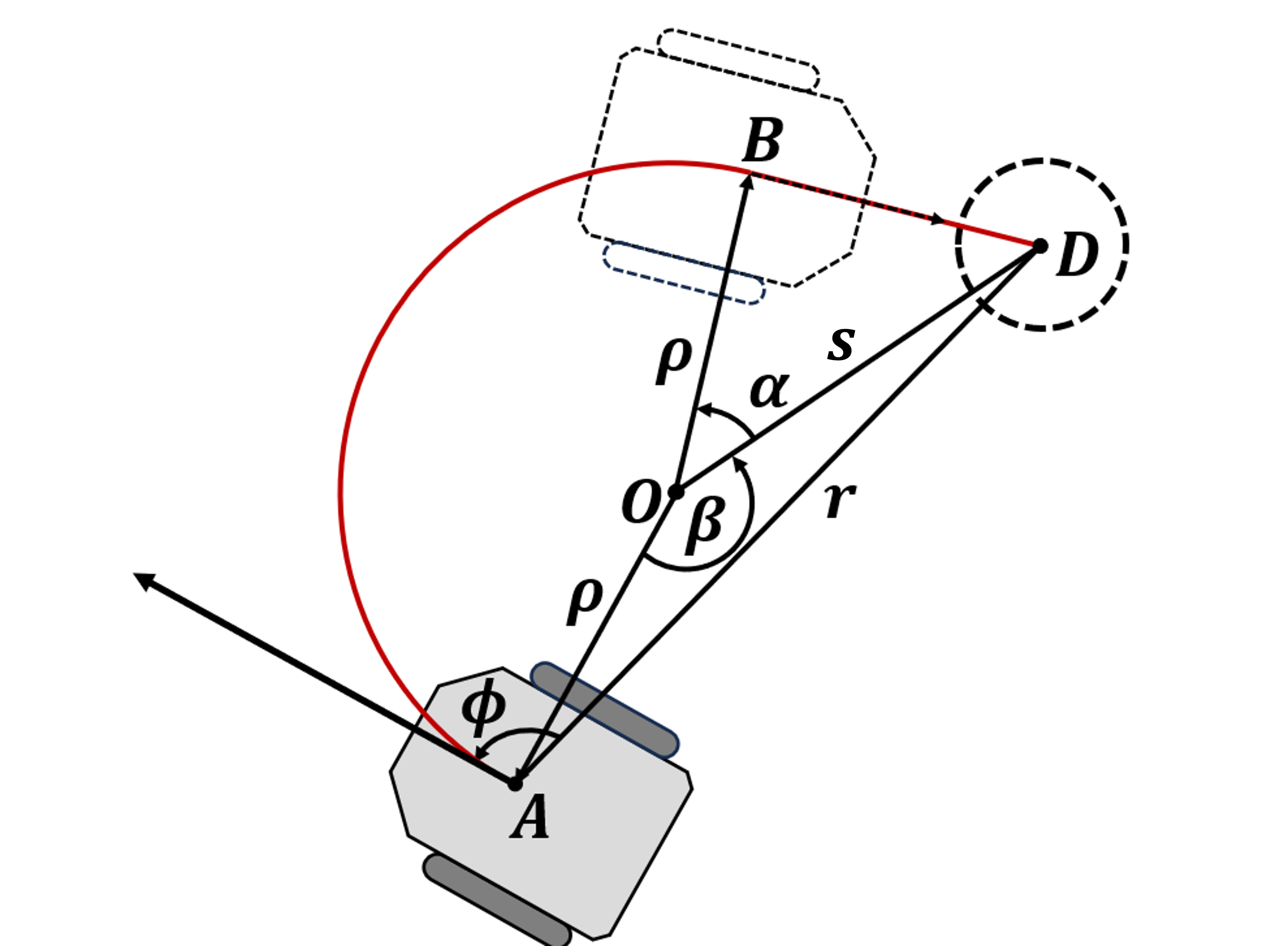}
        \caption{}
        \label{fig:subfig1}
    \end{subfigure}
    
    
    \begin{subfigure}{0.23\textwidth}
        \centering
        \includegraphics[width=\linewidth]{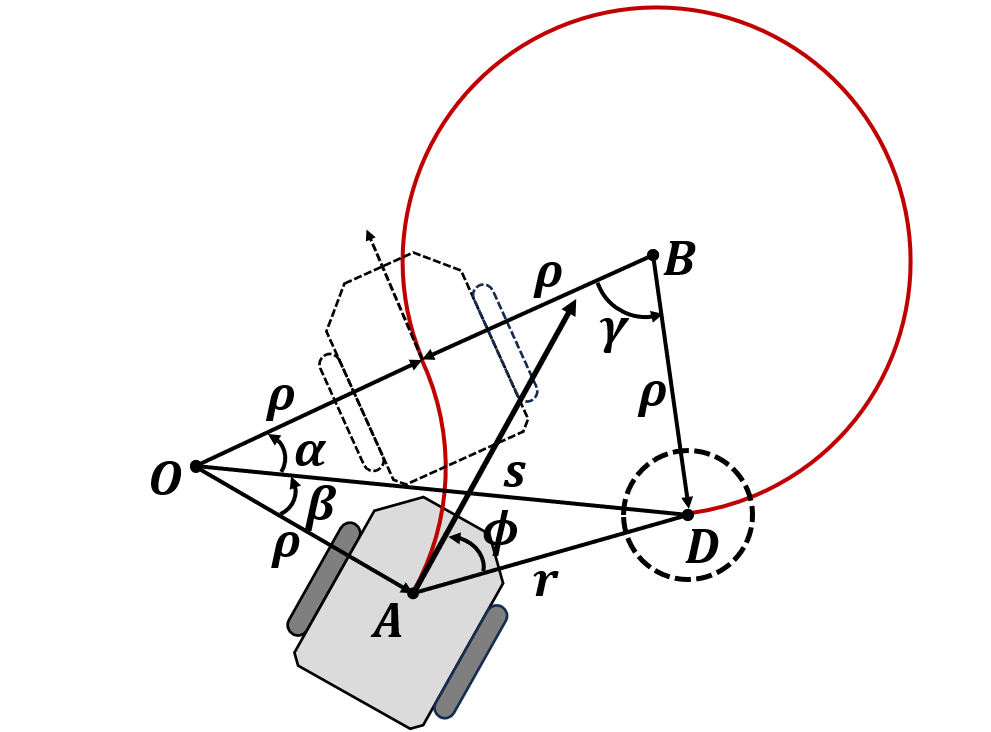}
        \caption{}
        \label{fig:subfig4}
    \end{subfigure}
    \begin{subfigure}{0.23\textwidth}
        \centering
        \includegraphics[width=\linewidth]{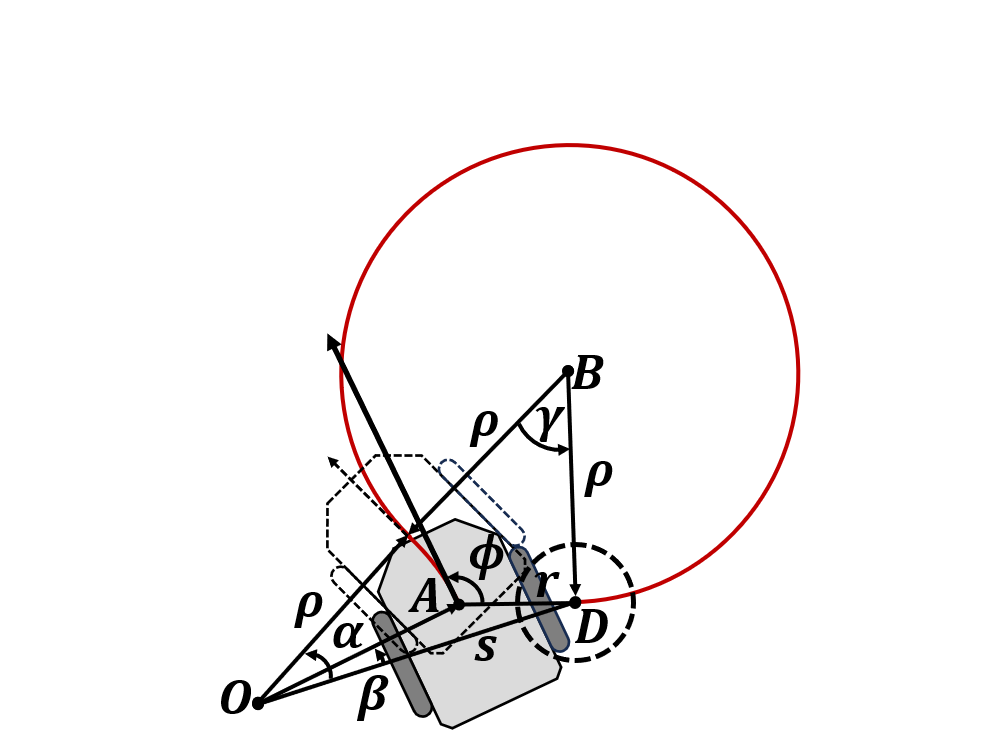}
        \caption{}
        \label{fig:subfig3}
    \end{subfigure}
    \caption{Illustration of four cases for the calculation of the Dubins path lengths without terminal heading constraints. (a) $\boldsymbol{\zeta}\in\mathcal{D}^c, |\phi|< \frac{\pi}{2}$; (b) $\boldsymbol{\zeta}\in\mathcal{D}^c, |\phi|\geq\frac{\pi}{2}$; (c) $\boldsymbol{\zeta}\in\mathcal{D}, |\phi|<\frac{\pi}{2}$; (d) $\boldsymbol{\zeta}\in\mathcal{D}, |\phi|\geq\frac{\pi}{2}$.}
    \label{fig:0002}
\end{figure}

\textbf{Case 1:} $\boldsymbol{\zeta}=[r, \phi]^\top \in\mathcal{D}^c$, let 
$s=\sqrt{r^{2}+\rho^{2}-2r\rho\sin|\phi|}$, $\alpha=\arccos{\frac{\rho}{s}}$, $\beta=\arccos{\frac{\rho-r\sin\phi}{s}}$, 
The Dubins path length is given by:  
\begin{equation}\label{eq:00017}
L(r, \phi)=
\begin{cases}
(\beta-\alpha)\rho+\sqrt{s^2-\rho^2}, & |\phi|< \frac{\pi}{2}, \\
(2\pi-\alpha-\beta)\rho+\sqrt{s^2-\rho^2}, & |\phi|\geq\frac{\pi}{2}.
\end{cases}
\end{equation}

\textbf{Case 2:} $\boldsymbol{\zeta}=[r, \phi]^\top \in\mathcal{D}$, let 
$s=\sqrt{r^{2}+\rho^{2}+2r\rho\sin |\phi|}$, $\alpha=\arccos{\frac{3\rho^2+s^2}{4\rho s}}$. $\beta=\arccos{\frac{\rho^2+s^2-r^2}{2\rho s}}$, $\gamma=\arccos{\frac{5\rho^2-s^2}{4\rho^2}}$, 
The Dubins path length is given by:  
\begin{equation}\label{eq:00018}
L(r, \phi)=
\begin{cases}
(2\pi+\alpha+\beta-\gamma)\rho, & |\phi|< \frac{\pi}{2},\\
(2\pi+\alpha-\beta-\gamma)\rho, & |\phi|\geq\frac{\pi}{2}.
\end{cases}
\end{equation}

The derivation of \eqref{eq:00017} and \eqref{eq:00018} follows from the law of cosines applied to the geometric configuration in Fig.~\ref{fig:0002}, and the monotonicity of the Dubins path length function $L$ is established in the following theorem.

\begin{theorem}[Monotonicity]\label{th:0001}
\emph{Let the minimum turning radius be $\rho$. For $(r,\phi)\in\mathcal{D}^c$, the Dubins path length function $L(r,\phi)$ is strictly increasing in $\phi$ over $[0,\pi]$.}
\end{theorem}
\begin{proof}
Consider $(r,\phi) \in \mathcal{D}^c$, where $r > 2 \rho \sin\phi$, and define
$A = \sqrt{s^2 - (\rho - r \sin\phi)^2}$, 
$B = \sqrt{s^2 - \rho^2}$, and
$C = \rho - r \sin\phi$.

\textbf{Case 1:} $\phi \in [0, \frac{\pi}{2}]$.  
Observe that $(s^2 - \rho C)^2 - A^2 B^2 = r^2 s^2 \sin^2\phi \ge 0$, which implies $\frac{s^2 - \rho C}{A} \ge B$. Therefore,
\begin{equation*}
\begin{aligned}
\frac{dL}{d\phi} 
&= \rho \left( \frac{d\beta}{d\phi} - \frac{d\alpha}{d\phi} \right) + \frac{d}{d\phi} \sqrt{s^2 - \rho^2} \\
&= \frac{r \rho \cos\phi (s^2 - \rho C)}{s^2 A} + \frac{r \rho^3 \cos\phi}{s^2 B} - \frac{r \rho \cos\phi}{B} \\
&= \frac{r \rho \cos\phi}{s^2} \left( \frac{s^2 - \rho C}{A} - B \right) \ge 0,
\end{aligned}
\end{equation*}
where the equality holds if and only if $\phi = 0$.

\textbf{Case 2:} $\phi \in [\frac{\pi}{2}, \pi]$.  
Observe that $s^2 - \rho C = r(r - \rho \sin\phi) > r \rho \sin\phi \ge 0$. Therefore,
\begin{equation*}
\begin{aligned}
\frac{dL}{d\phi} 
&= -\rho \left( \frac{d\alpha}{d\phi} + \frac{d\beta}{d\phi} \right) + \frac{d}{d\phi} \sqrt{s^2 - \rho^2} \\
&= \frac{r \rho^3 \cos\phi}{s^2 B} - \frac{r \rho \cos\phi (s^2 - \rho C)}{s^2 A} - \frac{r \rho \cos\phi}{B} \\
&= -\frac{r \rho \cos\phi}{s^2} \left( B + \frac{s^2 - \rho C}{A} \right) > 0.
\end{aligned}
\end{equation*}

In conclusion, for all $\phi \in [0, \pi]$, we have $\frac{dL}{d\phi} \ge 0$, with equality if and only if $\phi = 0$.  
Hence, the Dubins path length function $L(r,\phi)$ is strictly monotonically increasing with respect to $\phi$ over the interval $[0, \pi]$.
\end{proof}

\begin{remark}
     If $(r,\phi)\in \mathcal{D}$, then $L(r,\phi)$ is strictly decreasing with respect to $\phi \in [\arcsin\frac{r}{2\rho}, \pi-\arcsin\frac{r}{2\rho}]$,  and is strictly increasing with respect to $\phi \in [0, \arcsin\frac{r}{2\rho}]$ and $\phi \in [\pi-\arcsin\frac{r}{2\rho}, \pi]$. The variation of $L(r,\phi)$ with respect to $\phi$ is illustrated in Fig.~\ref{fig:0003}.
\end{remark}

\begin{figure}[!h]
    \centering
    \begin{subfigure}{0.23\textwidth}
        \centering
        \includegraphics[width=\linewidth]{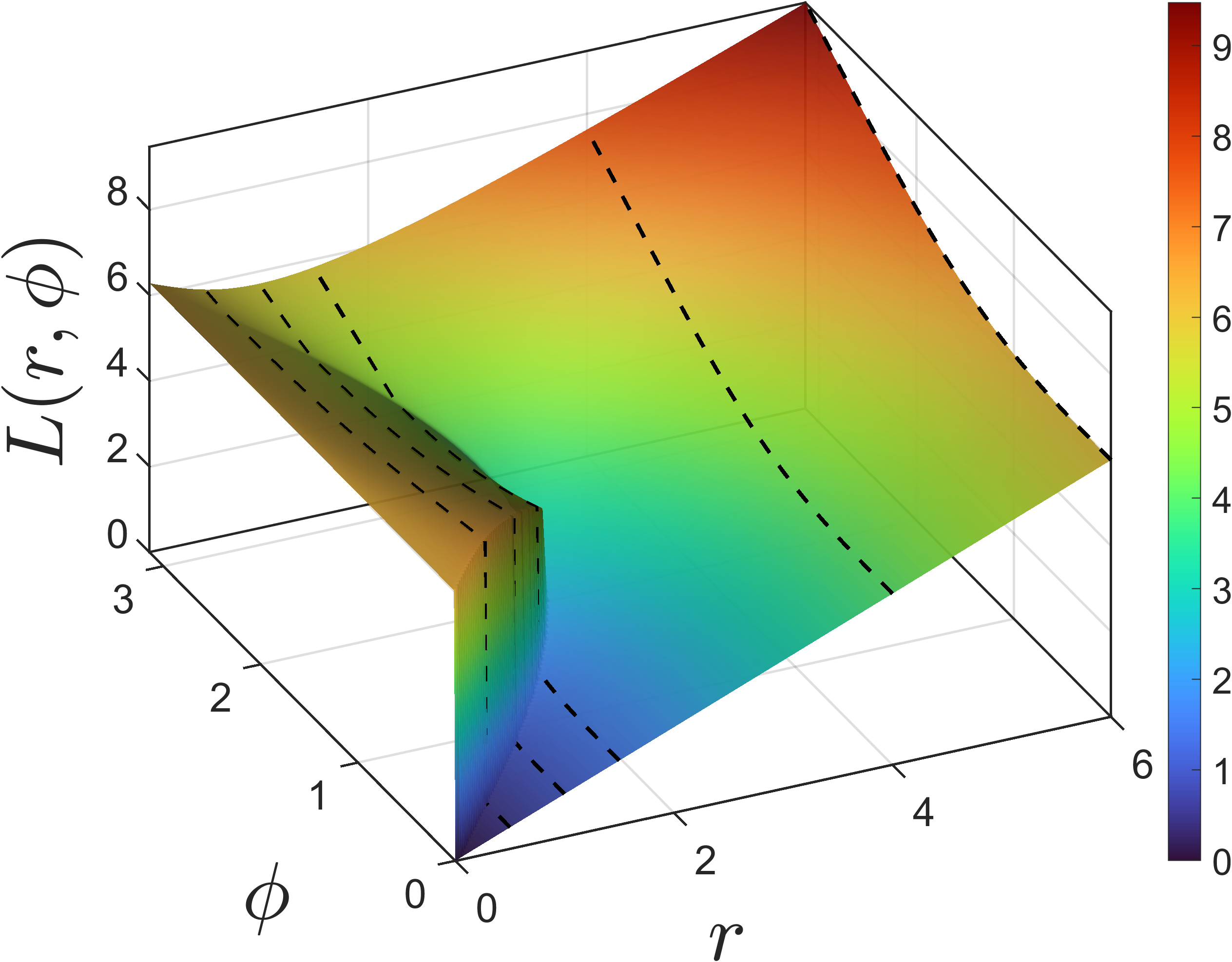}
        \caption{}
        \label{fig:subfig1}
    \end{subfigure}
    \hfill
    \begin{subfigure}{0.23\textwidth}
        \centering
        \includegraphics[width=\linewidth]{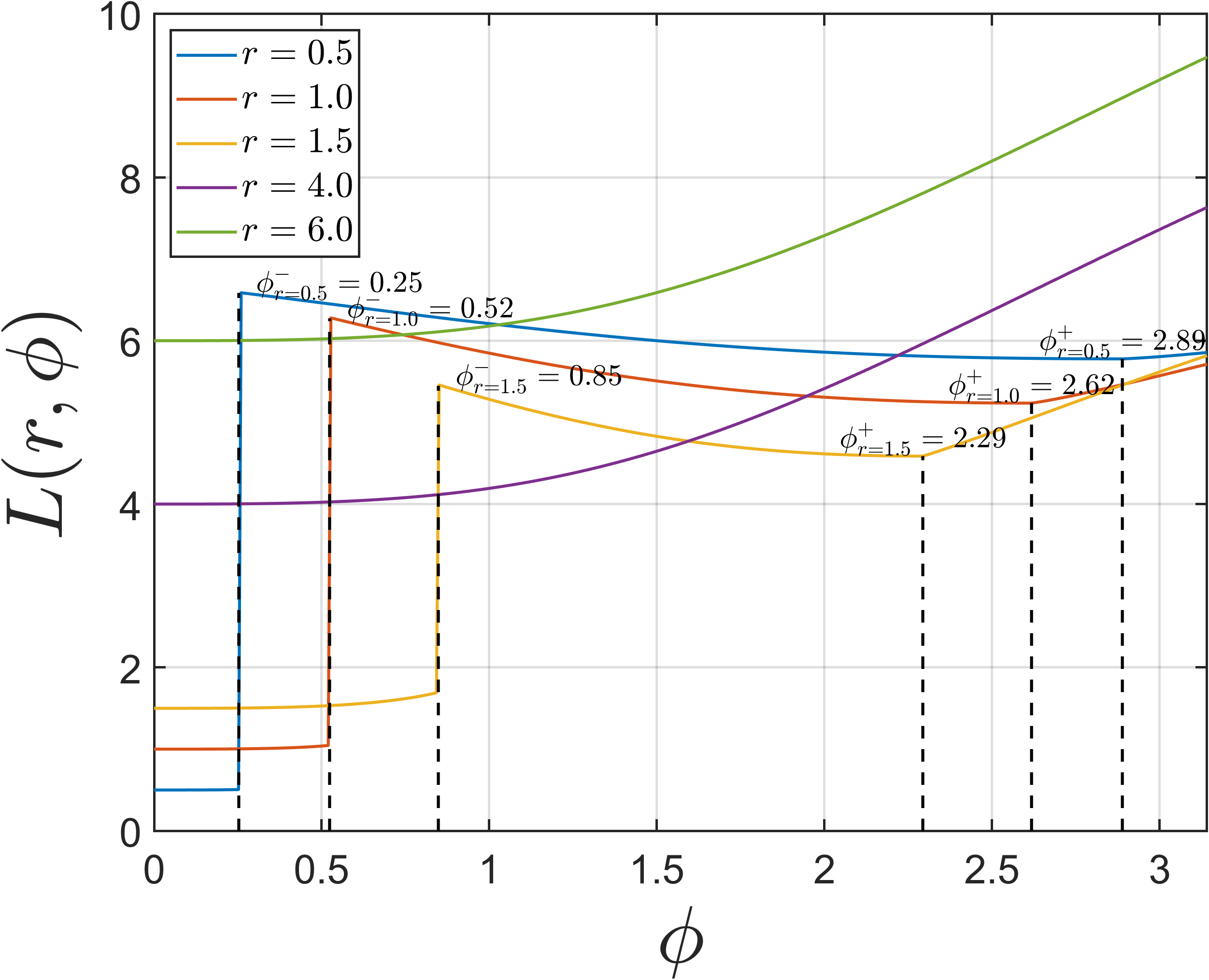}
        \caption{}
        \label{fig:subfig2}
    \end{subfigure}
    \caption{(a) The surface of $L(r,\phi)$ for $\rho=1$. The black dashed lines indicate the cross-sections at $r=0.5, 1, 1.5, 4, 6$. 
    (b) The corresponding variations of $L(r,\phi)$ with respect to $\phi$.}
    \label{fig:0003}
\end{figure}
\subsection{Max-Consensus Control Mechanism}

After introducing the virtual time variable $T$, our objective is to ensure that all robots achieve simultaneous arrival at the target under a unified virtual time scale. To this end, we employ a max-consensus mechanism to drive the virtual times of individual robots toward agreement, thereby achieving global synchronization. Prior work \cite{9143934} has investigated the use of max-consensus to address simultaneous arrival for multi-agent systems. However, the approach does not account for curvature constraints, and thus it is not readily applicable to nonholonomic mobile robots. Moreover, it cannot guarantee the optimal arrival time. The advantage of adopting the max-consensus mechanism is that even if some robots are initially very close to their targets, they will be regulated by the distributed control law to ``wait for" the ``slowest" robot in the team by moving more distances, and in this way, the distributed control law can ensure simultaneous arrival for all agents at the same final time.

Specifically, for the $i$-th robot, its desired virtual time is defined as
\begin{equation}\label{121212}
    T_i^d = \max_{j\in \mathcal{N}_i} T_j,
\end{equation}
and note that the maximum is attained over only the virtual time from the neighboring robots via communication. To ensure that the actual virtual time variable of robot $i$ is equal to the desired one (i.e., $T_i = T_i^d$), the control input $\omega_i$ must be designed to adjust the robot's state $(r_i,\phi_i)$ so that the discrepancy $|T_i^d-T_i|$ is minimized. We first consider a simple case: when $T_i = T_i^d$, the control law given in (\ref{eq:0006}) can be directly applied. For the case $T_i^d > T_i$, we discuss the following two scenarios:

\textbf{Case 1:} $\boldsymbol{\zeta}_i \in \mathcal{D}^c$. According to Theorem \ref{th:0001}, the virtual time variable $T_i$ is strictly increasing with respect to the bearing angle $\phi_i \in [0, \pi]$. Treating $T_i$ as a univariate function of $\phi_i$, the inverse function theorem guarantees the existence of a strictly monotone inverse function, which preserves the monotonicity of the original function. Therefore, the inverse function $T_i^{-1}(T_i^d)$ yields the desired bearing angle $\phi_i^d$. However, since $\phi_i$ is restricted to the interval $[0, \pi]$, if $T_i(\pi) < T_i^d$, no solution exists for the inverse function. In this case, we set the desired bearing angle to $\phi_i^d = \pi$. From (\ref{0003}), the corresponding desired heading is $\theta_i^d = \arg(\boldsymbol{p}_i - \boldsymbol{p}_i^d) - \phi_i^d$, and the heading error is defined as $\delta_i = \theta_i^d - \theta_i$. 
A saturated proportional controller is adopted to regulate the heading error, following \cite{doi:10.2514/1.G004074,doi:10.2514/1.G001719}. The control law is $\omega_i = \mathrm{Sat}_{a}^{b}(k_\theta \delta_i)$, where $k_\theta$ is the control gain and $a=-\bar{\omega}_i,\, b=\bar{\omega}_i$. The saturation function $\mathrm{Sat}_a^b:\mathbb{R}\to\mathbb{R}$ is defined as $\mathrm{Sat}_a^b(x)=x$ for $x\in[a,b]$, $\mathrm{Sat}_a^b(x)=a$ for $x<a$, and $\mathrm{Sat}_a^b(x)=b$ for $x>b$.

\textbf{Case 2:} $\boldsymbol{\zeta}_i \in \mathcal{D}$. Within the curvature-constrained region, the virtual time variable $T_i$ exhibits only local monotonicity, which is undesirable for control design. To prevent potential instability or discontinuous behavior in this region, the control law given in (\ref{eq:0006}) is adopted, aiming to drive the robot rapidly out of the curvature-constrained region and thereby restore system stability.

In summary, the distributed simultaneous-arrival control law can be expressed as
\begin{equation}
\omega_i =
\begin{cases}
\mathrm{Sat}_{-\bar{\omega}_i}^{\bar{\omega}_i} \big( k_\theta (\theta_i^d - \theta_i) \big), & \boldsymbol{\zeta}_i \in \mathcal{D}^c \text{ and } T_i \neq T_i^d, \\
\omega^*(r_i, \phi_i), & \text{otherwise}.
\end{cases}
\end{equation}

\subsection{Convergence Analysis}

Under the aforementioned hybrid control strategy, which combines the optimal Dubins-based control with the saturated proportional controller, we next show that all robots will  arrive at the target simultaneously if at least one robot succeeds in reaching its target.

\begin{theorem}\label{thm:002}
    \emph{For any distinct $i,j\in \mathbb{Z}_1^N$, there does not exist a time $t$ such that $\|\boldsymbol{p}_i(t) - \boldsymbol{p}_i^d\| = 0$ while $\|\boldsymbol{p}_j(t) - \boldsymbol{p}_j^d\|> 0$. In other words, when robot $i$ reaches its target, robot $j$ must have also arrived at its respective target for all  $j \ne i$.}
\end{theorem}

\begin{proof}
    Suppose that at time $t_1>0$, at least one robot reaches its target, and let $\mathcal{I}\ne \emptyset$ denote the set of robots that reach their target, and the set of robots that have not reached their targets is denoted by $\mathcal{J} = \mathcal{I}^c$. We prove this by contradiction. Suppose that at $t=t_1$, there is at least one robot that has not reached its target; namely, $\mathcal{J}$ is non-empty. Since the communication graph $\mathcal{G}$ is connected, there must exist $i \in \mathcal{I}$ and $j \in \mathcal{J}$ such that $j \in \mathcal{N}_i$. Namely, robot $i$ reaches its target at time $t_1$ (i.e., $\|\boldsymbol{p}_i(t_1) - \boldsymbol{p}_i^d\| = 0$). Then we must have $\lim_{t\to t_1} \phi_i = 0$, since the bearing angle $\phi_i\to 0$ as the robot reaches the target. However, robot $j$ has not yet reached its target, so we have $T_i^d - T_i \geq T_j - T_i > 0$ by \eqref{121212}. This implies $\lim_{t\to t_1} \phi_i^d = \lim_{t\to t_1} T_i^{-1}(T_i^d) > \lim_{t\to t_1} T_i^{-1}(T_i)$ $=\lim_{t\to t_1} \phi_i$, indicating that the bearing angle of robot $i$ cannot reach zero as it approaches the target, a contradiction. 
\end{proof}

    \begin{figure}[!h]
    \centering
    \begin{subfigure}{0.23\textwidth}
        \centering
        \includegraphics[width=\linewidth]{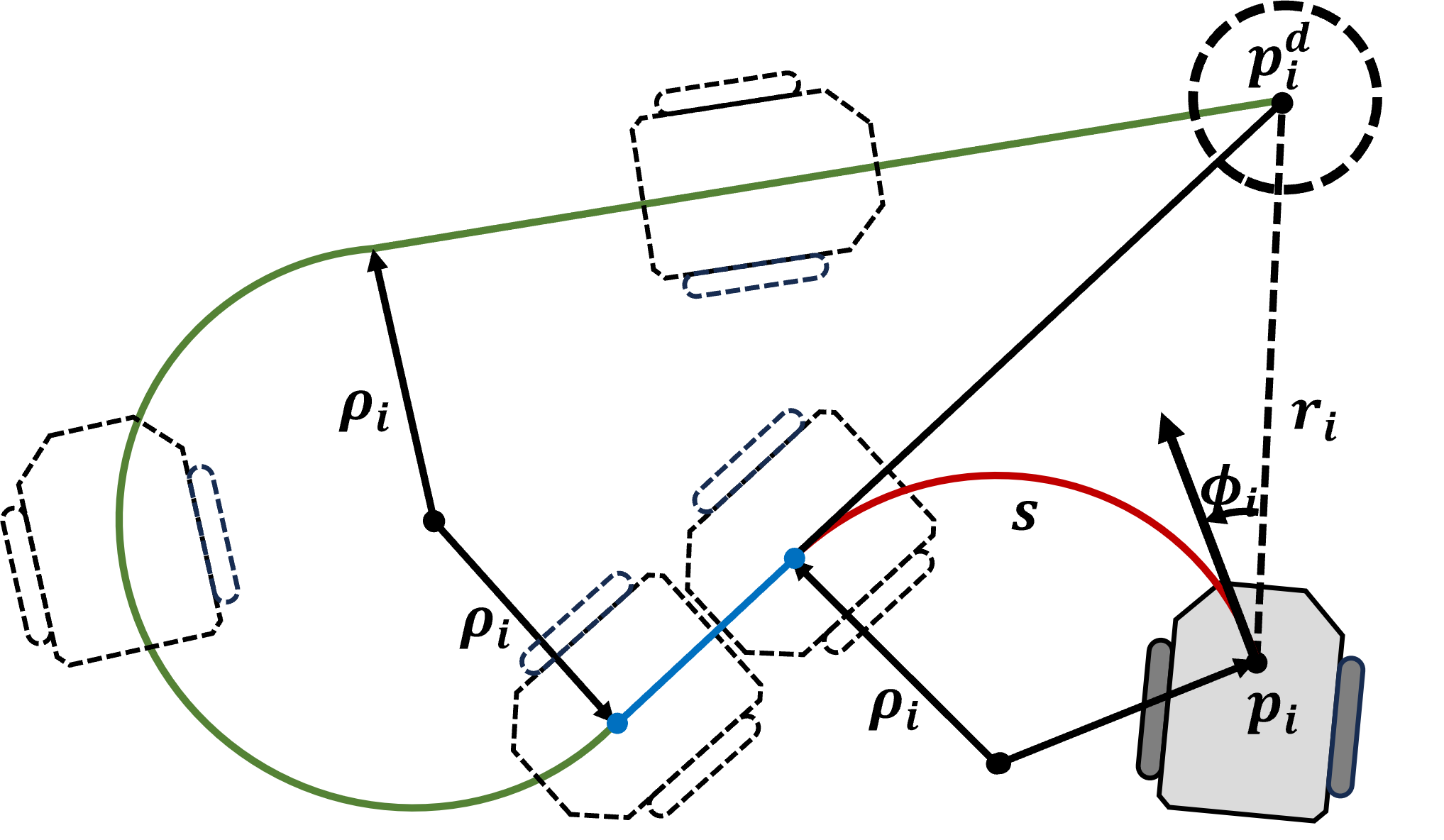}
        \caption{}
        \label{fig:subfig000}
    \end{subfigure}
    \begin{subfigure}{0.23\textwidth}
        \centering
        \includegraphics[width=\linewidth]{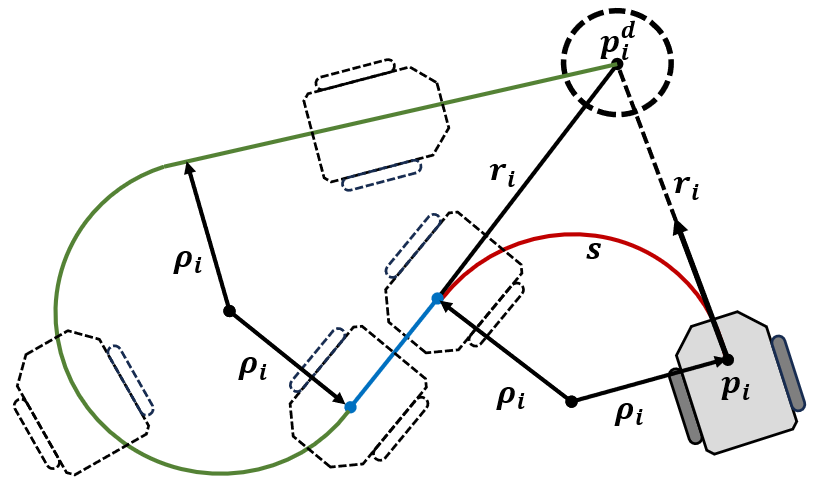}
        \caption{}
        \label{fig:subfig001}
    \end{subfigure}
     \begin{subfigure}{0.23\textwidth}
        \centering
        \includegraphics[width=\linewidth]{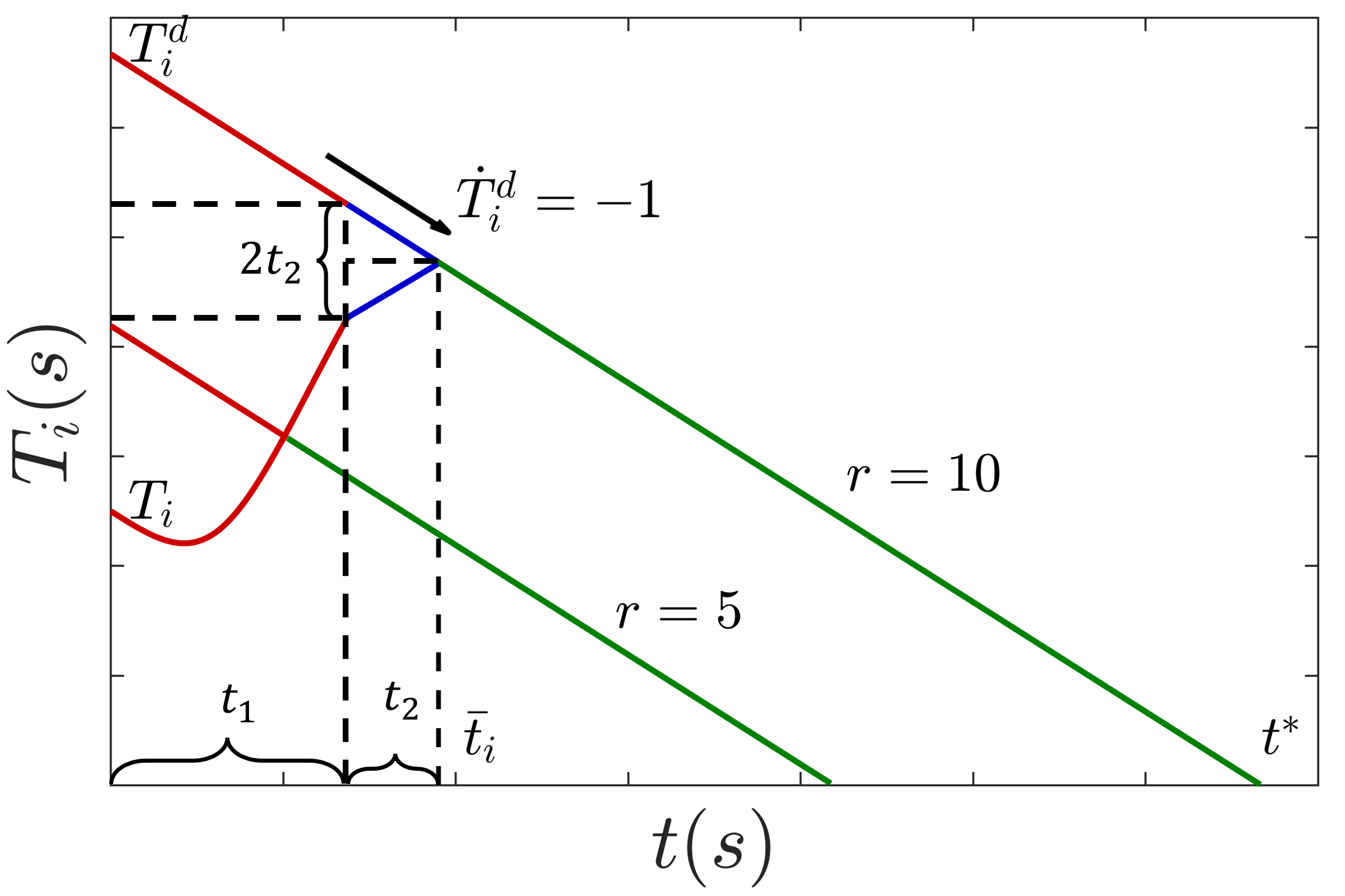}
        \caption{}
        \label{fig:subfig002}
    \end{subfigure}
    \begin{subfigure}{0.23\textwidth}
        \centering
        \includegraphics[width=\linewidth]{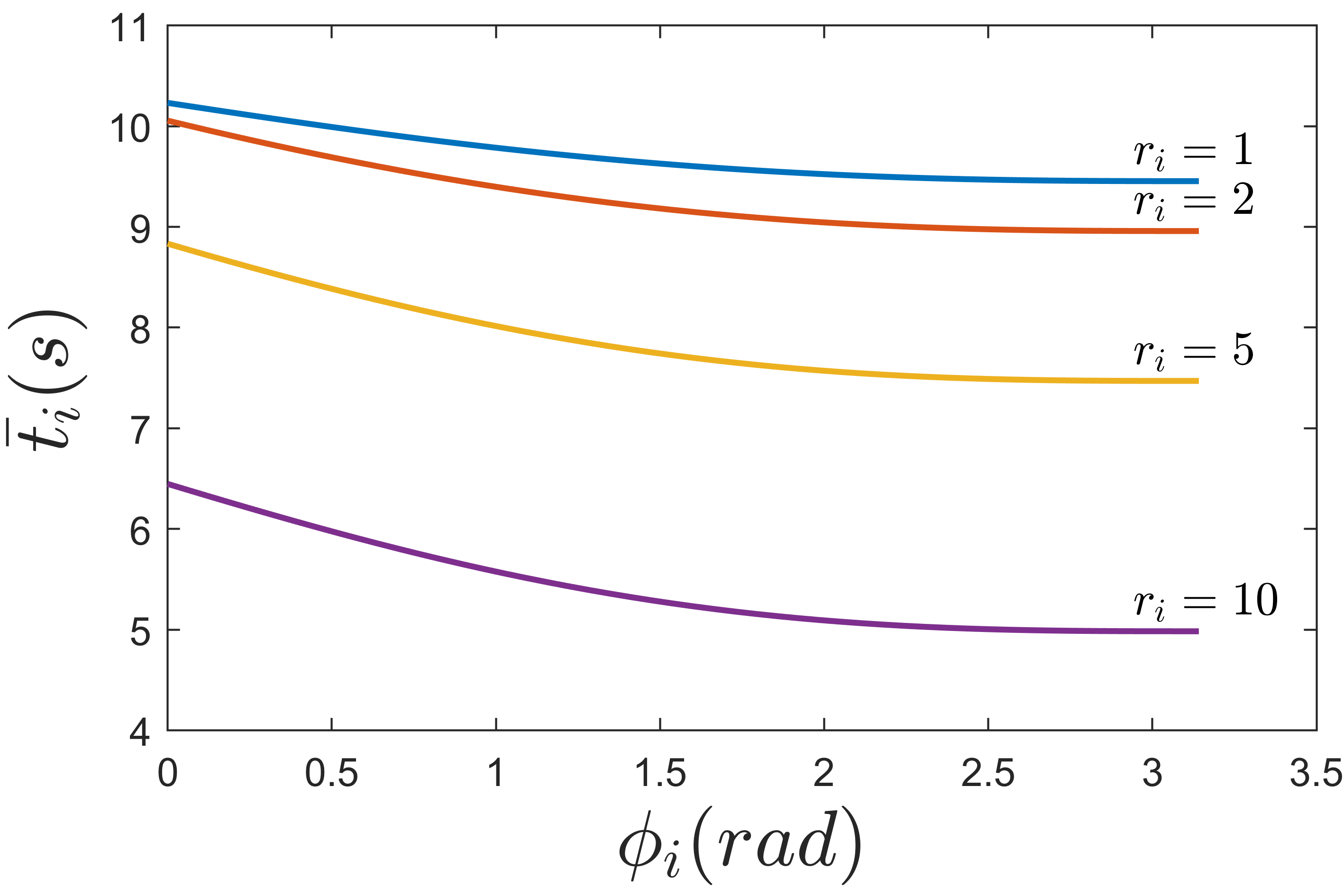}
        \caption{}
        \label{fig:subfig003}
    \end{subfigure}
    \caption{Illustration of the max-consensus process among robots. (a)(b) Ideal trajectories of robot $i$ when the initial state is $\phi_i=0$, where the red, blue, and green curves correspond to the curves of the same colors in (c); (c) Variation of the virtual times for robots $i$ and $j$; (d) Curves showing $\bar{t}_i$ versus the bearing angle $\phi_i$ for robot $i$ with $r_i=1,2,5,10$, while the initial state of robot $j$ is ($r_j=10, \phi_j=\pi$).}
    \label{fig:0004}
\end{figure}
The assumption that both $\mathcal{I}$ and $\mathcal{J}$ are non-empty is invalid. It follows that either $\mathcal{I}$ or $\mathcal{J}$ must be empty. In general, whether $\mathcal{I}=\emptyset$ or $\mathcal{J}=\emptyset$ depends on the initial states of the robots. We present a relatively relaxed theorem here.

\begin{theorem}\label{thm:003}
    \emph{Suppose robot $i^*$ satisfies the initial states $T_{i^*}=\max_{i\in \mathbb{Z}_1^N } T_i$. If the gain $k_\theta\to \infty$ and
    \[
        T_{i^*}\geq \sum_{i=1,i\neq i^*}^N \Bigg(\frac{T_i^d}{2}-\frac{r_i}{2v_i}-\frac{2\arctan{\frac{\rho_i}{r_i}}}{v_i}\Bigg) + 2\rho_{i^*},
    \]
    then all robots can simultaneously reach their target positions at $t^*=T_{i^*}$, and $t^*$ is the optimal arrival time.}
\end{theorem}

\begin{proof}
The purpose of taking $k_\theta\to \infty$ is to make the saturated proportional control closely approximate the optimal control, which requires $k_\theta\geq\frac{\bar{\omega}_i}{\delta_i}$ or $k_\theta\leq-\frac{\bar{\omega}_i}{\delta_i}$ at all times and facilitates analysis.

First, we show the following result: For robot $i$, if the neighbor with the largest virtual time is robot $j$, and $\boldsymbol{\zeta}_i \in \mathcal{D}^c$ is satisfied before reaching consensus with $j$, then we have
\begin{equation}\label{1515}
\bar{t}_i \leq \frac{T_i^d}{2}-\frac{r_i}{2v_i}-\frac{2\arctan\frac{\rho_i}{r_i}}{v_i},
\end{equation}
where $\bar{t}_i$ is the time when $T_i$ reaches $T^d_i$ under the condition that $T^d_i = T_{i^*}$. On one hand, robot $j$ satisfies $\dot{T}_j = -1$ under optimal control; on the other hand, since $T_j > T_i$ (if $T_j = T_i$, then $\bar{t}_i = 0$, and the result clearly holds), it follows that $\delta_i > 0$, and the heading rate control is $\omega_i = \lim_{k_\theta\to \infty}\mathrm{Sat}_{-\bar{\omega}_i}^{\bar{\omega}_i}\left( k_\theta \delta_i \right) = \pm \bar{\omega}_i$. By converting the time relation $T_i =\frac{L_i}{v_i}$ into a geometric form, $\bar{t}_i$ can be determined based on the path length before consensus is reached (i.e., the sum of the red circular arc length $s$ and the blue line segment length $l$ in Fig.~\ref{fig:subfig000}). Specifically, as shown in Fig.~\ref{fig:subfig002}, $\bar{t}_i = t_1 + t_2$, where $t_1 = \frac{s}{v_i}$, and $t_2 = \frac{l}{v_i} = \frac{T^d_i(t_1) - T_i(t_1)}{2}$. By combining this with equation \eqref{eq:00017}, an analytical expression for $\bar{t}_i$ in terms of the initial bearing angle $\phi_i$ is obtained. Using a differentiation method similar to that in Theorem \ref{th:0001}, we derive $\frac{\partial \bar{t}_i}{\partial \phi_i} \leq 0$, which implies that the mapping from the initial bearing angle $\phi_i$ to $\bar{t}_i$ is decreasing over $\phi_i \in [0, \pi]$ (see Fig.~\ref{fig:subfig003}). Therefore, the maximum value of $\bar{t}_i$ occurs at $\phi_i = 0$ (see Fig.~\ref{fig:subfig001}), yielding the inequality \eqref{1515}.

Then, since the communication graph $\mathcal{G}$ is connected, it admits a spanning tree. Let robot $i^*$ be chosen as the root of this spanning tree, serving as the reference node for propagating information through the network. Then, for each neighbor $j \in \mathcal{N}_{i^*}$, the time required to reach consensus satisfies $t_j\leq\bar{t}_j$, and the corresponding estimate can be propagated recursively along the spanning tree. In the worst-case scenario, $\mathcal{G}$ degenerates into a chain topology in which each robot communicates only with its immediate successor. In this case, the total time required for the network to achieve consensus is bounded by $t \leq \sum_{i=1,i\neq i^*}^N \bar{t}_i$. Furthermore, since $T_{i^*} - t \geq 2\rho_{i^*} \geq 2\rho_{i^*}|\sin{\phi_{i^*}}|$, robot $i^*$ remains in the curvature-feasible region $\boldsymbol{\zeta}_{i^*} \in \mathcal{D}^c$.

Finally, according to the optimality of Dubins paths, the time for robot $i^*$ to reach its target from the initial position is equal to $T_{i^*}$. Combining the above results, $t^*$ is the optimal arrival time as it satisfies $t^* \geq T_{i^*}$.
\end{proof}

\begin{remark}
    Our work primarily focuses on the distributed control algorithm based on the max-consensus protocol; therefore, collision avoidance among robots is not discussed in detail. Nevertheless, our method can be seamlessly integrated with existing internal collision avoidance strategies \cite{5746538,4543489}. In the third simulation, we implement a simple trigger-based cooperative collision avoidance mechanism to demonstrate the extensibility of the simultaneous arrival algorithm. Each robot is modeled as a disk of radius $r_c$. For any two distinct robots $i,j \in \mathbb{Z}_1^N$, collision avoidance is activated if the following conditions are satisfied simultaneously:  

1) $\|\boldsymbol{p}_i - \boldsymbol{p}_j\| < d_s,$ where $d_s > \sqrt{2\rho_i r_c + r_c^2} + \sqrt{2\rho_j r_c + r_c^2}$ ensures activation only when the robots are sufficiently close.  

2) $(\boldsymbol{p}_i - \boldsymbol{p}_j)^\top (\dot{\boldsymbol{p}}_i - \dot{\boldsymbol{p}}_j) < 0,$ indicating that the robots are moving toward each other along their relative velocity.  

3) Let the unit heading directions of robots $i$ and $j$ be $\boldsymbol{s}_i = [\cos\theta_i, \sin\theta_i]^\top$ and $\boldsymbol{s}_j = [\cos\theta_j, \sin\theta_j]^\top$, respectively. Define the safety circle of robot $j$ as $\mathcal{C}_j = \{\boldsymbol{q} \in \mathbb{R}^2 : \|\boldsymbol{q} - \boldsymbol{p}_j\| \le r_s\},$ with $r_s = 2r_c$. The anticipated trajectory of robot $i$ is given by $\boldsymbol{r}_i(\lambda) = \boldsymbol{p}_i + \lambda \boldsymbol{s}_i, \lambda \ge 0$. If there exists $\lambda^* > 0$ such that $\|\boldsymbol{r}_i(\lambda^*) - \boldsymbol{p}_j\| \le r_s$, the ray intersects the \emph{safety circle}. Similarly, for trajectory line segments $\boldsymbol{r}_i(\lambda) = \boldsymbol{p}_i + \lambda \boldsymbol{s}_i$ and $\boldsymbol{r}_j(\mu) = \boldsymbol{p}_j + \mu \boldsymbol{s}_j, \lambda,\mu \ge 0$, if there exist $(\lambda^*, \mu^*)$ such that $\boldsymbol{r}_i(\lambda^*) = \boldsymbol{r}_j(\mu^*)$, the trajectories intersect. If either ray–circle or the trajectory intersection holds, a potential collision is predicted.  

When all three conditions are satisfied, robots $i$ and $j$ adjust their headings in the agreed direction. For implementation simplicity, their angular velocities are set to the maximum values: $\omega_i = \bar{\omega}_i, \ \omega_j = \bar{\omega}_j.$

\end{remark}

\begin{figure}[!h] 
    \centering
    \begin{subfigure}{0.23\textwidth}
        \centering
        \includegraphics[width=\linewidth]{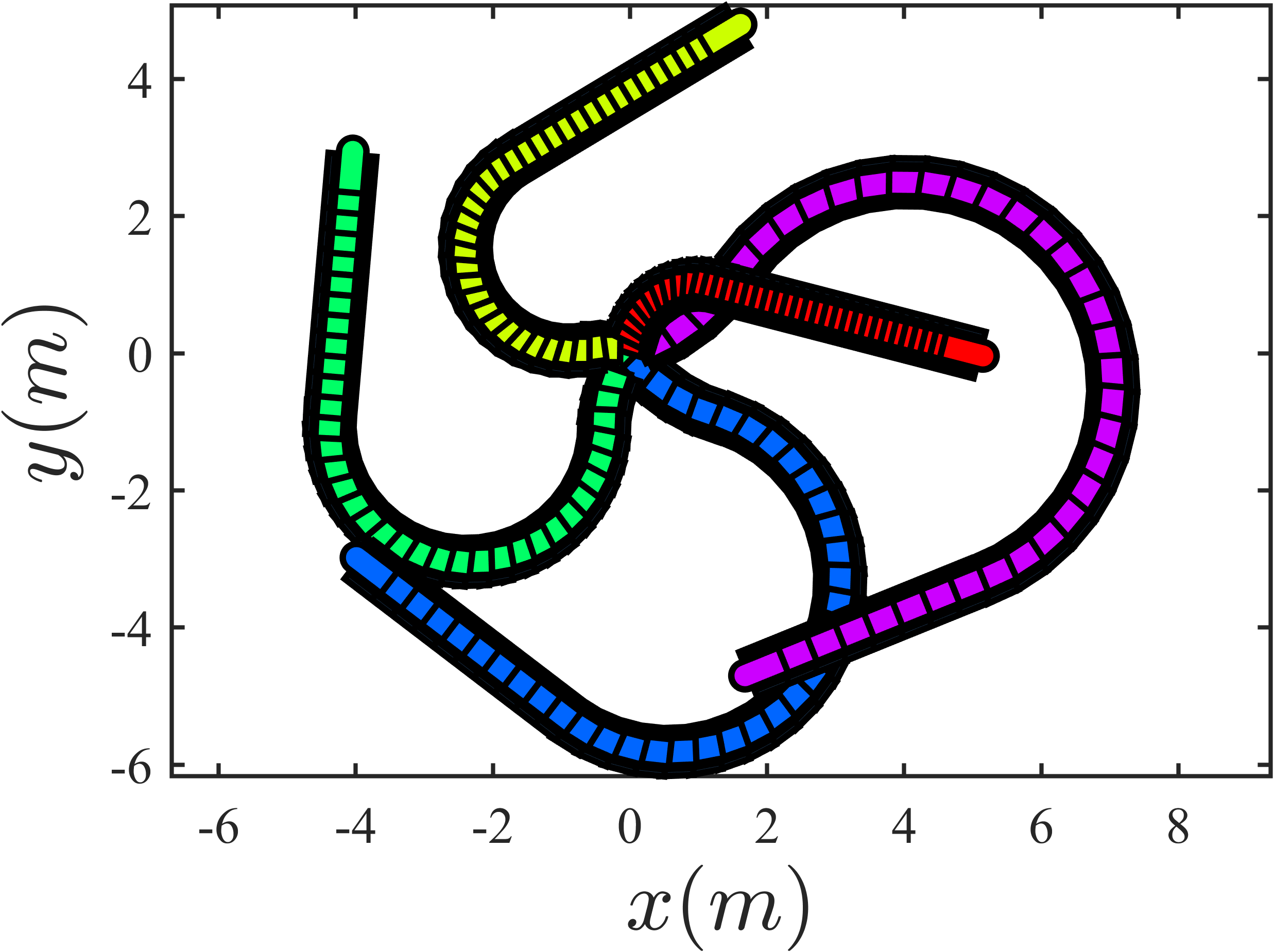}
        \label{fig:s1subfig1}
    \end{subfigure}
    \begin{subfigure}{0.23\textwidth}
        \centering
        \includegraphics[width=\linewidth]{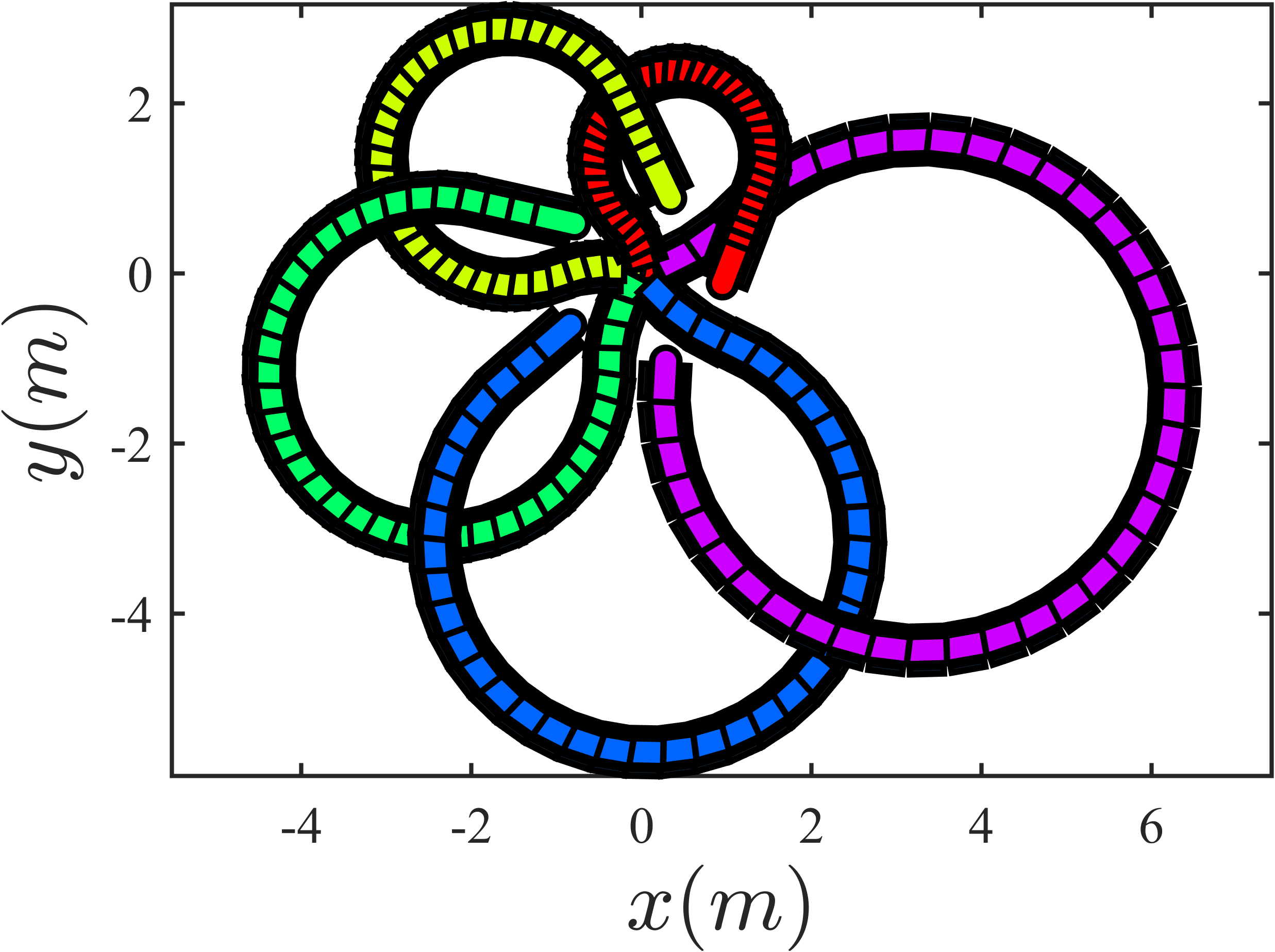}
        \label{fig:s1subfig2}
    \end{subfigure}
    \begin{subfigure}{0.23\textwidth}
        \centering
        \includegraphics[width=\linewidth]{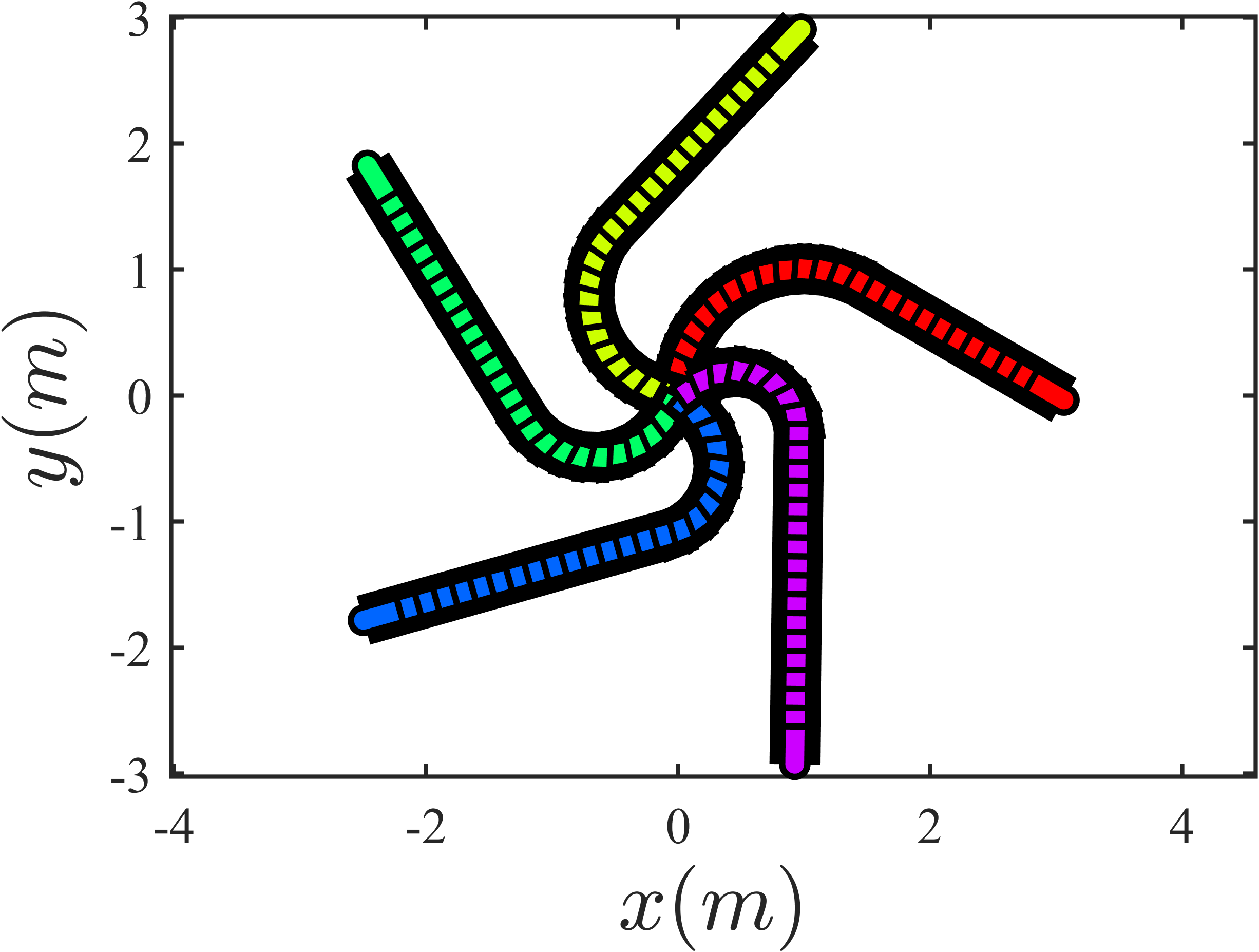}
        \label{fig:s1subfig2}
    \end{subfigure}
    \begin{subfigure}{0.23\textwidth}
        \centering
        \includegraphics[width=\linewidth]{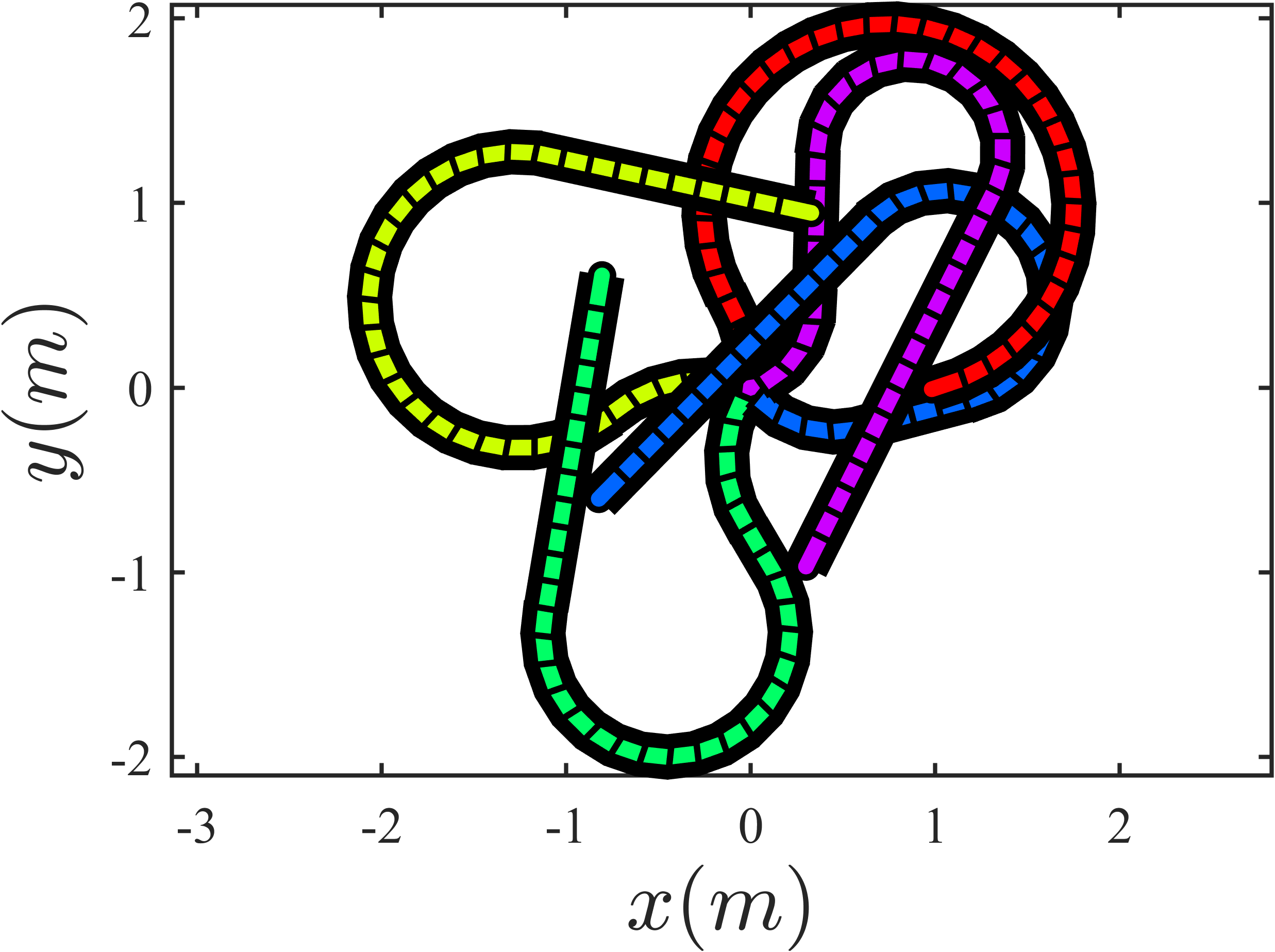}
        \label{f158165}
    \end{subfigure}
\caption{The first simulation results. Initialize $N=5$ robots at the origin. The initial heading angles $\theta_i$ are set by a counterclockwise rotation of $\frac{\pi}{2}$ applied to the vector pointing from the robot to its target: 
$[\cos\theta_i, \sin\theta_i]^\top = E \frac{\boldsymbol{p}^d_i - \boldsymbol{p}_i}{\|\boldsymbol{p}^d_i - \boldsymbol{p}_i\|},$ 
where $E$ is a rotation matrix. Target positions $\boldsymbol{p}^d_i$ are uniformly distributed along a circle of radius $d$ and 
$\boldsymbol{p}^d_i = [d \cos(\psi i),d \sin(\psi i)]^\top, 
\psi = \frac{2\pi}{N+1}, i \in \mathbb{Z}_1^5$. Four experimental groups are conducted with circle radius $d = 5m, 1m, 3m, 1m$, respectively. The control gain is $k_\theta = 100$ and the arrival threshold is $\varepsilon = 0.01m$, i.e., the simulation terminated when $\|\boldsymbol{p}_i(t^*)-\boldsymbol{p}^d_i\|\leq \varepsilon$ for all robots. In the first two groups, the linear velocity is set as $v_i = 1 + 0.5(i-1)m/s$ and the maximum angular velocity as $\bar{\omega}_i = 1rad/s$. In the last two groups, the linear velocity is $v_i = 1m/s$ and the maximum angular velocity is $\bar{\omega}_i = 1 + 0.25(i-1)rad/s$ for all $i \in \mathbb{Z}_1^5$.}
\label{fig:0005}   
\end{figure}

\section{Simulations And Experiments}
\subsection{Simulations}
We adopt a ring graph as the communication topology, in which each Dubins-car-modelled robot only communicates with its two immediate neighbors. This topology ensures very low communication and computational overhead. In the first simulation, we set $N=5$ robots to reach their respective target positions simultaneously. As shown in Fig.~\ref{fig:0005} and Fig.~\ref{fig:0006}, regardless of the initial states, all robots successfully achieved consensus on the virtual time variables and arrived at the target points simultaneously. In the second simulation, we increased the number of robots to $N=50$, requiring them to reach the target point simultaneously. As illustrated in Fig.~\ref{fig:0007}, all robots are still able to arrive at the target point at the same time. In the third simulation, we used $N=5$ robots to illustrate the effectiveness and safety of the proposed algorithm when integrated with an internal collision avoidance mechanism. As shown in Fig.~\ref{fig:0008}, the robots successfully achieve simultaneous arrival while maintaining collision-free trajectories.

\begin{figure}[!htbp] 
    \centering
    \begin{subfigure}{0.115\textwidth}
        \centering
        \includegraphics[width=\linewidth]{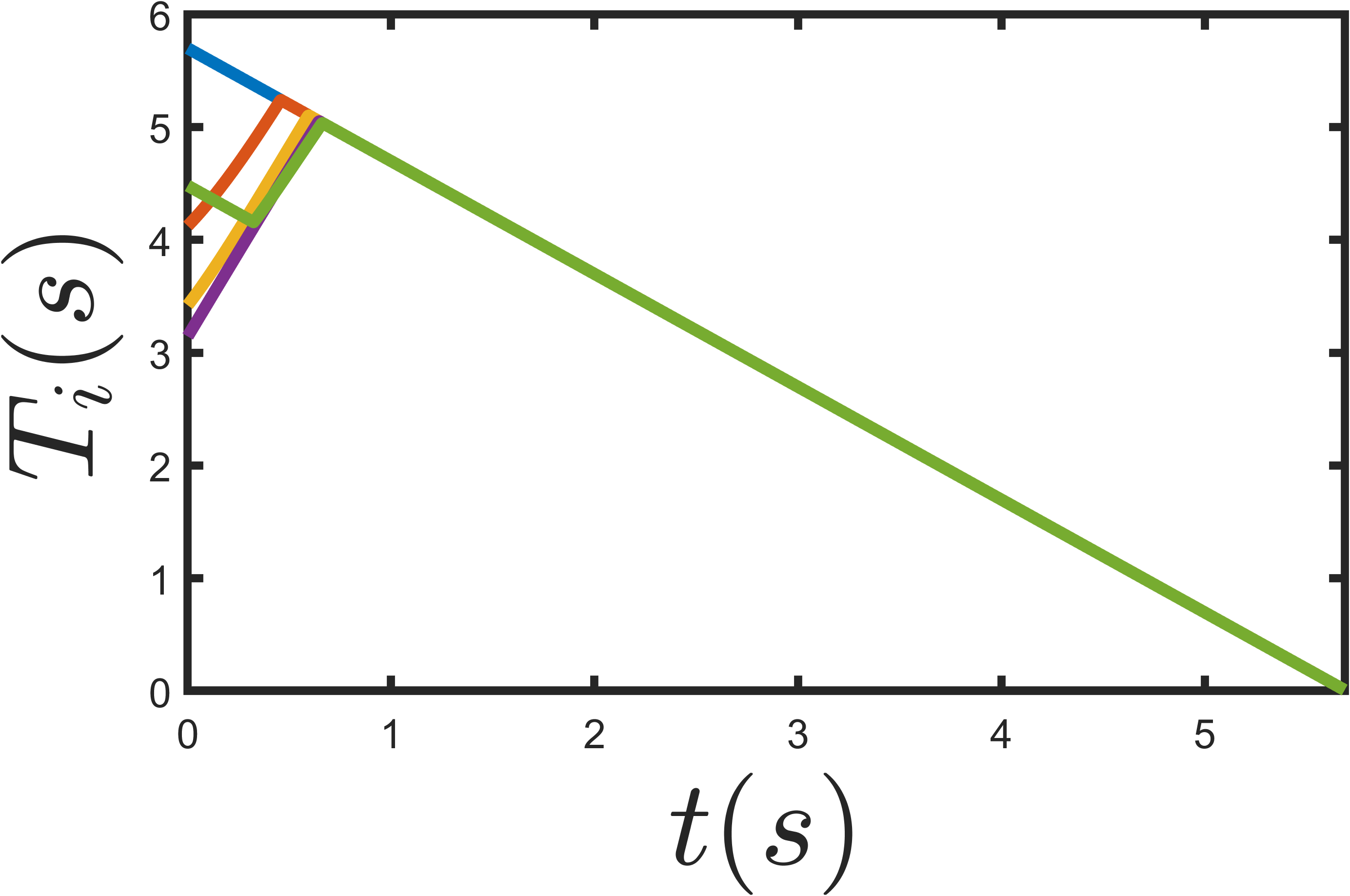}
        \label{fig:s1subfig1}
    \end{subfigure}
    \begin{subfigure}{0.115\textwidth}
        \centering
        \includegraphics[width=\linewidth]{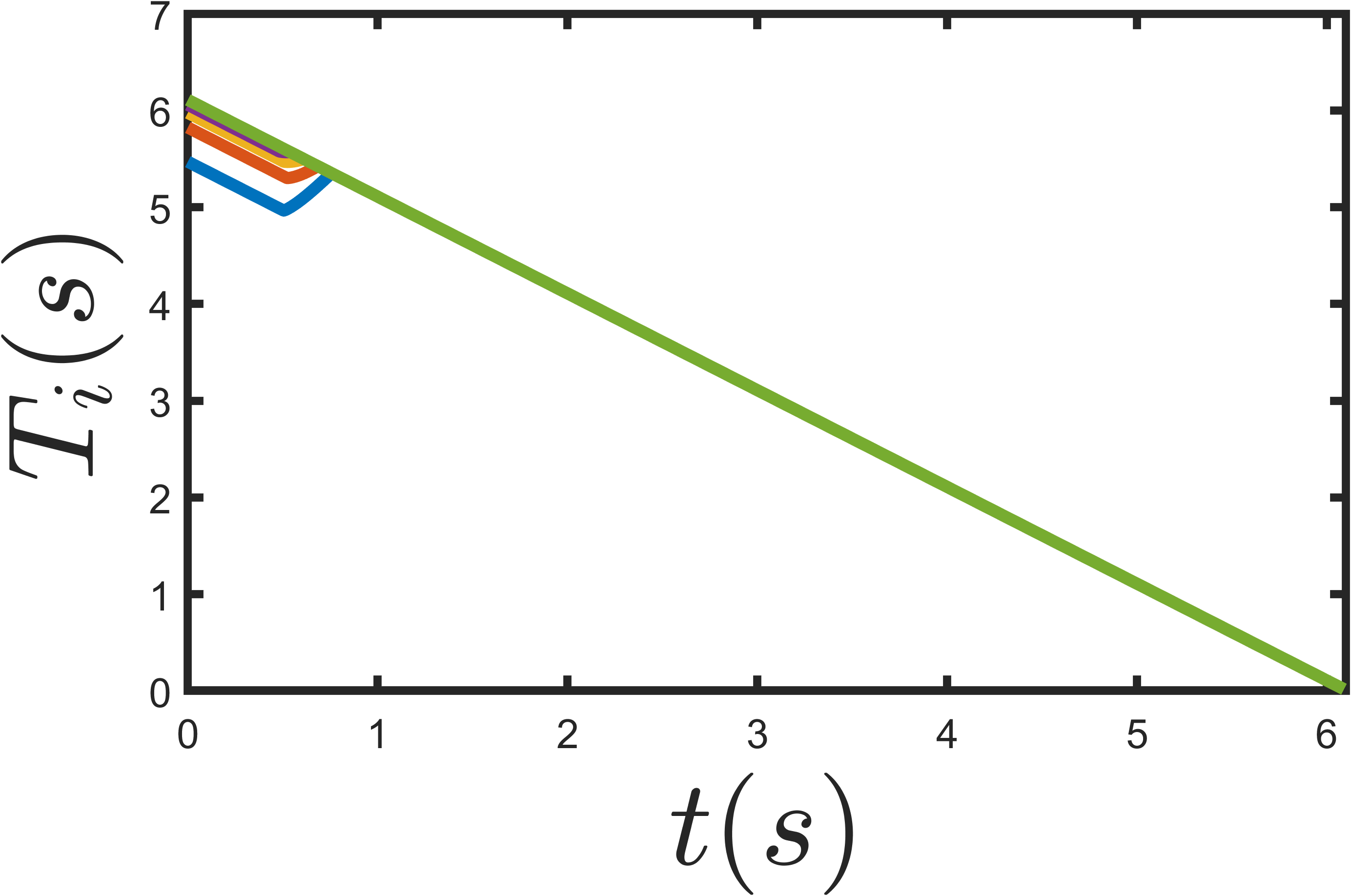}
        \label{fig:s1subfig2}
    \end{subfigure}
    \begin{subfigure}{0.115\textwidth}
        \centering
        \includegraphics[width=\linewidth]{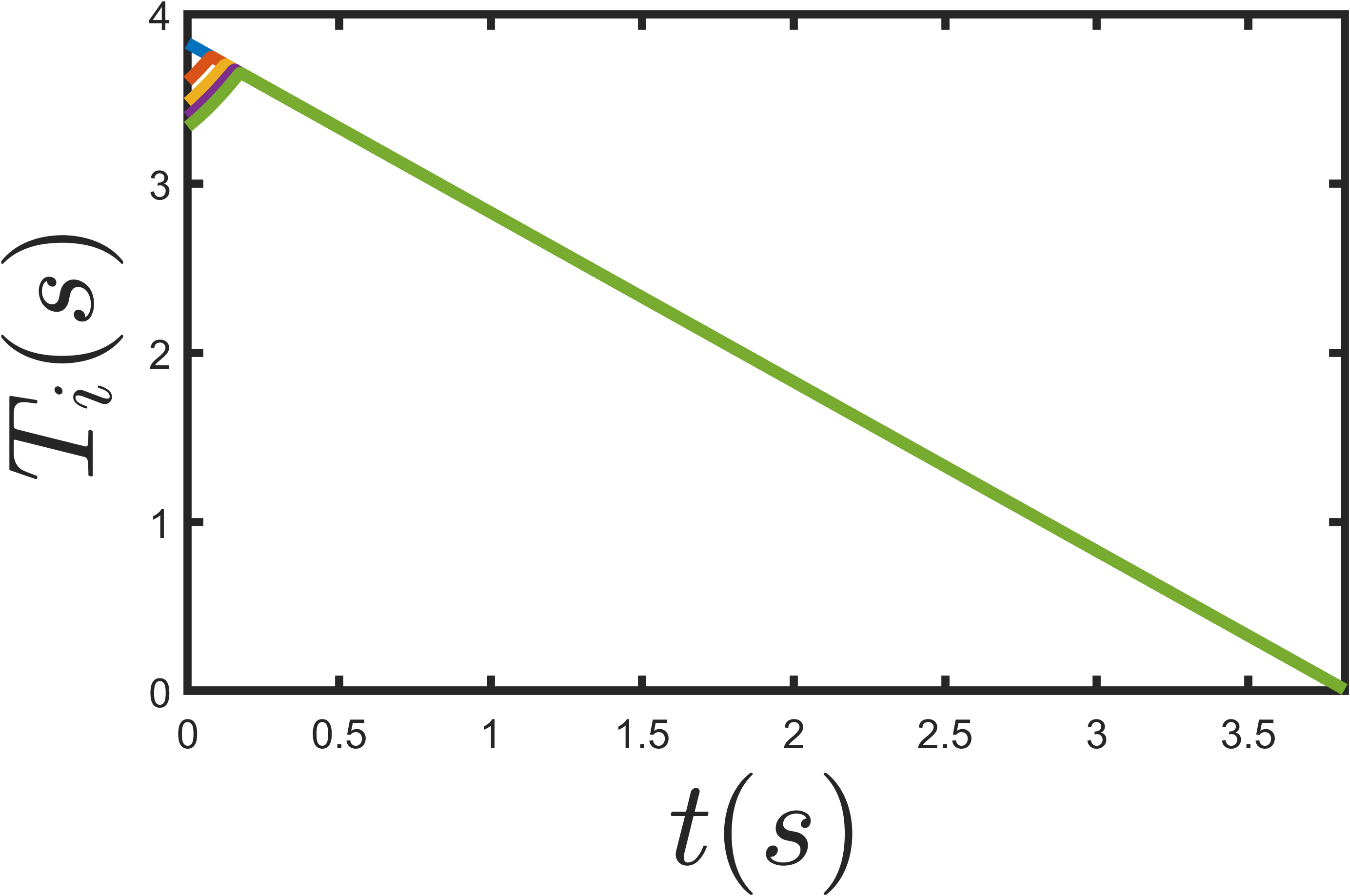}
        \label{fig:s1subfig2}
    \end{subfigure}
    \begin{subfigure}{0.115\textwidth}
        \centering
        \includegraphics[width=\linewidth]{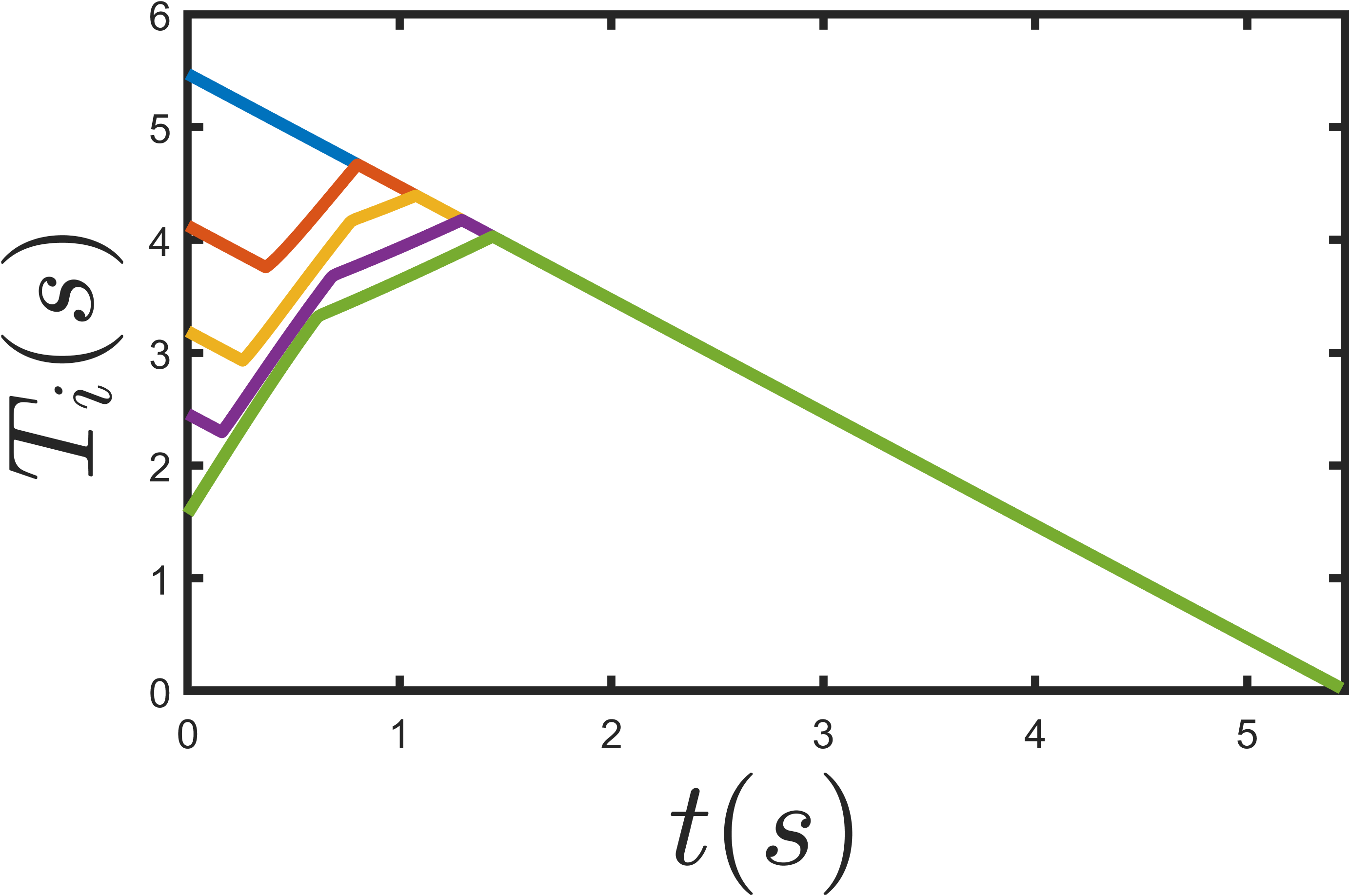}
        \label{fig:s1subfig2}
    \end{subfigure}
    
    \begin{subfigure}{0.115\textwidth}
        \centering
        \includegraphics[width=\linewidth]{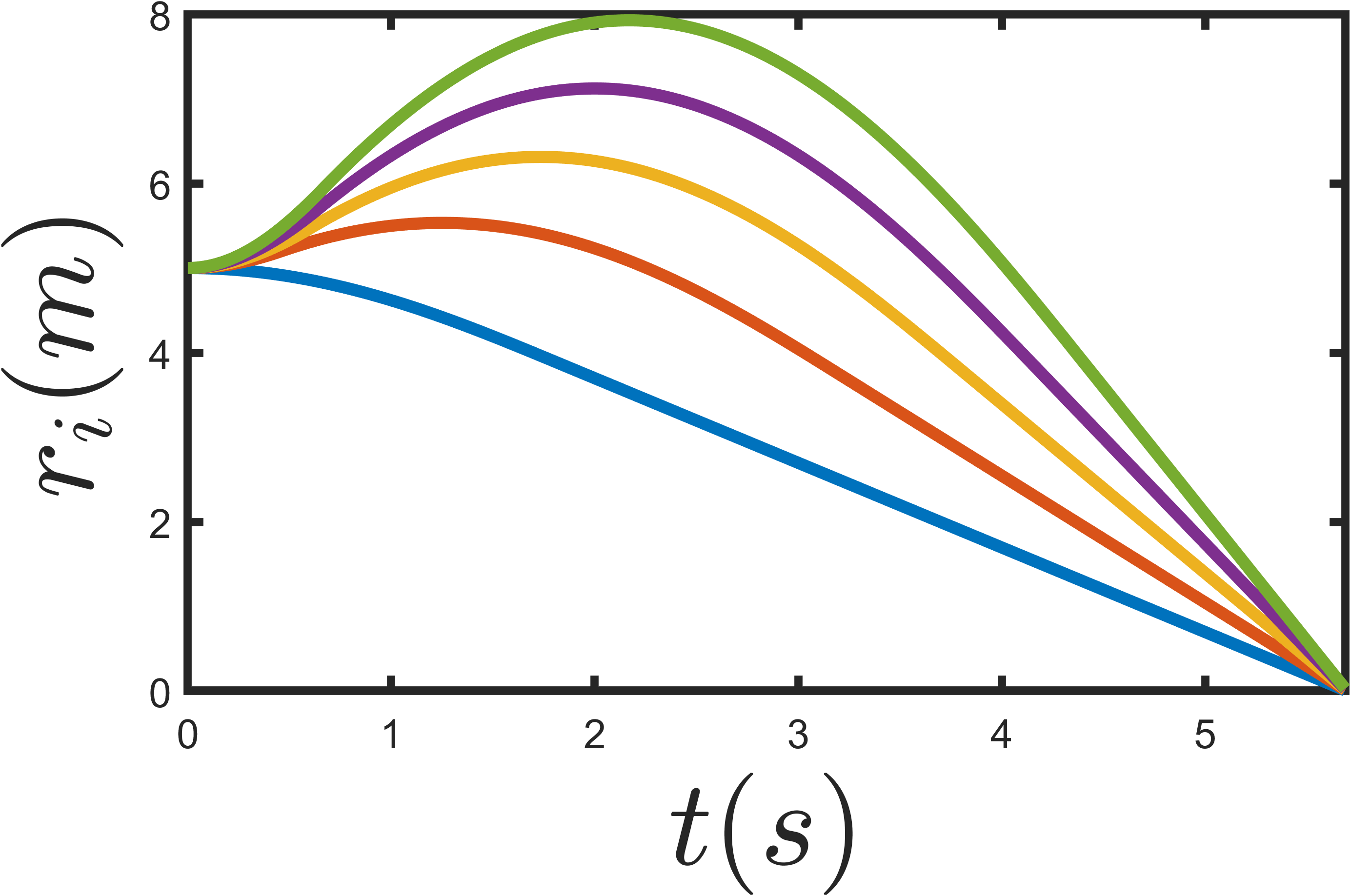}
        \label{fig:s1subfig2}
    \end{subfigure}
    \begin{subfigure}{0.115\textwidth}
        \centering
        \includegraphics[width=\linewidth]{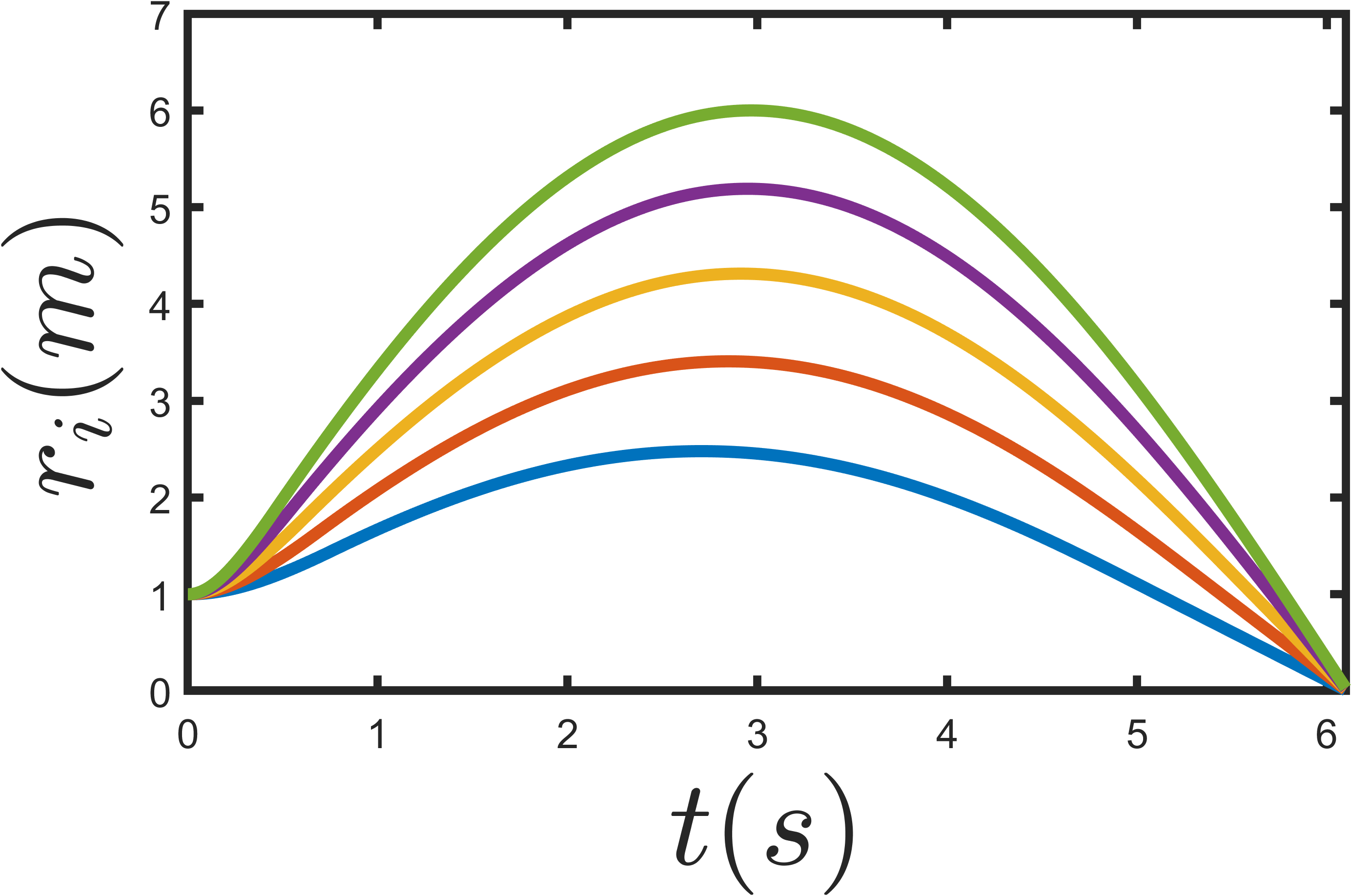}
        \label{15516}
    \end{subfigure}
    \begin{subfigure}{0.115\textwidth}
        \centering
        \includegraphics[width=\linewidth]{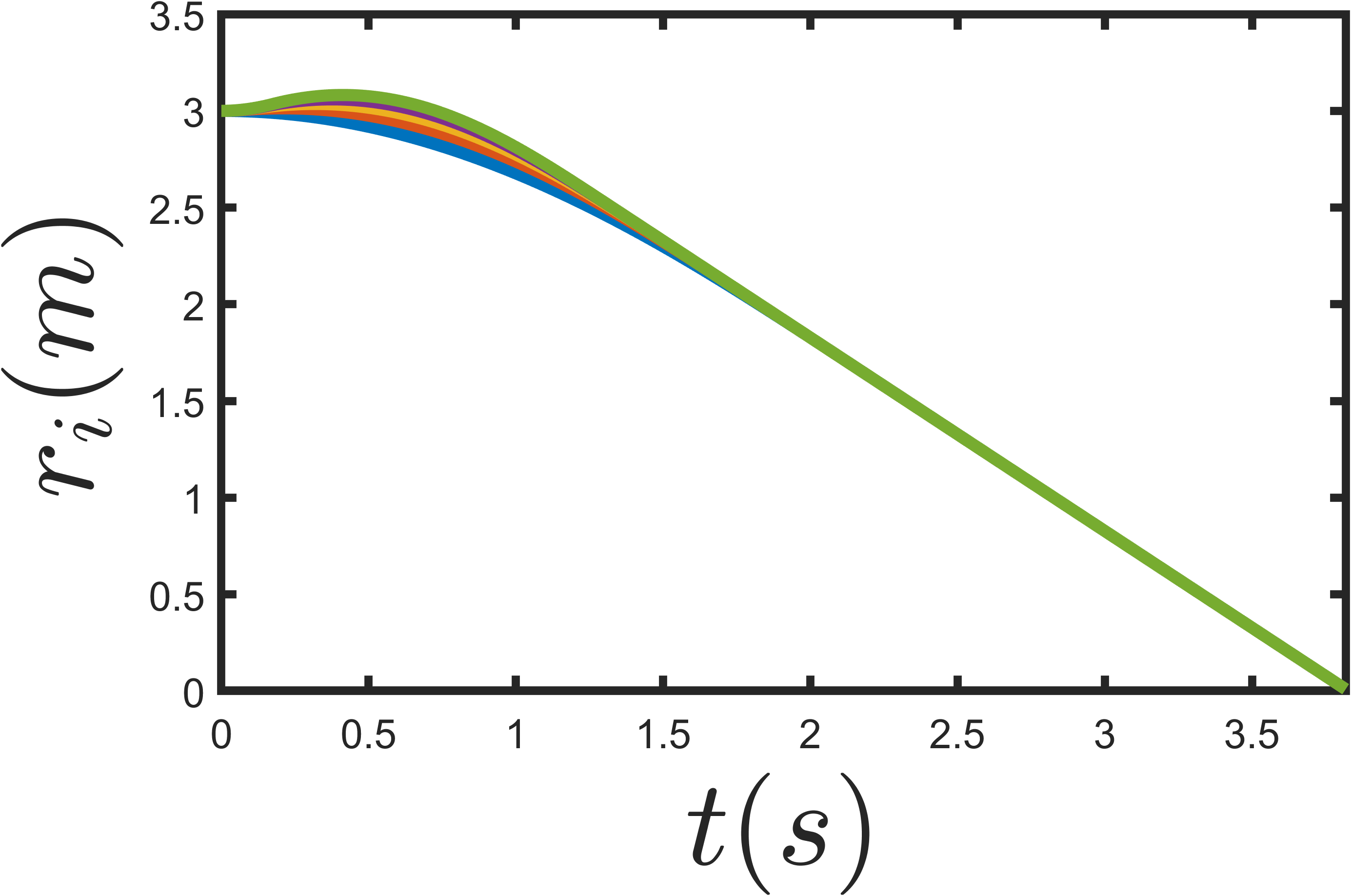}
        \label{fig:s1subfig2}
    \end{subfigure}
    \begin{subfigure}{0.115\textwidth}
        \centering
        \includegraphics[width=\linewidth]{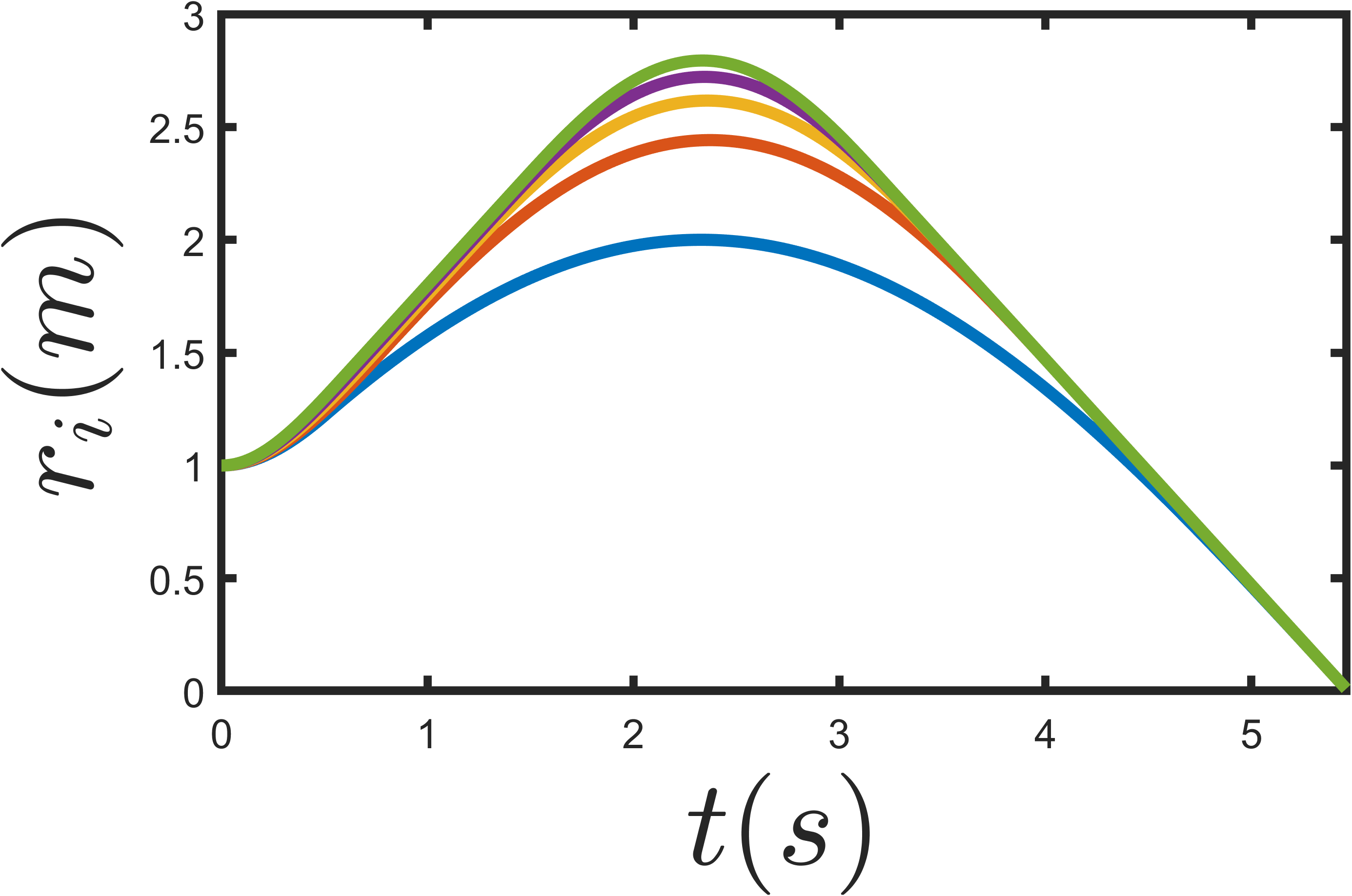}
        \label{fig:s1subfig2}
    \end{subfigure}
    
    \begin{subfigure}{0.115\textwidth}
        \centering
        \includegraphics[width=\linewidth]{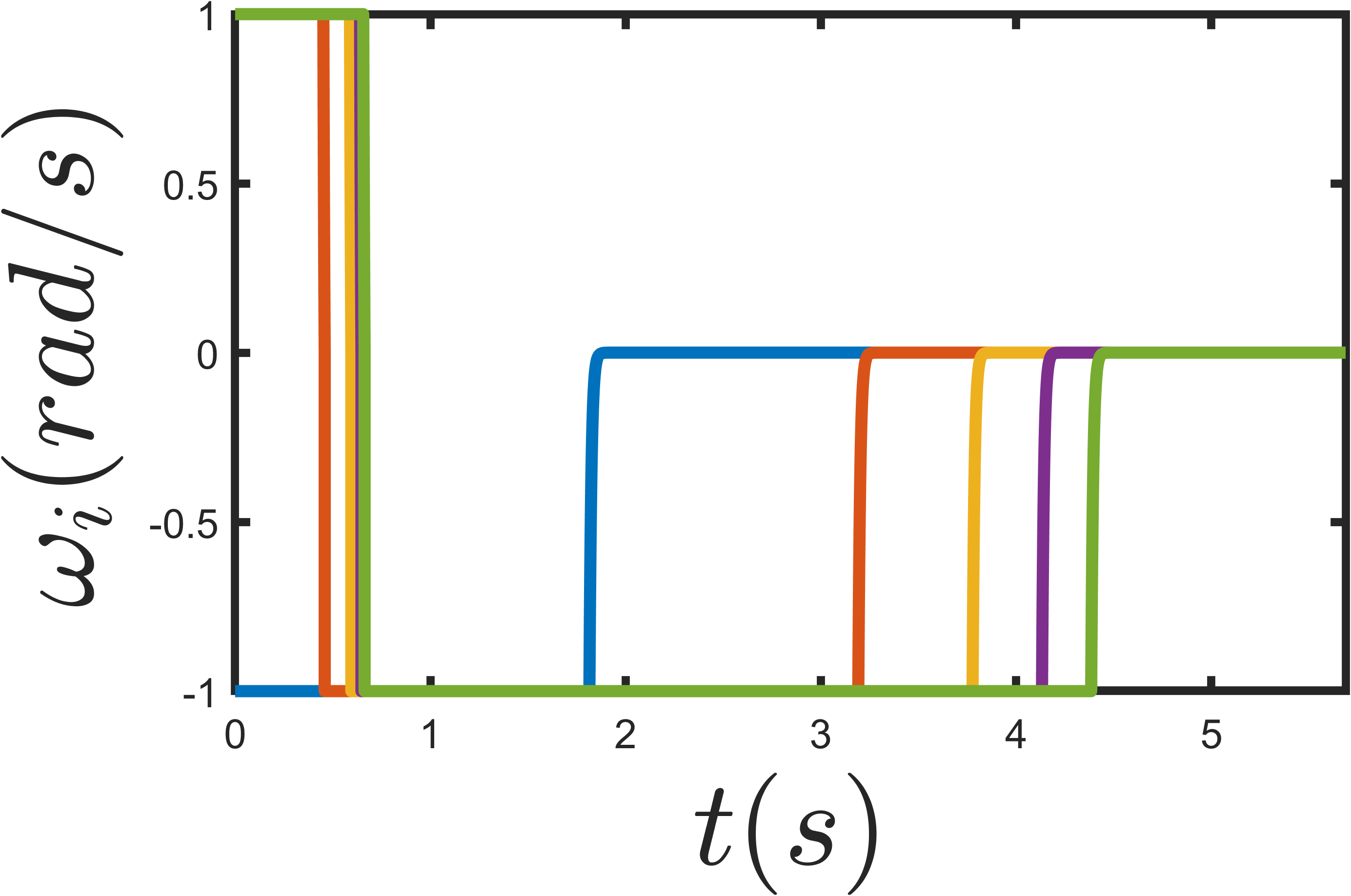}
        \label{fig:s1subfig2}
    \end{subfigure}
    \begin{subfigure}{0.115\textwidth}
        \centering
        \includegraphics[width=\linewidth]{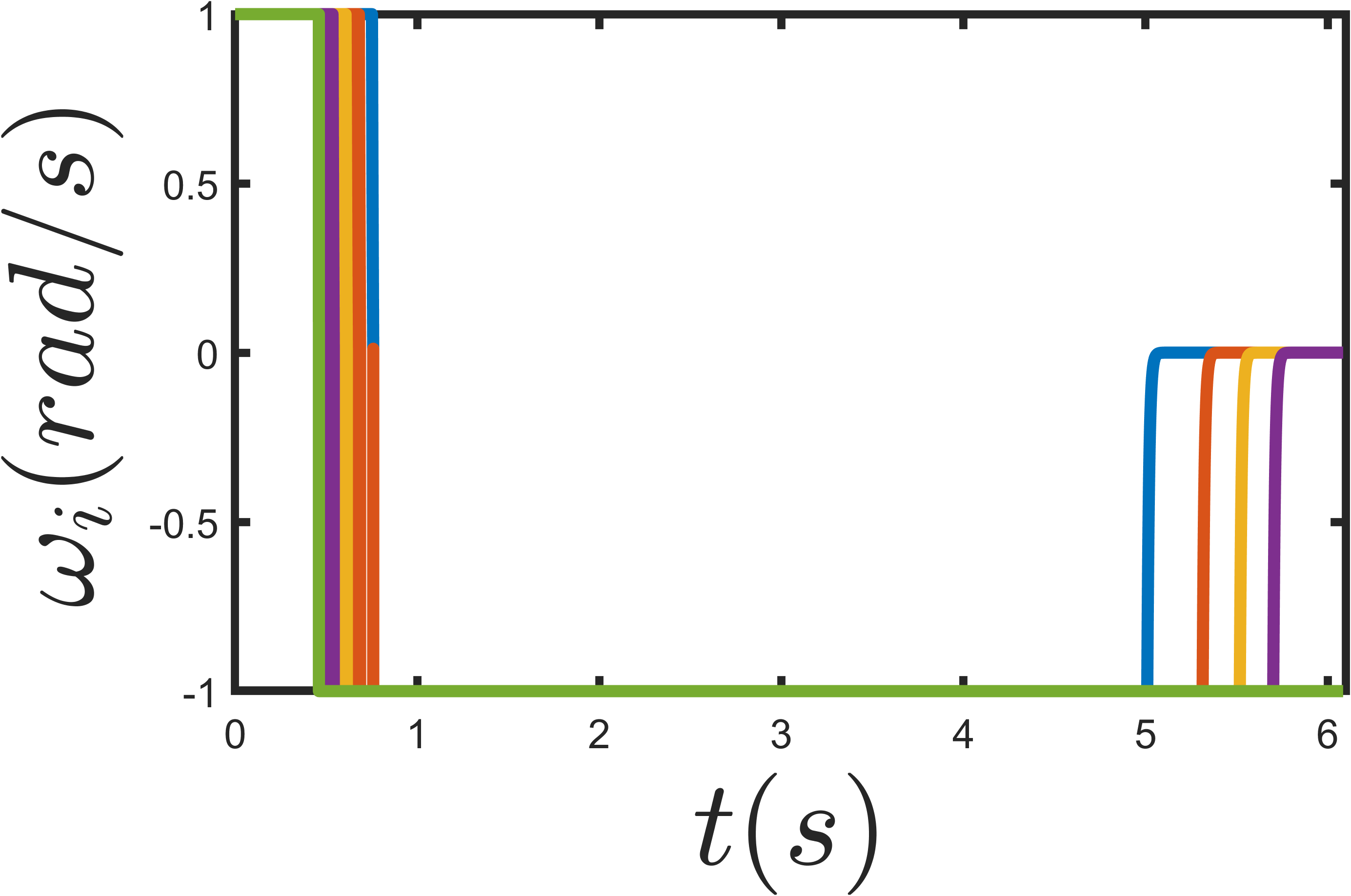}
        \label{fig:s1subfig2}
    \end{subfigure}
    \begin{subfigure}{0.115\textwidth}
        \centering
        \includegraphics[width=\linewidth]{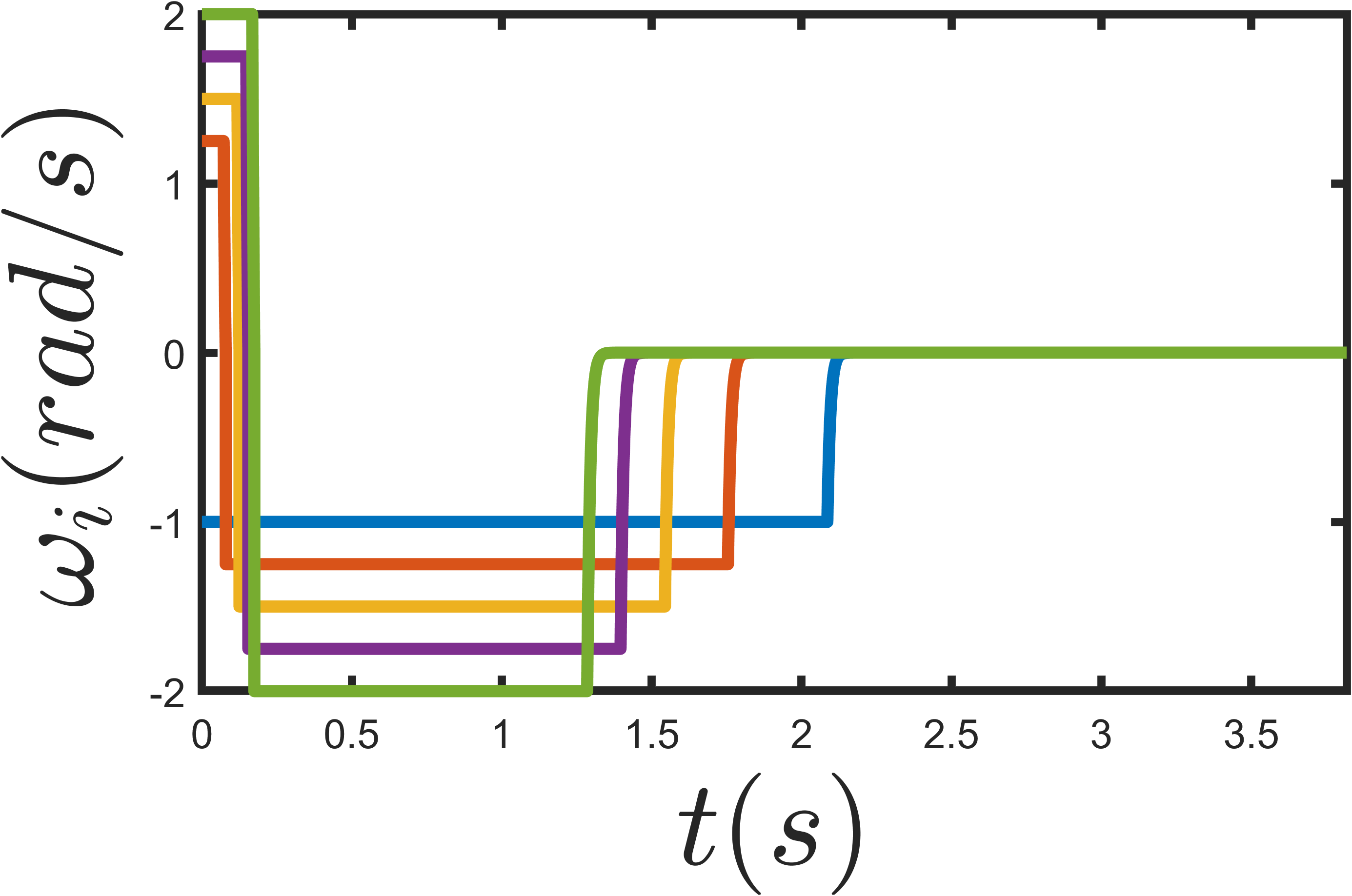}
        \label{fig:s1subfig2}
    \end{subfigure}
    \begin{subfigure}{0.115\textwidth}
        \centering
        \includegraphics[width=\linewidth]{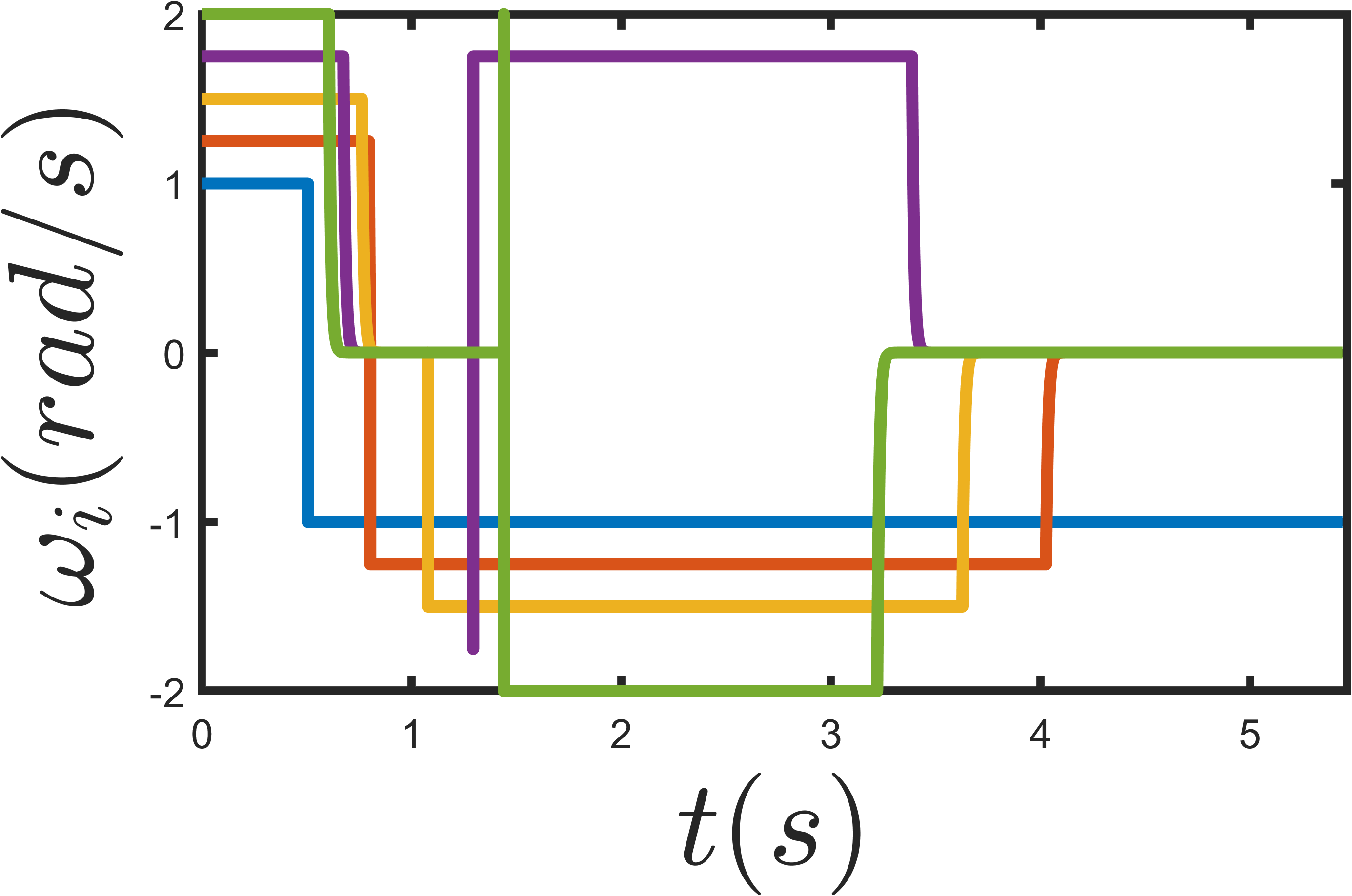}
        \label{fig:s1subfig2}
    \end{subfigure}
\caption{Data from the first simulation. The first, second, and third rows correspond to the time histories of the virtual time variable $T_i$, the distance to the target $r_i$, and the control input $\omega_i$, respectively. From left to right, the columns correspond to the four different initial state settings.
}
\label{fig:0006}
\end{figure}

\begin{figure}[!htbp]
    \centering
    \begin{subfigure}{0.23\textwidth}
        \centering
        \includegraphics[width=\linewidth]{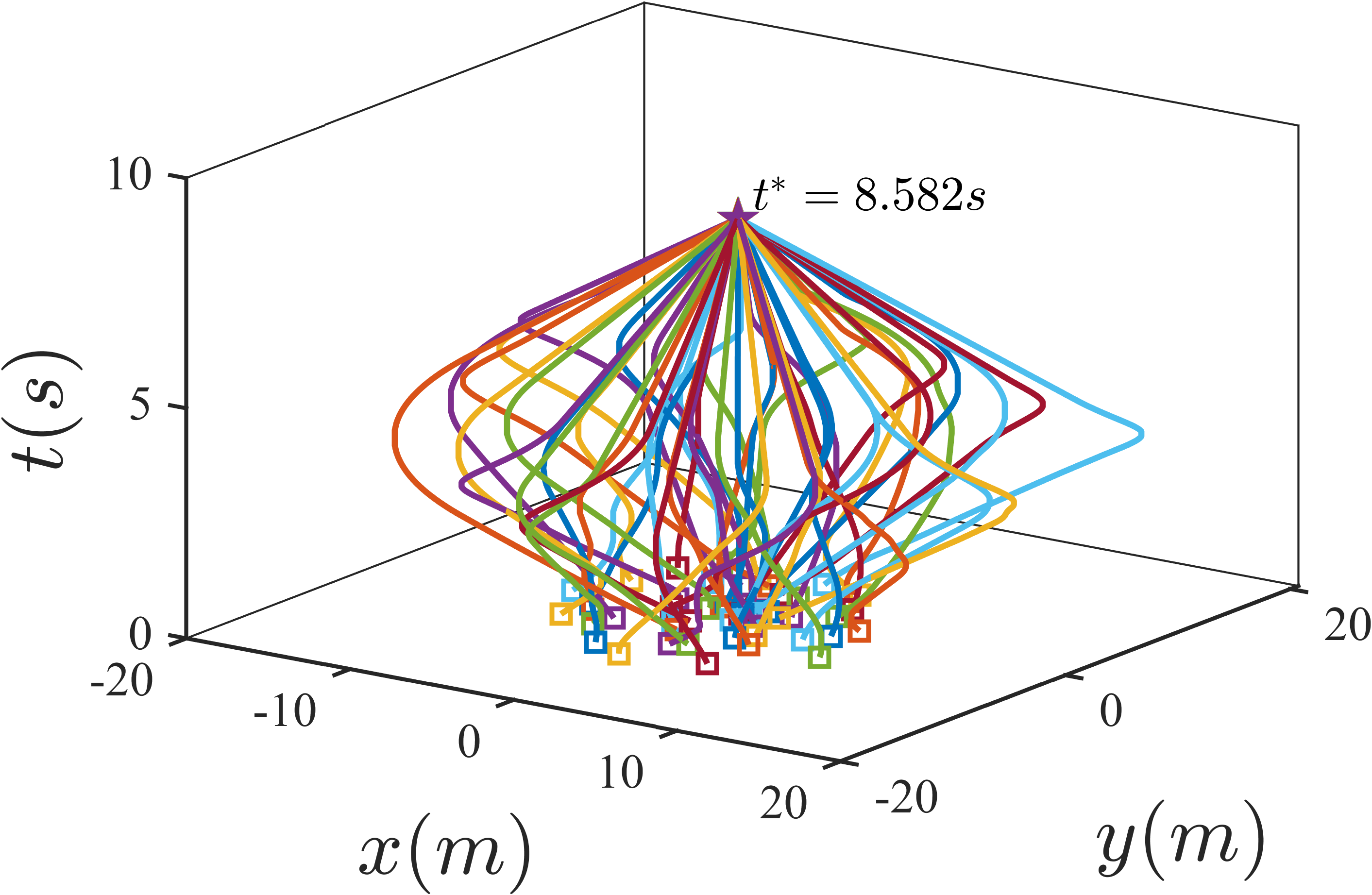}
        \label{fig:s1subfig1}
    \end{subfigure}
    \begin{subfigure}{0.23\textwidth}
        \centering
        \includegraphics[width=\linewidth]{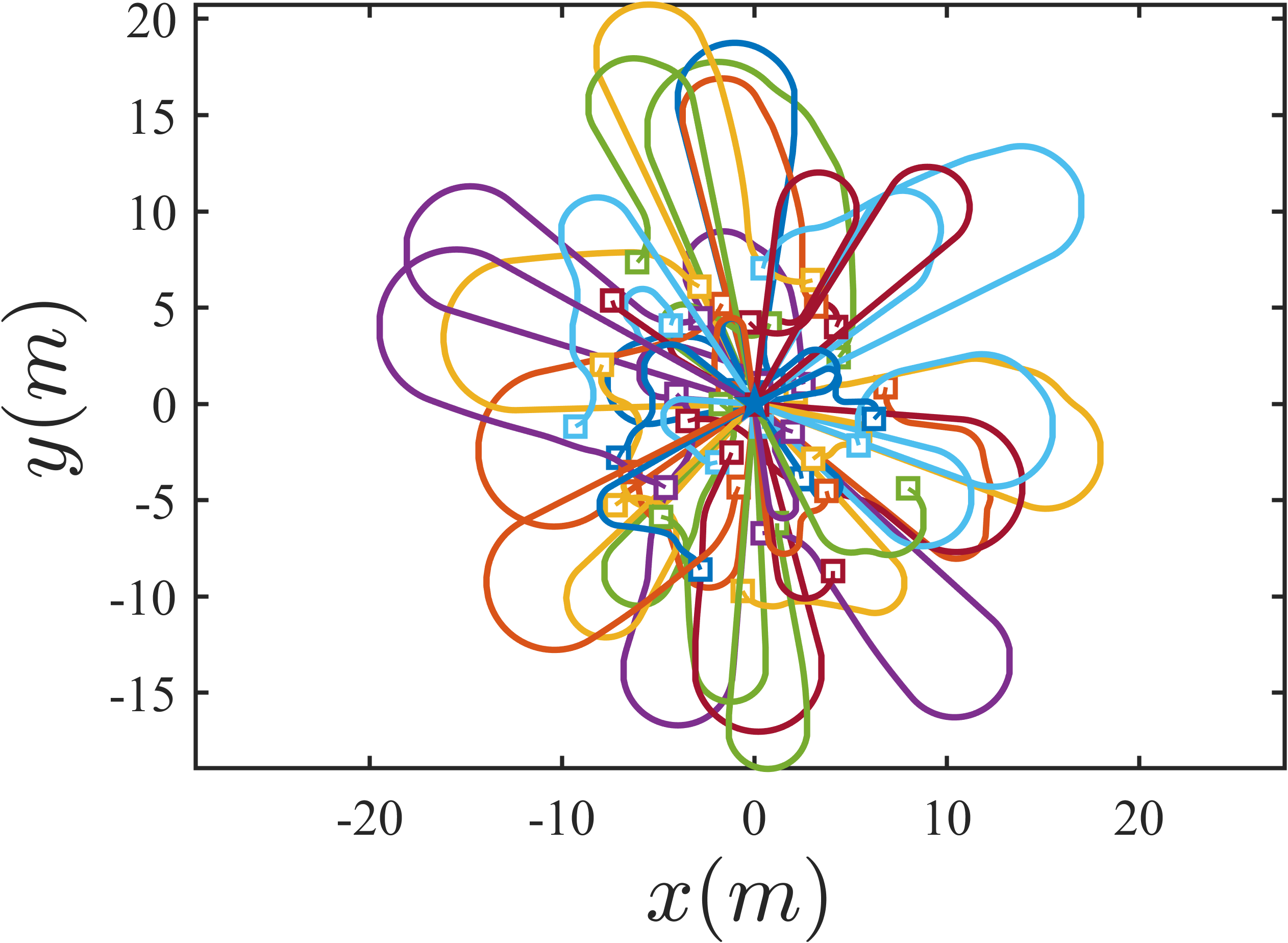}
        \label{fig:s1subfig2}
    \end{subfigure}
    
    \begin{subfigure}{0.15\textwidth}
        \centering
        \includegraphics[width=\linewidth]{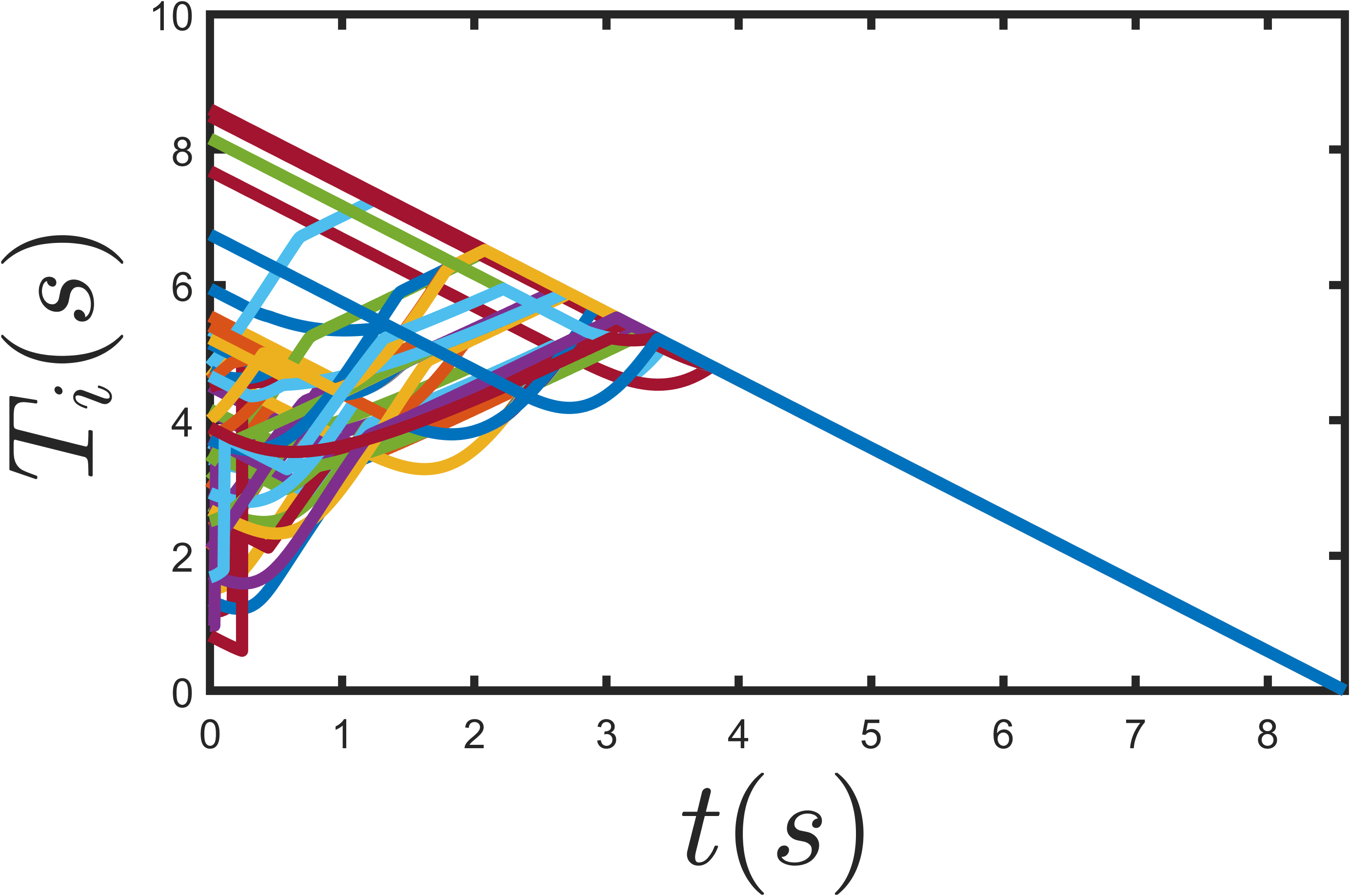}
        \label{fig:s1subfig3}
    \end{subfigure}
    \begin{subfigure}{0.15\textwidth}
        \centering
        \includegraphics[width=\linewidth]{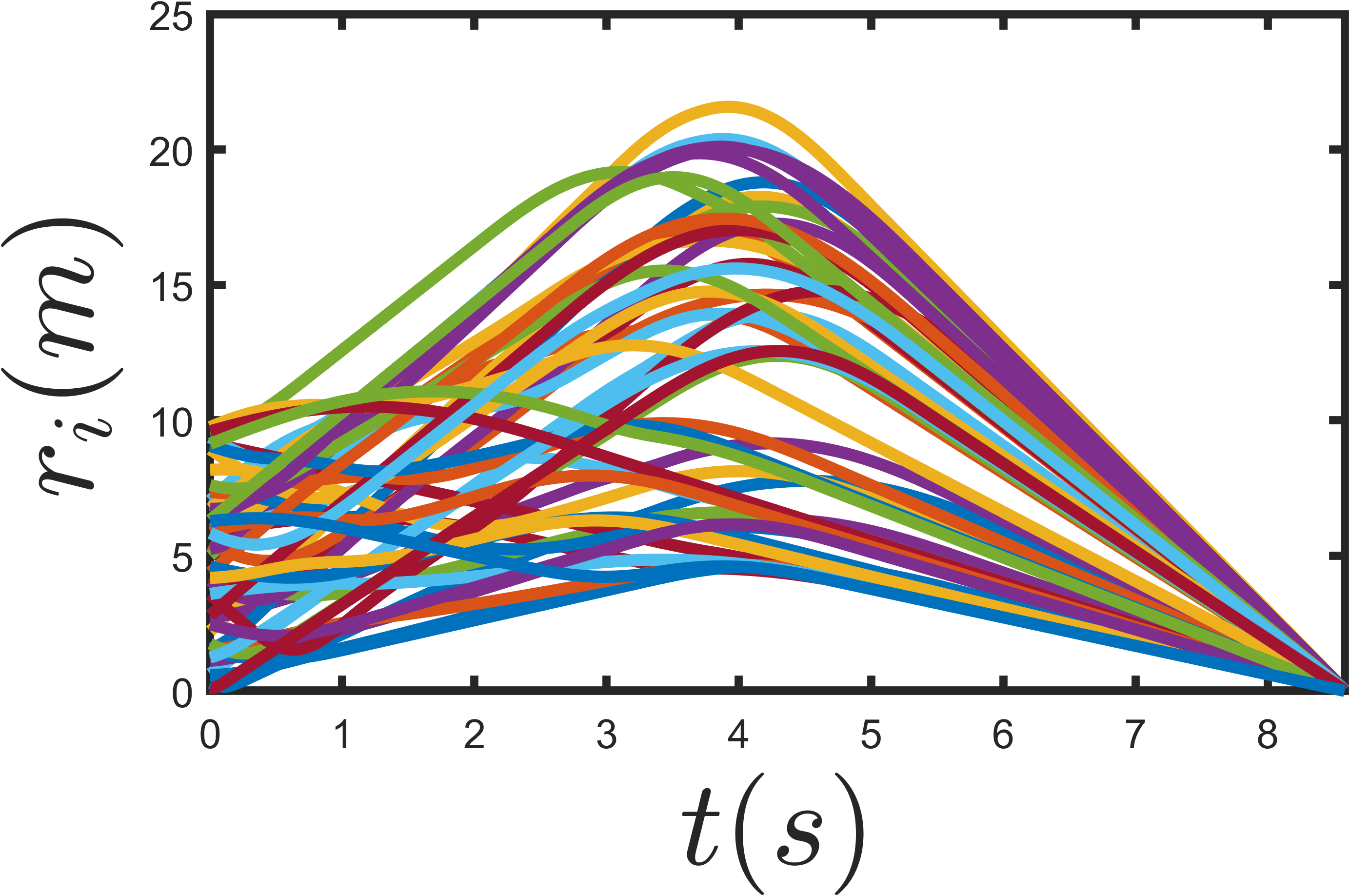}
        \label{fig:s1subfig4}
    \end{subfigure}
    \begin{subfigure}{0.15\textwidth}
        \centering
        \includegraphics[width=\linewidth]{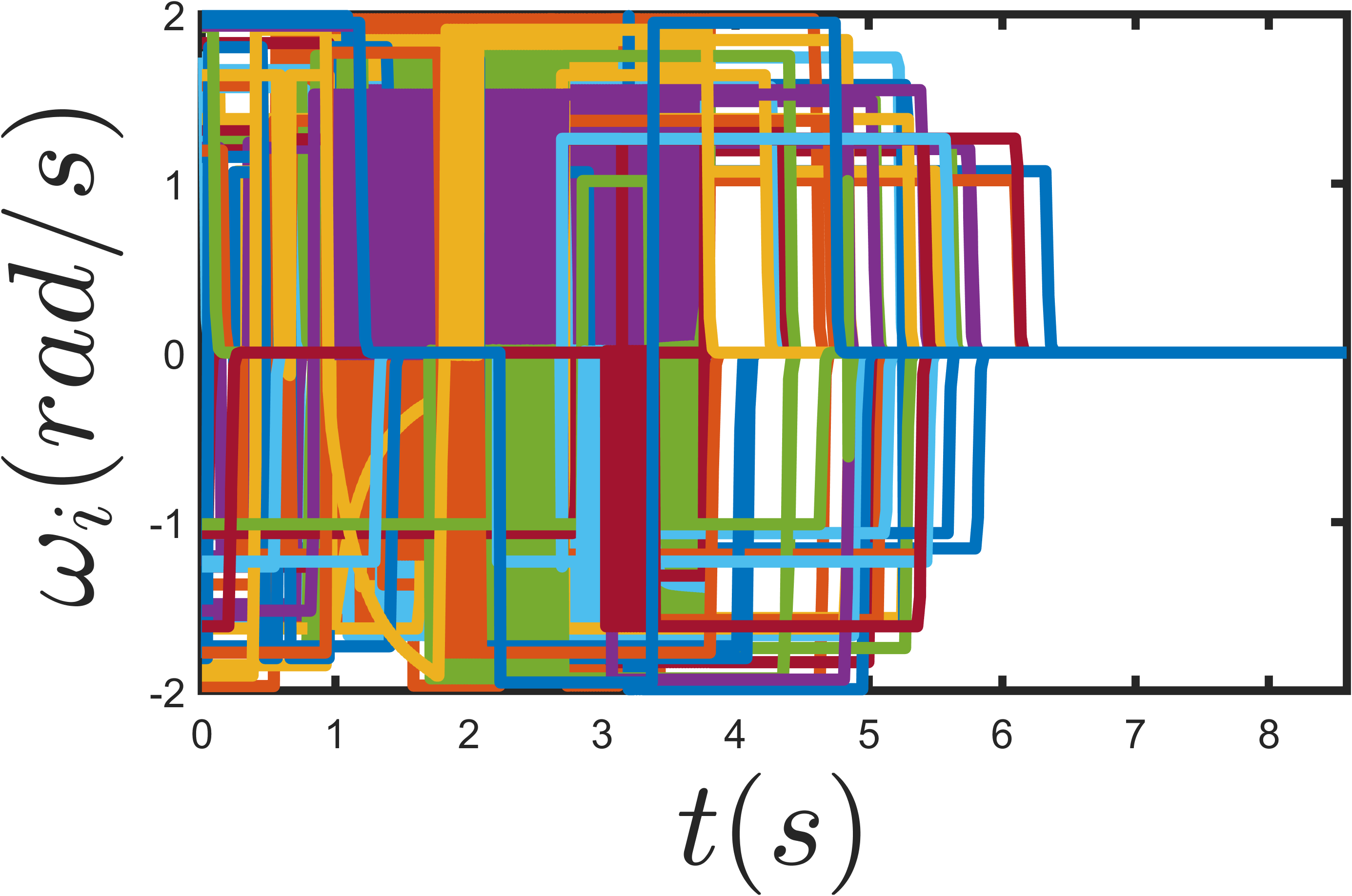}
        \label{fig:0001727}
    \end{subfigure}

    \caption{The second simulation results. In a circular region with the origin as the center and a radius of $10m$, $N=50$ curvature-constrained robots are randomly generated. The initial heading angles, linear velocities, and maximum angular velocities of the robots are randomly assigned, with $\theta_i \in [0,2\pi)$, $v_i \in [1m/s,5m/s]$, and $\bar{\omega}_i \in [1rad/s,2rad/s]$. The control gain is $k_\theta = 100$ and the arrival threshold is $\varepsilon = 0.01m$. The top row of the figure shows the robot position-time trajectories and motion paths, where small squares and stars indicate the initial and final positions, and arrows represent the initial headings. The final time is $t^*=8.583s$. The bottom row shows the time histories of the virtual time variable $T_i$, the distances to the target $r_i$, and the control inputs $\omega_i$ for $i\in \mathbb{Z}_1^{50}$.}
    \label{fig:0007}
\end{figure}

\begin{figure}[!htbp]
    \centering
    \begin{subfigure}{0.23\textwidth}
        \centering
        \includegraphics[width=\linewidth]{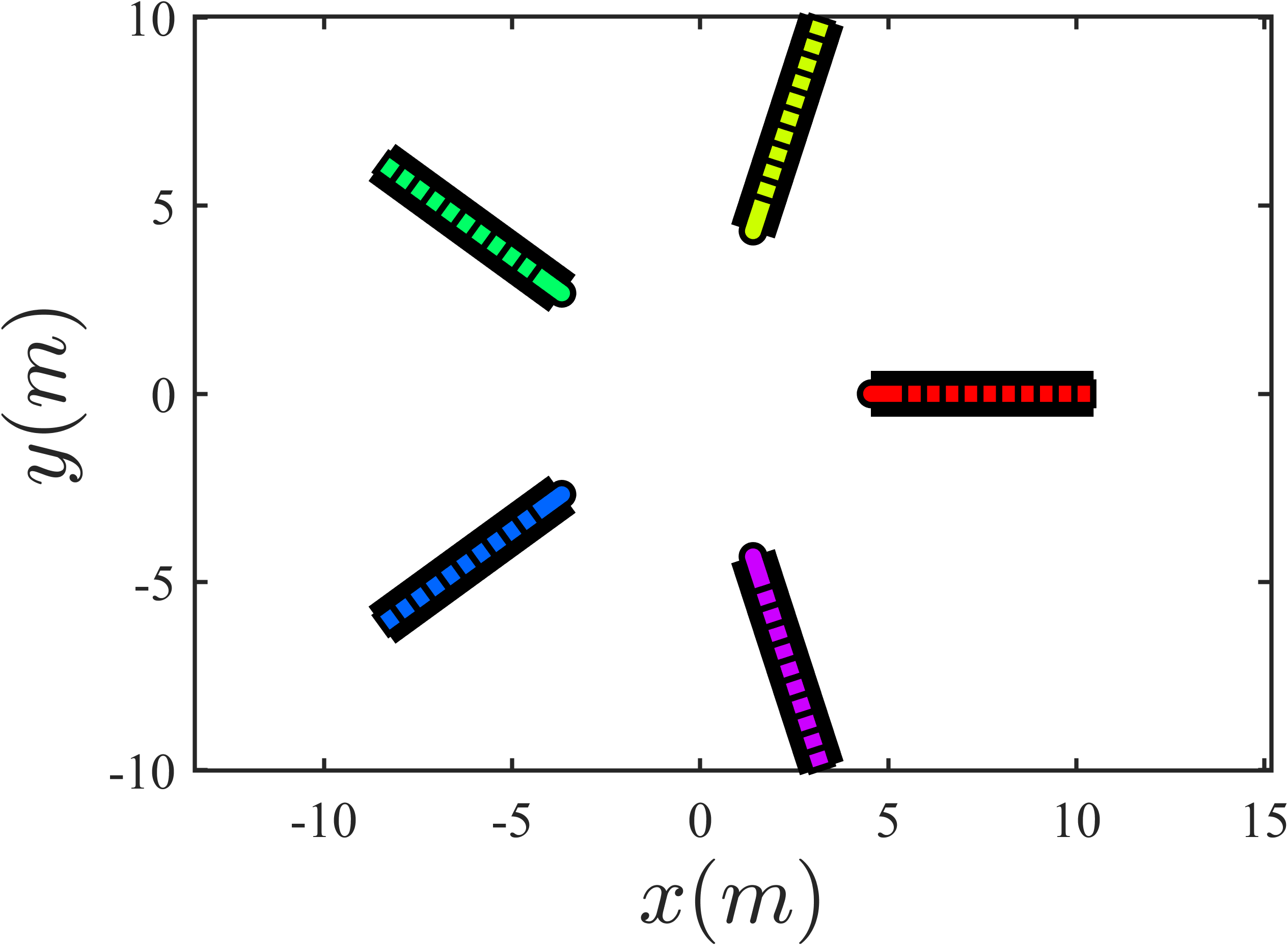}
        \caption{Snapshot at $t=t^*/4$}
        \label{fig:subfig1}
    \end{subfigure}
    \hfill
    \begin{subfigure}{0.23\textwidth}
        \centering
        \includegraphics[width=\linewidth]{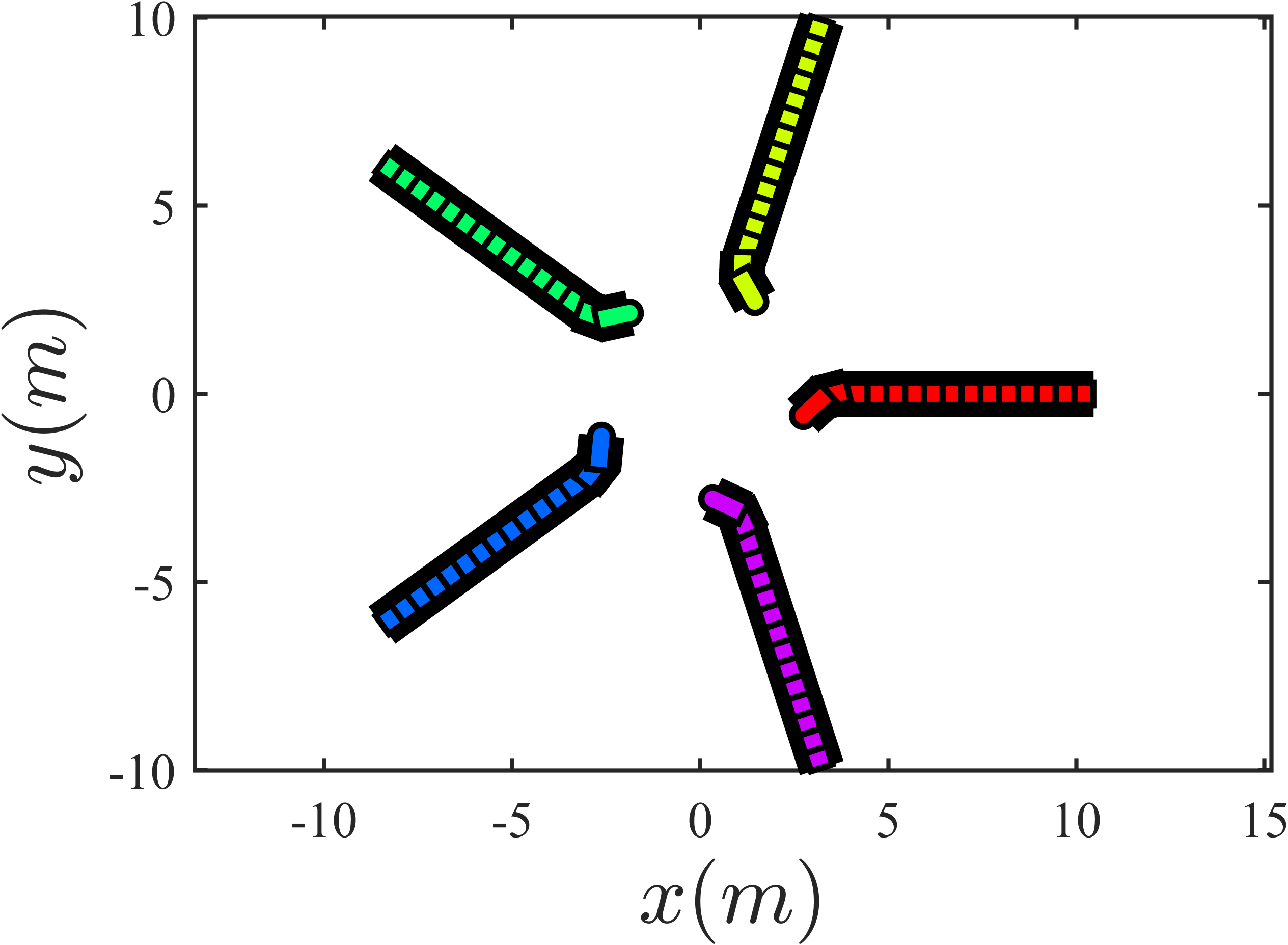}
        \caption{Snapshot at $t=t^*/3$}
        \label{fig:subfig2}
    \end{subfigure}
    
    \begin{subfigure}{0.23\textwidth}
        \centering
        \includegraphics[width=\linewidth]{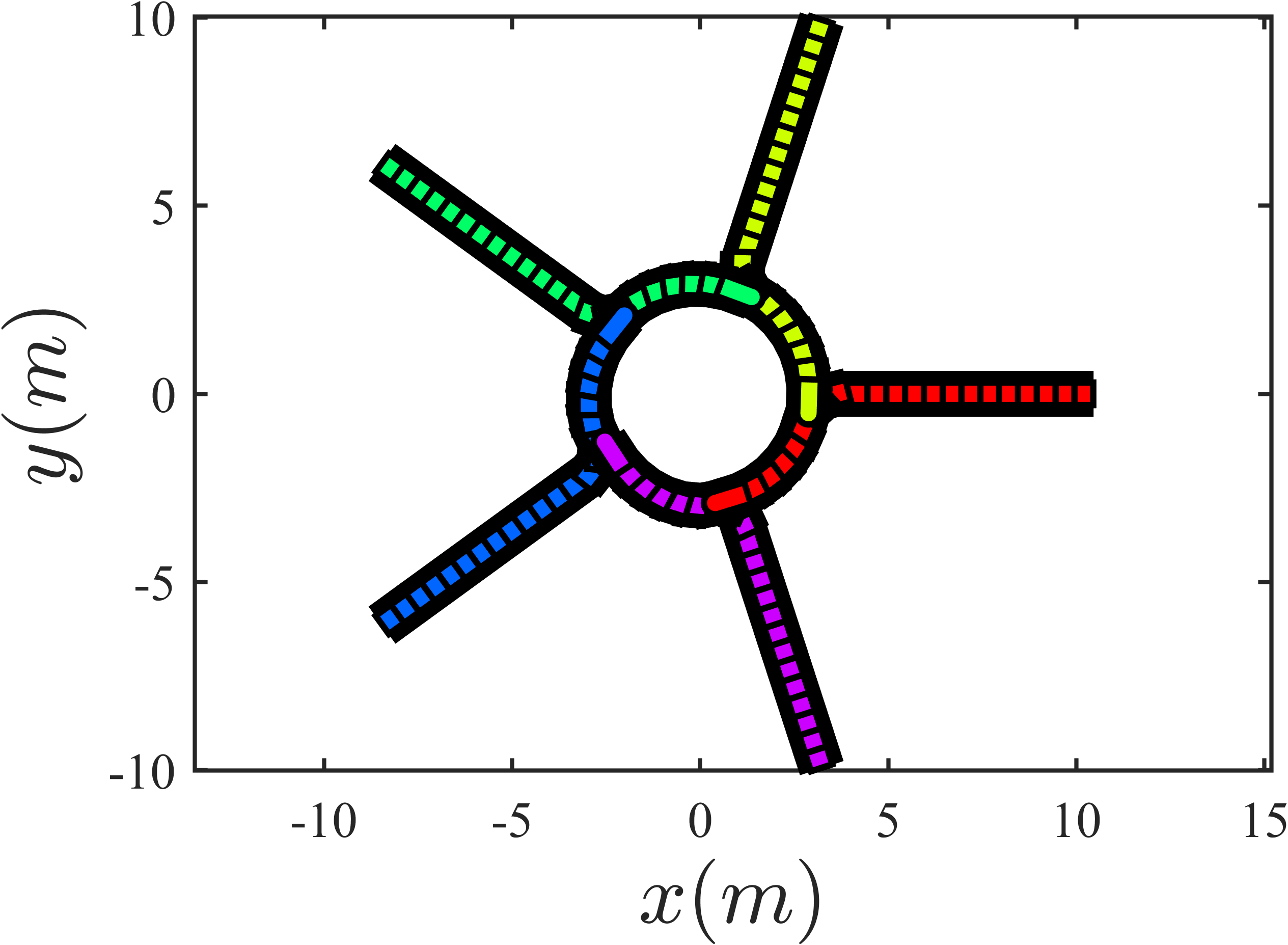}
        \caption{Snapshot at $t=t^*/2$}
        \label{fig:subfig3}
    \end{subfigure}
    \hfill
    \begin{subfigure}{0.23\textwidth}
        \centering
        \includegraphics[width=\linewidth]{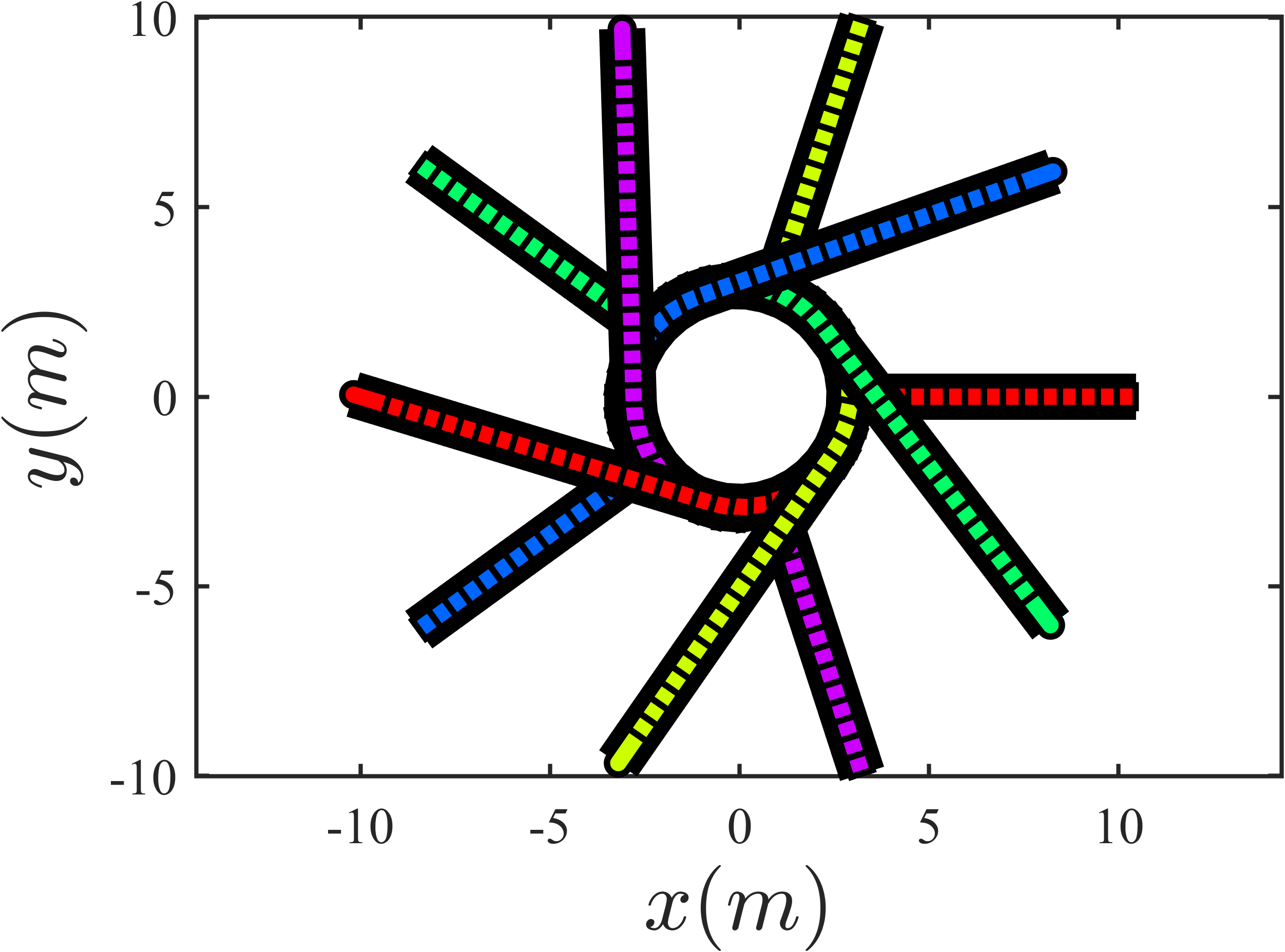}
        \caption{Snapshot at $t=t^*$}
        \label{fig:subfig4}
    \end{subfigure}
    \caption{The third simulation results. The initial states of $N=5$ robots are set as $\boldsymbol{p}_i = [10 \cos(\psi i), 10 \sin(\psi i)]^\top$, and the target positions are $\boldsymbol{p}^d_i = [10 \cos(\psi i + \frac{\pi}{2}), 10 \sin(\psi i + \frac{\pi}{2})]^\top$, where $\psi = \frac{2\pi}{N+1}$ and $i \in \mathbb{Z}_1^{5}$. The initial heading angles are aligned with the direction of the vector from the robot to its target, i.e., $[\cos\theta_i, \sin\theta_i]^\top = \frac{\boldsymbol{p}^d_i - \boldsymbol{p}_i}{\|\boldsymbol{p}^d_i - \boldsymbol{p}_i\|}$ for $i \in \mathbb{Z}_1^{5}$. The robots' linear velocities are $v_i = 1m/s$, and their maximum angular velocities are $\bar{\omega}_i = 1 + 0.1 (i-1)rad/s$. The control gain is $k_\theta = 100$ and the arrival threshold is $\varepsilon = 0.01m$. The simultaneous arrival time is $t^* = 21.754s$.}
    \label{fig:0008}
\end{figure}

\subsection{Real-World Experiments}

\begin{table}[!h]
\centering
\renewcommand{\arraystretch}{1.3} 
\captionsetup{labelformat=empty} 
\caption{\centering \scriptsize TABLE \uppercase\expandafter{\romannumeral1} \\ \centering \scriptsize INITIAL AND TARGET CONFIGURATIONS FOR EXPERIMENTS}
\label{tab:my-table}
\begin{tabular}{|c@{\hskip 4pt}|c|c|c|c|}
\hline
\multicolumn{2}{|c|}{\textbf{Example}} & $\boldsymbol{\xi}_i(0)=(\boldsymbol{p}_i,\theta_i)$ & $(v_i,\bar{\omega}_i)$ & $\boldsymbol{p}^d_i$ \\ \hline
\multirow{2}{*}{\textbf{Set 1}} & \textbf{UAV 1} & $(12.72,-1.59,3.13)$ & $(0.8,2)$ & $(0,-2)$  \\ \cline{2-5} 
                      & \textbf{UAV 2}  & $(9.37,1.95,3.12)$ & $(1,1)$ & $(0,2)$   \\ \hline
\multirow{2}{*}{\textbf{Set 2}} & \textbf{UAV 1}  & $(5.59,-2.15,0.05)$ & $(0.5,0.4)$ & $(0,-2)$  \\ \cline{2-5} 
                      & \textbf{UAV 2} & $(10.68,0.77,3.26)$ & $(1,1)$ & $(0,2)$  \\ \hline
\end{tabular}
\end{table}

\begin{figure}[!t]
    \centering
    \begin{subfigure}{0.23\textwidth}
        \centering
        \includegraphics[width=\linewidth]{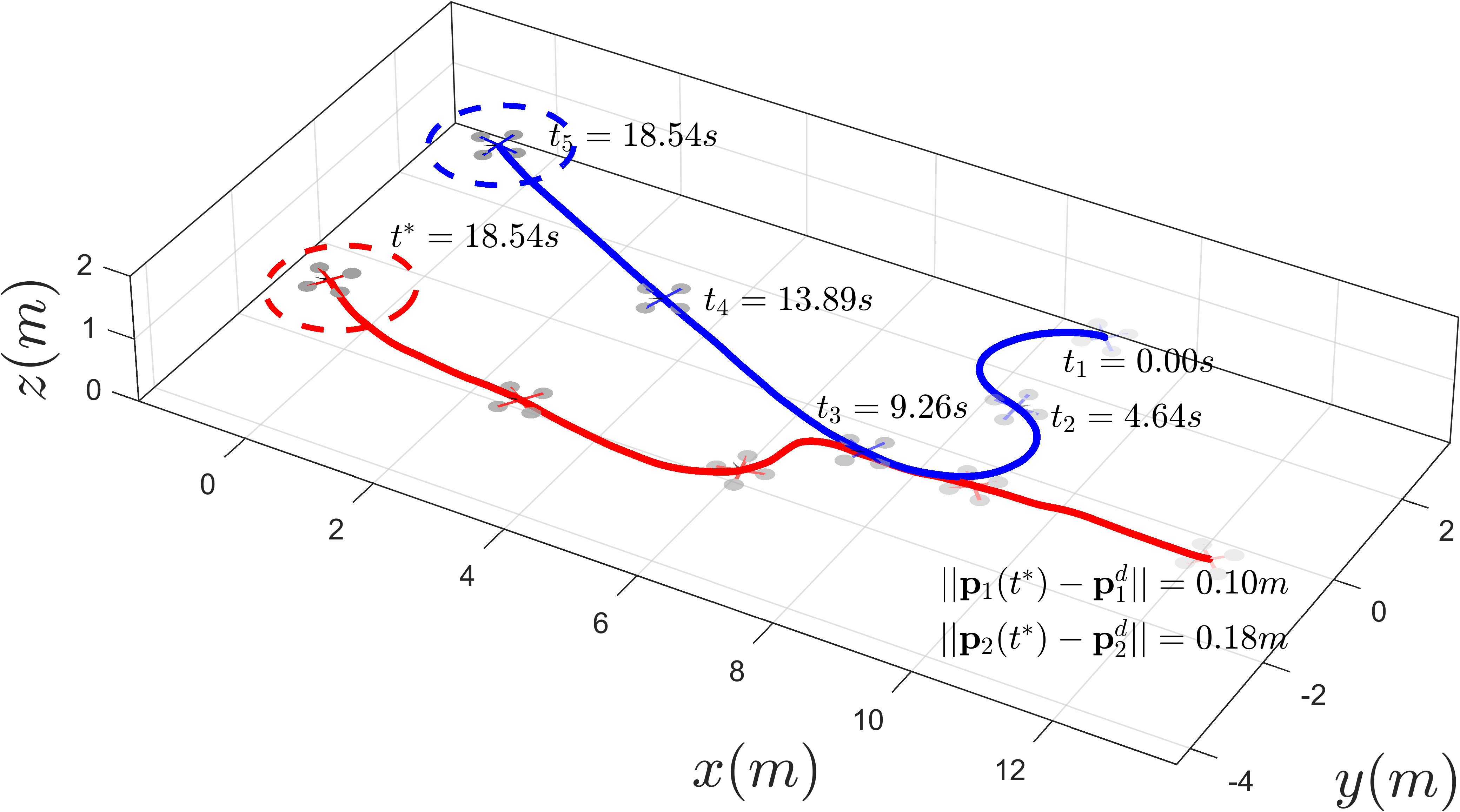}
        \label{fig:subfig1}
    \end{subfigure}
    \begin{subfigure}{0.115\textwidth}
        \centering
        \includegraphics[width=\linewidth]{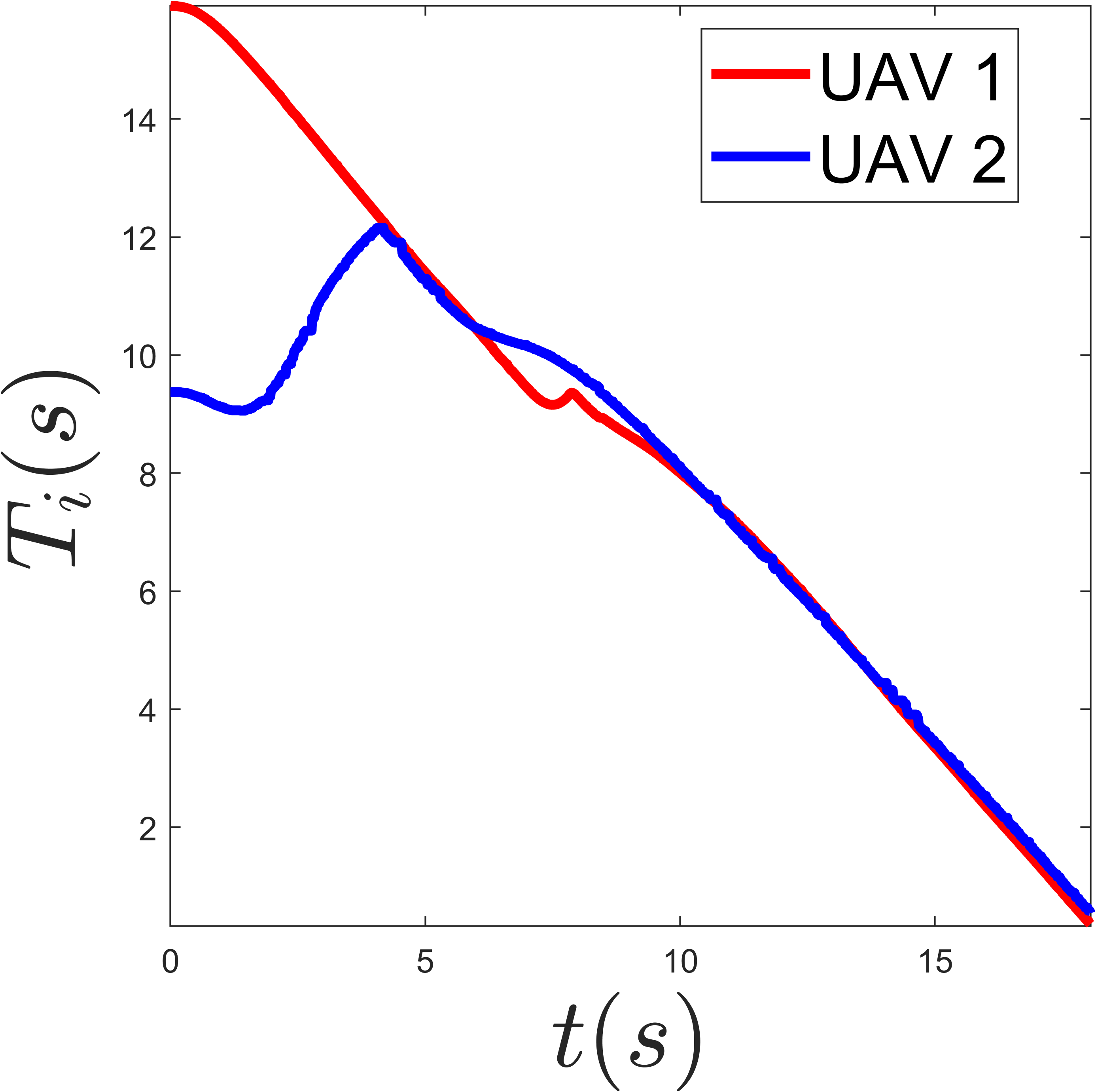}
        \label{fig:subfig2}
    \end{subfigure}
    \begin{subfigure}{0.115\textwidth}
        \centering
        \includegraphics[width=\linewidth]{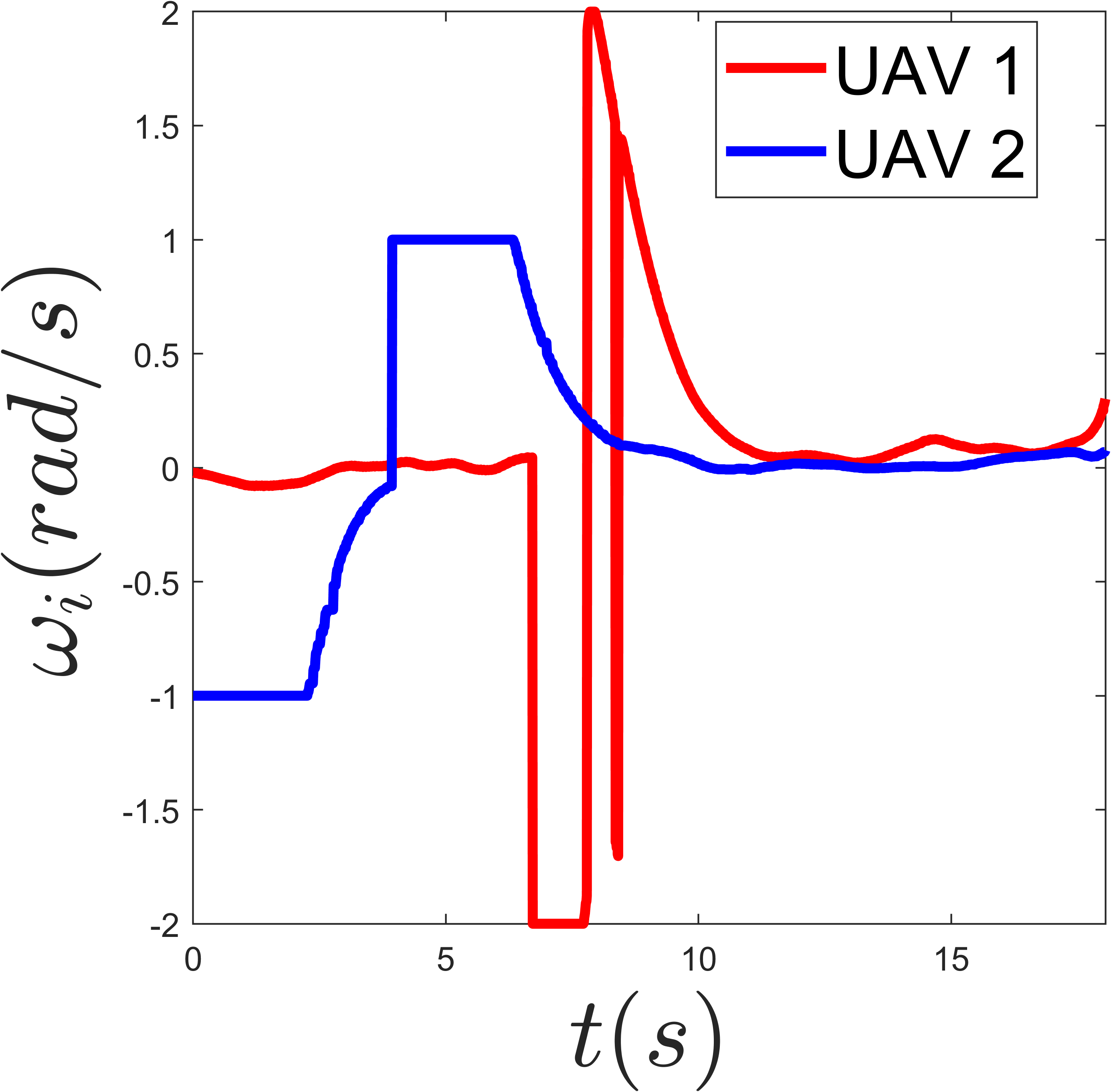}
        \label{fig:subfig2}
    \end{subfigure}
    \begin{subfigure}{0.23\textwidth}
        \centering
        \includegraphics[width=\linewidth]{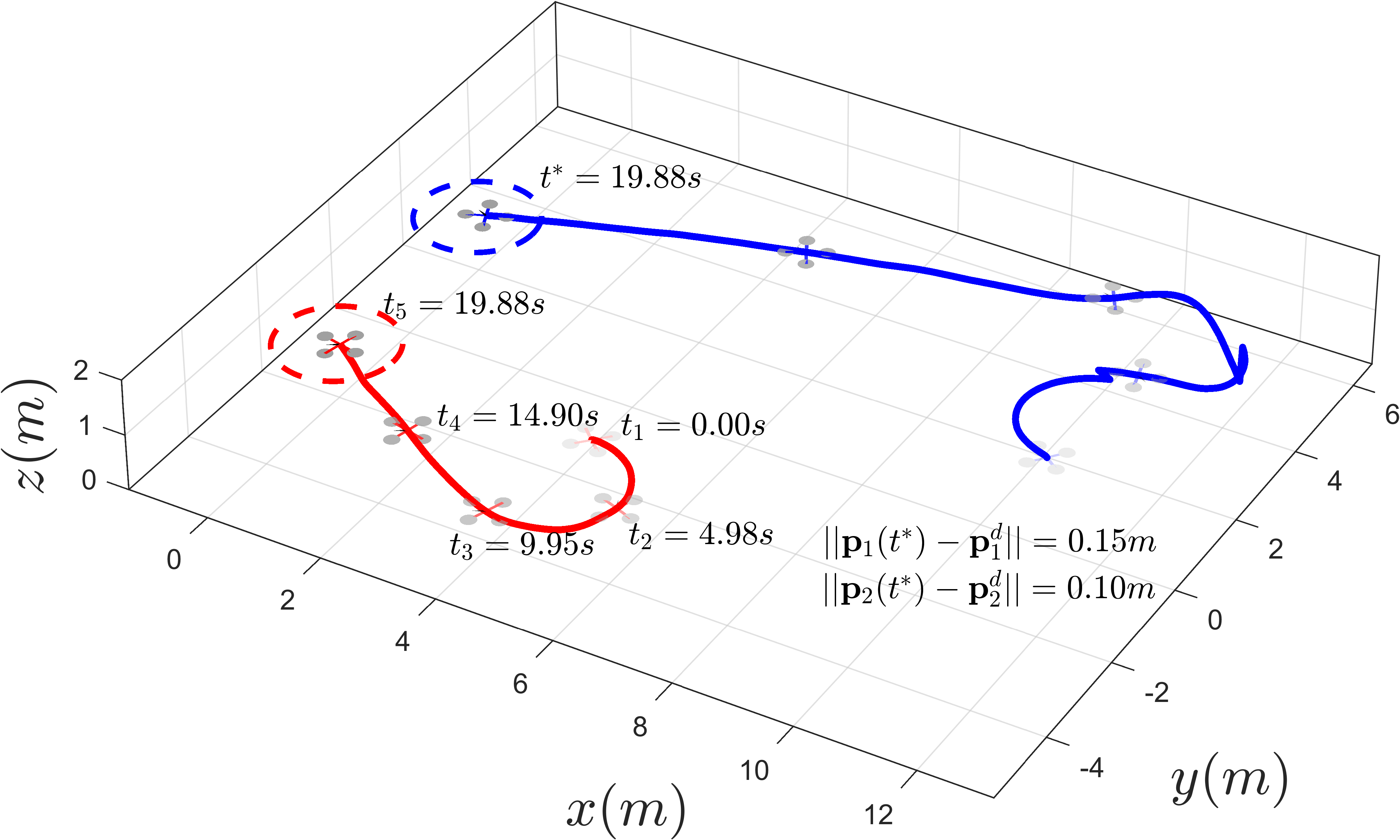}
        \label{fig:subfig1}
    \end{subfigure}
    \begin{subfigure}{0.115\textwidth}
        \centering
        \includegraphics[width=\linewidth]{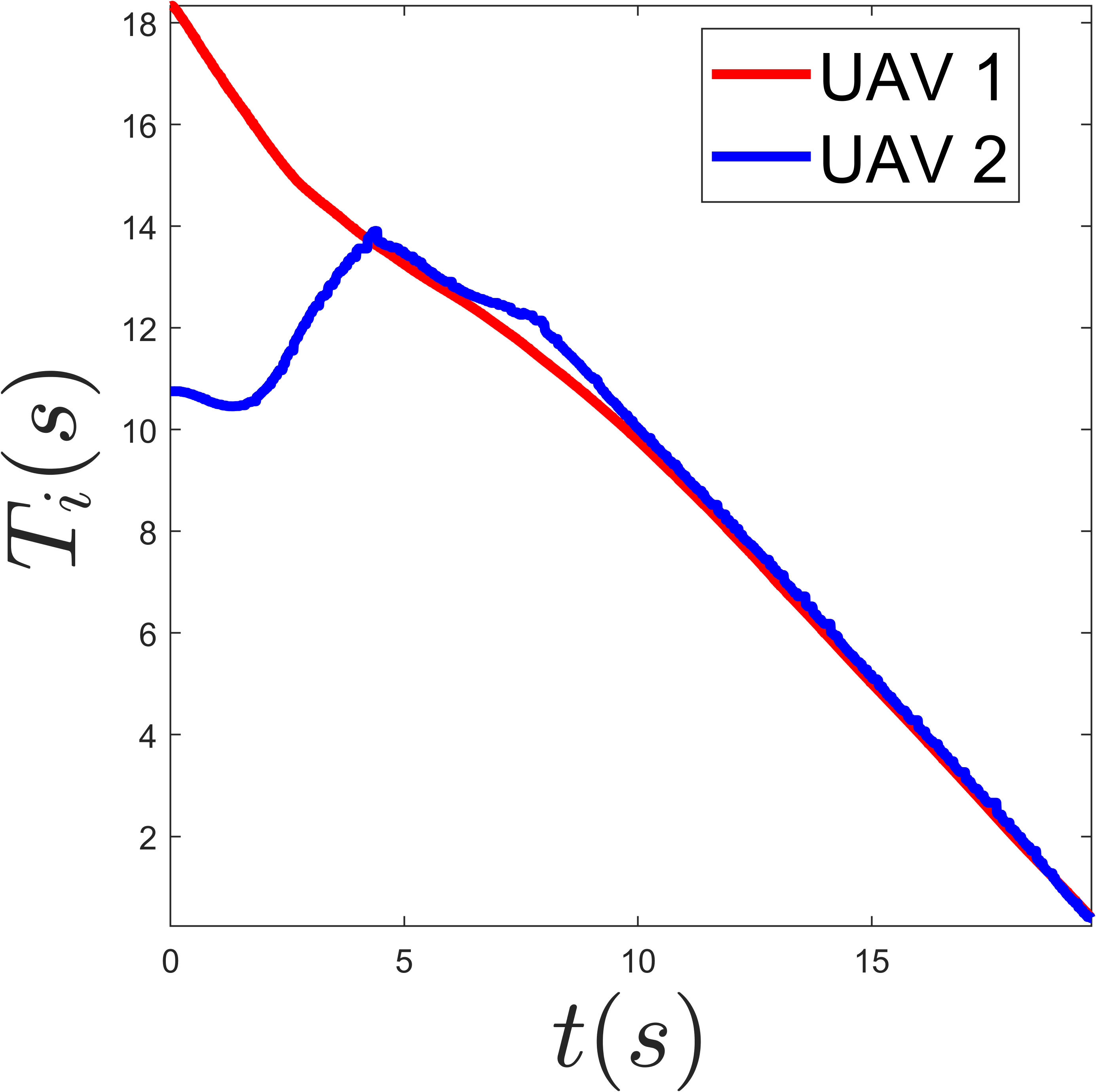}
        \label{fig:subfig2}
    \end{subfigure}
    \begin{subfigure}{0.115\textwidth}
        \centering
        \includegraphics[width=\linewidth]{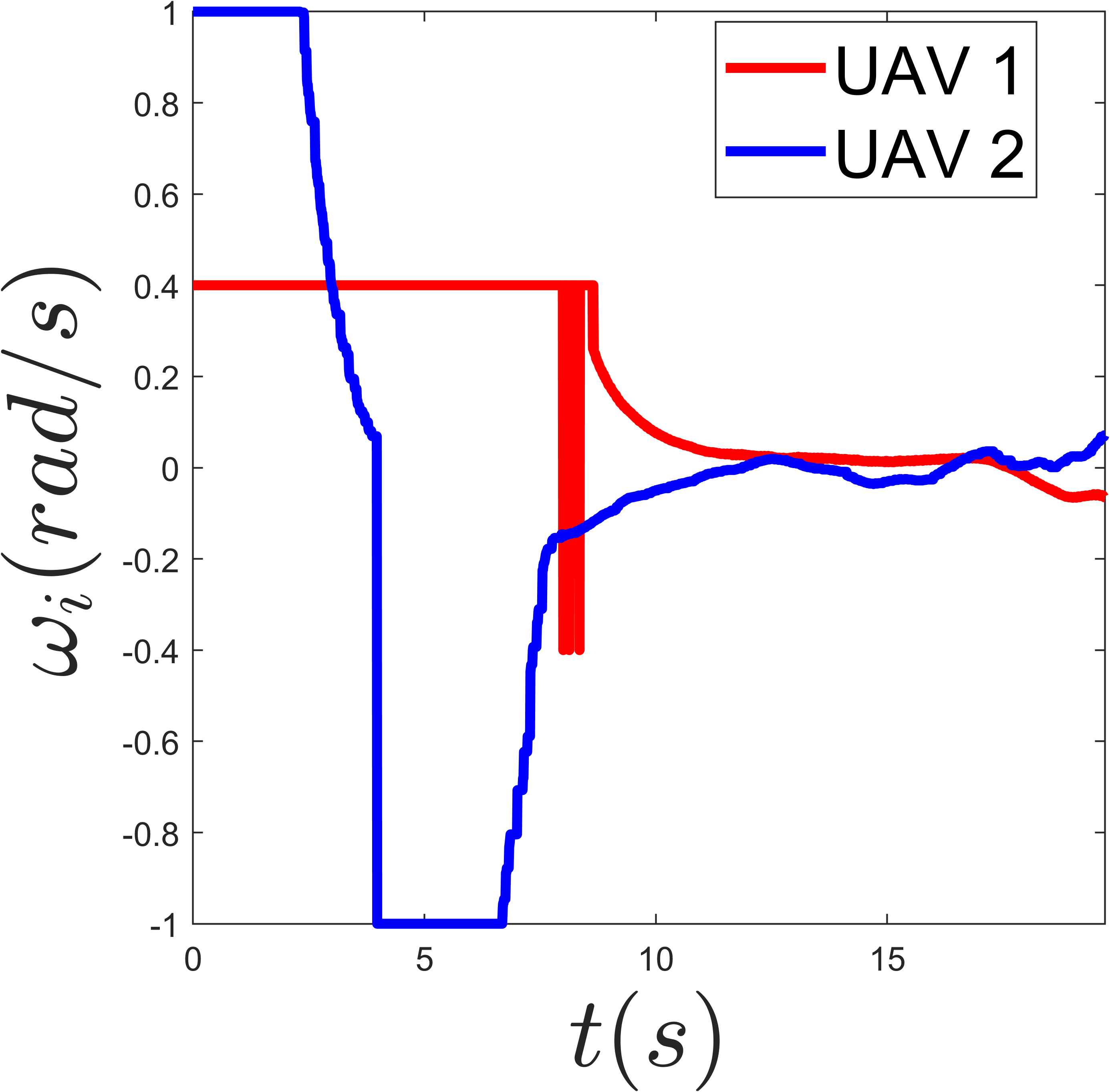}
        \label{fig:subfig2}
    \end{subfigure}
    \caption{Simultaneous arrival of two UAVs. Snapshots are shown at times $t=it^*/4,i\in\mathbb{Z}_0^{4}$. The upper figure corresponds to the initial condition set 1, while the lower figure corresponds to the initial condition set 2. From left to right, the figures represent the trajectory plot, virtual time evolution, and control input variation, respectively.}
    \label{fig:0010}
\end{figure}

In this study, we utilized two quadrotor UAVs, each equipped with PX4 flight controllers. These UAVs are capable of manually adjusting their curvature while maintaining a constant speed, a feature that can be strictly enforced to assess performance under various initial conditions. The UAVs' positioning is facilitated by an indoor motion capture system, with inter-UAV communication enabled via WiFi. Two distinct sets of initial conditions were configured for the experiment (see TABLE~\ref{tab:my-table}, with base units of meters (m), seconds (s), and radians (rad)), and a simple proportional control method was employed to maintain the flight altitude at $z=1m$. Simultaneous arrival is defined as the condition when the virtual time difference $|T_1 - T_2| \leq \delta_t = 0.5s$. The experiment concludes once any UAV reaches its target.

The results are presented in Fig.~\ref{fig:0010}. Notably, even in the presence of disturbances such as oscillations in the $z$-axis and communication delays, the error in simultaneous arrival can be kept within a small range (specifically, at time $t^*$, $\left|\|\boldsymbol{p}_1(t^*)-\boldsymbol{p}^d_1\|-\|\boldsymbol{p}_2(t^*)-\boldsymbol{p}^d_2\|\right| < 0.1m$). Additionally, the evolution of virtual time, variations in bearing angle, and UAV trajectories collectively provide indirect validation of Theorem~\ref{thm:002} and Theorem~\ref{thm:003}.

\section{Conclusion and Future Work}\label{Conclusion}
	We present a distributed control method based on a max-consensus protocol for multi-robot systems with curvature constraints and constant speeds, achieving simultaneous arrival at target points. By introducing a virtual time variable and leveraging the geometric optimization properties of Dubins paths, the proposed method ensures global time synchronization while maintaining optimality, scalability, and low communication overhead. In the control design, a hybrid strategy combining saturated proportional control and optimal control effectively addresses the local non-monotonicity in the curvature-constrained region, thereby guaranteeing system stability. Extensive simulations and experiments with multiple quadrotors demonstrate the effectiveness, robustness, and scalability of the algorithm in large-scale robot systems, and verify the safety of the approach when obstacle avoidance strategies are incorporated. 
 
    Future work includes a rigorous theoretical analysis of the zero-dynamics behavior and global convergence in curvature-constrained regions, as well as the design of more refined obstacle avoidance strategies, aiming to extend the algorithm to complex environments. Furthermore, we will focus on extending the simultaneous arrival problem to three-dimensional environments and validating it through experiments with multiple fixed-wing UAVs.

\begin{appendices}

	\end{appendices}

	\bibliographystyle{IEEEtran}
	\bibliography{IEEEabrv.bib}

@book{MesbahiEgerstedt+2010,

title = {Graph Theoretic Methods in Multiagent Networks},
author = {Mehran Mesbahi and Magnus Egerstedt},
publisher = {Princeton University Press},
address = {Princeton},
doi = {doi:10.1515/9781400835355},
isbn = {9781400835355},
year = {2010},
lastchecked = {2025-08-03}
}

@ARTICLE{10938343,
  author={Wilson, James P. and Gupta, Shalabh and Wettergren, Thomas A.},
  journal={IEEE Transactions on Robotics}, 
  title={Generalized Multispeed Dubins Motion Model}, 
  year={2025},
  volume={41},
  number={},
  pages={2861-2878},
  doi={10.1109/TRO.2025.3554436}}

@article{SHKEL2001179,
  title={Classification of the Dubins set},
  author={Shkel, Andrei M and Lumelsky, Vladimir},
  journal={Robotics and Autonomous Systems},
  volume={34},
  number={4},
  pages={179--202},
  year={2001},
  publisher={Elsevier}
}

@INPROCEEDINGS{9482855,
  author={Tran, Dzung and Casbeer, David and Milutinović, Dejan},
  booktitle={2021 American Control Conference (ACC)}, 
  title={Synthesizing Simultaneous Arrival from Single Agent Time Optimal Controllers}, 
  year={2021},
  volume={},
  number={},
  pages={3902-3907},
  doi={10.23919/ACC50511.2021.9482855}}

@ARTICLE{5746538,
  author={Snape, Jamie and Berg, Jur van den and Guy, Stephen J. and Manocha, Dinesh},
  journal={IEEE Transactions on Robotics}, 
  title={The Hybrid Reciprocal Velocity Obstacle}, 
  year={2011},
  volume={27},
  number={4},
  pages={696-706},
  doi={10.1109/TRO.2011.2120810}}

@article{Li2020,
author = {Li, Kang and Wang, Jianan and Lee, Chang-Hun and Zhou, Rui and Zhao, Shiyu},
title = {Distributed Cooperative Guidance for Multivehicle Simultaneous Arrival Without Numerical Singularities},
journal = {Journal of Guidance, Control, and Dynamics},
volume = {43},
number = {7},
pages = {1365-1373},
year = {2020},
doi = {10.2514/1.G005010},

}

@ARTICLE{8303222,
  author={Li, Zhenhong and Ding, Zhengtao},
  journal={IEEE Transactions on Control Systems Technology}, 
  title={Robust Cooperative Guidance Law for Simultaneous Arrival}, 
  year={2019},
  volume={27},
  number={3},
  pages={1360-1367},
  doi={10.1109/TCST.2018.2804348}}

@ARTICLE{11004631,
  author={Singh, Kushal P. and Rao, Aditya K. and Tripathy, Twinkle},
  journal={IEEE Transactions on Control of Network Systems}, 
  title={Finite Time Max-consensus for Simultaneous Target Interception in Switching Graph Topologies}, 
  year={2025},
  volume={},
  number={},
  pages={1-11},
  doi={10.1109/TCNS.2025.3570423}}

@article{doi:10.2514/1.59512,
author = {Anderson, Ross P. and Bakolas, Efstathios and Milutinovi\'{c}, Dejan and Tsiotras, Panagiotis},
title = {Optimal Feedback Guidance of a Small Aerial Vehicle in a Stochastic Wind},
journal = {Journal of Guidance, Control, and Dynamics},
volume = {36},
number = {4},
pages = {975-985},
year = {2013},
doi = {10.2514/1.59512},
}

@INPROCEEDINGS{351019,
  author={Xuan-Nam Bui and Boissonnat, J.-D. and Soueres, P. and Laumond, J.-P.},
  booktitle={Proceedings of the 1994 IEEE International Conference on Robotics and Automation}, 
  title={Shortest path synthesis for Dubins non-holonomic robot}, 
  year={1994},
  volume={},
  number={},
  pages={2-7 vol.1},
  doi={10.1109/ROBOT.1994.351019}}

@INPROCEEDINGS{9143934,
  author={Zadka, Benjamin and Tripathy, Twinkle and Tsalik, Ronny and Shima, Tal},
  booktitle={2020 European Control Conference (ECC)}, 
  title={A Max-Consensus Cyclic Pursuit Based Guidance Law for Simultaneous Target Interception}, 
  year={2020},
  volume={},
  number={},
  pages={662-667},
  doi={10.23919/ECC51009.2020.9143934}}

@ARTICLE{11018241,
  author={Sinha, Abhinav and Mukherjee, Dwaipayan and Kumar, Shashi Ranjan},
  journal={IEEE Transactions on Aerospace and Electronic Systems}, 
  title={Consensus-driven Deviated Pursuit for Guaranteed Simultaneous Interception of Moving Targets}, 
  year={2025},
  volume={},
  number={},
  pages={1-12},
  doi={10.1109/TAES.2025.3575049}}

@ARTICLE{1333204,
  author={Olfati-Saber, R. and Murray, R.M.},
  journal={IEEE Transactions on Automatic Control}, 
  title={Consensus problems in networks of agents with switching topology and time-delays}, 
  year={2004},
  volume={49},
  number={9},
  pages={1520-1533},
  doi={10.1109/TAC.2004.834113}}

@ARTICLE{8106791,
  author={He, Shaoming and Wang, Wei and Lin, Defu and Lei, Hongbo},
  journal={IEEE Transactions on Aerospace and Electronic Systems}, 
  title={Consensus-Based Two-Stage Salvo Attack Guidance}, 
  year={2018},
  volume={54},
  number={3},
  pages={1555-1566},
  doi={10.1109/TAES.2017.2773272}}

@ARTICLE{8318665,
  author={Kang, Shen and Wang, Jianan and Li, Guang and Shan, Jiayuan and Petersen, Ian R.},
  journal={IEEE Transactions on Aerospace and Electronic Systems}, 
  title={Optimal Cooperative Guidance Law for Salvo Attack: An MPC-Based Consensus Perspective}, 
  year={2018},
  volume={54},
  number={5},
  pages={2397-2410},
  doi={10.1109/TAES.2018.2816880}}

@ARTICLE{795787,
  author={Ando, H. and Oasa, Y. and Suzuki, I. and Yamashita, M.},
  journal={IEEE Transactions on Robotics and Automation}, 
  title={Distributed memoryless point convergence algorithm for mobile robots with limited visibility}, 
  year={1999},
  volume={15},
  number={5},
  pages={818-828},
  doi={10.1109/70.795787}}

@article{doi:10.2514/1.G004074,
author = {Tsalik, Ronny and Shima, Tal},
title = {Circular Impact-Time Guidance},
journal = {Journal of Guidance, Control, and Dynamics},
volume = {42},
number = {8},
pages = {1836-1847},
year = {2019},
doi = {10.2514/1.G004074},


}

@INPROCEEDINGS{4543489,
  author={van den Berg, Jur and Ming Lin and Manocha, Dinesh},
  booktitle={2008 IEEE International Conference on Robotics and Automation}, 
  title={Reciprocal Velocity Obstacles for real-time multi-agent navigation}, 
  year={2008},
  volume={},
  number={},
  pages={1928-1935},
  doi={10.1109/ROBOT.2008.4543489}}

@ARTICLE{7484276,
  author={Panagou, Dimitra},
  journal={IEEE Transactions on Automatic Control}, 
  title={A Distributed Feedback Motion Planning Protocol for Multiple Unicycle Agents of Different Classes}, 
  year={2017},
  volume={62},
  number={3},
  pages={1178-1193},
  doi={10.1109/TAC.2016.2576020}}

@ARTICLE{1381658,
  author={Zhiyun Lin and Francis, B. and Maggiore, M.},
  journal={IEEE Transactions on Automatic Control}, 
  title={Necessary and sufficient graphical conditions for formation control of unicycles}, 
  year={2005},
  volume={50},
  number={1},
  pages={121-127},
  doi={10.1109/TAC.2004.841121}}

@ARTICLE{4200856,
  author={Dimarogonas, Dimos V. and Kyriakopoulos, Kostas J.},
  journal={IEEE Transactions on Automatic Control}, 
  title={On the Rendezvous Problem for Multiple Nonholonomic Agents}, 
  year={2007},
  volume={52},
  number={5},
  pages={916-922},
  doi={10.1109/TAC.2007.895897}}

@ARTICLE{5783895,
  author={Yu, Jingjin and LaValle, Steven M. and Liberzon, Daniel},
  journal={IEEE Transactions on Automatic Control}, 
  title={Rendezvous Without Coordinates}, 
  year={2012},
  volume={57},
  number={2},
  pages={421-434},
  doi={10.1109/TAC.2011.2158172}}

@article{doi:10.2514/1.G001719,
author = {Tekin, Raziye and Erer, Koray S. and Holzapfel, Florian},
title = {Control of Impact Time with Increased Robustness via Feedback Linearization},
journal = {Journal of Guidance, Control, and Dynamics},
volume = {39},
number = {7},
pages = {1682-1689},
year = {2016},
doi = {10.2514/1.G001719},
}

@article{MLYNCH2003173,
title = {Optimal control of the thrusted skate},
journal = {Automatica},
volume = {39},
number = {1},
pages = {173-176},
year = {2003},
issn = {0005-1098},
doi = {https://doi.org/10.1016/S0005-1098(02)00165-6},
author = {Kevin {M. Lynch}}
}

@article{MATVEEV2011515,
title = {A method for guidance and control of an autonomous vehicle in problems of border patrolling and obstacle avoidance},
journal = {Automatica},
volume = {47},
number = {3},
pages = {515-524},
year = {2011},
issn = {0005-1098},
doi = {https://doi.org/10.1016/j.automatica.2011.01.024},
author = {Alexey S. Matveev and Hamid Teimoori and Andrey V. Savkin}
}

@ARTICLE{1597196,
  author={In-Soo Jeon and Jin-Ik Lee and Min-Jea Tahk},
  journal={IEEE Transactions on Control Systems Technology}, 
  title={Impact-time-control guidance law for anti-ship missiles}, 
  year={2006},
  volume={14},
  number={2},
  pages={260-266},
  doi={10.1109/TCST.2005.863655}}

@ARTICLE{7300456,
  author={Lau, Darwin and Eden, Jonathan and Oetomo, Denny},
  journal={IEEE Transactions on Robotics}, 
  title={Fluid Motion Planner for Nonholonomic 3-D Mobile Robots With Kinematic Constraints}, 
  year={2015},
  volume={31},
  number={6},
  pages={1537-1547},
  doi={10.1109/TRO.2015.2482078}}

@ARTICLE{10540263,
  author={He, Xiaodong and Li, Zhongkui},
  journal={IEEE Transactions on Automatic Control}, 
  title={Simultaneous Position and Orientation Planning of Nonholonomic Multirobot Systems: A Dynamic Vector Field Approach}, 
  year={2024},
  volume={69},
  number={12},
  pages={8354-8369},
  doi={10.1109/TAC.2024.3406475}}

@INPROCEEDINGS{10886648,
  author={Qiao, Yike and He, Xiaodong and Li, Zhongkui},
  booktitle={2024 IEEE 63rd Conference on Decision and Control (CDC)}, 
  title={Motion Planning of 3D Nonholonomic Robots via Curvature-Constrained Vector Fields}, 
  year={2024},
  volume={},
  number={},
  pages={5807-5812},
  doi={10.1109/CDC56724.2024.10886648}}

@INPROCEEDINGS{7799011,
  author={Andreetto, Marco and Divan, Stefano and Fontanelli, Daniele and Palopoli, Luigi},
  booktitle={2016 IEEE 55th Conference on Decision and Control (CDC)}, 
  title={Hybrid feedback path following for robotic walkers via bang-bang control actions}, 
  year={2016},
  volume={},
  number={},
  pages={4855-4860},
  doi={10.1109/CDC.2016.7799011}}

@article{doi:10.2514/1.G001349,
author = {Saleem, Abdul and Ratnoo, Ashwini},
title = {Lyapunov-Based Guidance Law for Impact Time Control and Simultaneous Arrival},
journal = {Journal of Guidance, Control, and Dynamics},
volume = {39},
number = {1},
pages = {164-173},
year = {2016},
doi = {10.2514/1.G001349},



}

@article{ZHENG2013401,
title = {Rendezvous of unicycles: A bearings-only and perimeter shortening approach},
journal = {Systems \& Control Letters},
volume = {62},
number = {5},
pages = {401-407},
year = {2013},
issn = {0167-6911},
doi = {https://doi.org/10.1016/j.sysconle.2013.02.006},
author = {Ronghao Zheng and Dong Sun}
}

@article{doi:10.1137/040620564,
author = {Lin, J. and Morse, A. S. and Anderson, B. D. O.},
title = {The Multi-Agent Rendezvous Problem. Part 2: The Asynchronous Case},
journal = {SIAM Journal on Control and Optimization},
volume = {46},
number = {6},
pages = {2120-2147},
year = {2007},
doi = {10.1137/040620564},
}

@article{doi:10.1137/040620552,
author = {Lin, J. and Morse, A. S. and Anderson, B. D. O.},
title = {The Multi-Agent Rendezvous Problem. Part 1: The Synchronous Case},
journal = {SIAM Journal on Control and Optimization},
volume = {46},
number = {6},
pages = {2096-2119},
year = {2007},
doi = {10.1137/040620552},
}
\end{document}